\documentclass{article}
\usepackage{ijcai24}

\usepackage{times}
\usepackage{soul}
\usepackage{url}
\usepackage[hidelinks]{hyperref}
\usepackage[utf8]{inputenc}
\usepackage[small]{caption}
\usepackage{graphicx}
\usepackage{amsmath}
\usepackage{amsthm}
\usepackage{booktabs}
\usepackage{algorithm}
\usepackage{algorithmic}
\usepackage[switch]{lineno}

\usepackage{algorithm}  
\usepackage{algorithmic}    

\usepackage{microtype}
\usepackage{subfig}
\usepackage{booktabs} % for professional tables
\usepackage{bbding}

\usepackage[utf8]{inputenc} % allow utf-8 input
\usepackage[T1]{fontenc}    % use 8-bit T1 fonts
\usepackage{url}            % simple URL typesetting
\usepackage{booktabs}       % professional-quality tables
\usepackage{amsfonts}       % blackboard math symbols
\usepackage{nicefrac}       % compact symbols for 1/2, etc.
\usepackage{microtype}      % microtypography
\usepackage{xcolor}         % colors

\usepackage{xspace}
\usepackage{textcomp}

\usepackage{stmaryrd}
\usepackage{xspace}

\usepackage{algorithm}

\usepackage{amsmath}
\usepackage{amssymb}
\usepackage{mathtools}
\usepackage{amsthm}
\usepackage{dsfont}
\usepackage{diagbox}
\usepackage{adjustbox}
\usepackage{multirow}
\usepackage{makecell}
\usepackage{colortbl}
\usepackage{tikz}
\usepackage{bm}

\theoremstyle{plain}
\newtheorem{theorem}{Theorem}%[section]

\theoremstyle{definition}
\newtheorem{definition}[theorem]{Definition}

\theoremstyle{remark}

\definecolor{cvprblue}{rgb}{0.21,0.49,0.74}
\definecolor{mygrey}{rgb}{ .863,  .863,  .863}
\definecolor{mypink}{rgb}{0.999,  0.414,   0.414}
\definecolor{mygreen}{rgb}{ .612,  .859,  .659}
\definecolor{myred}{rgb}{ .906,  .600,  .569}

\newcommand{\rightsign}{\color{mygreen} \CheckmarkBold}
\newcommand{\falsesign}{\color{myred} \XSolidBrush}

\newrobustcmd*{\mytriangleup}[1]{\tikz{\filldraw[draw=#1,fill=#1] (0,0) --
(0.1cm,0) -- (0.05cm,0.1cm);}}

\title{Injecting Imbalance Sensitivity for Multi-Task Learning}

\author{
Zhipeng Zhou$^1$
\and
Liu Liu$^{2,\dagger}$\and
Peilin Zhao$^{2}$\And
Wei Gong$^{1,\dagger}$\\
\affiliations
$^1$University of Science and Technology of China\thanks{$\dagger$ Corresponding authors. Work done when Z. Zhou works as an intern in Tencent AI Lab.}\\
$^2$Tencent AI Lab\\
\emails
zzp1994@mail.ustc.edu.cn,
\{leonliuliu, masonzhao\}@tencent.com,
weigong@ustc.edu.cn
}
\begin{document}

\maketitle

\begin{abstract}
Multi-task learning (MTL) has emerged as a promising approach for deploying deep learning models in real-life applications. Recent studies have proposed optimization-based learning paradigms to establish task-shared representations in MTL. However, our paper empirically argues that these studies, specifically gradient-based ones, primarily emphasize the conflict issue while neglecting the potentially more significant impact of imbalance/dominance in MTL. In line with this perspective, we enhance the existing baseline method by injecting imbalance-sensitivity through the imposition of constraints on the projected norms. To demonstrate the effectiveness of our proposed IMbalance-sensitive Gradient (\texttt{IMGrad}) descent method, we evaluate it on multiple mainstream MTL benchmarks, encompassing supervised learning tasks as well as reinforcement learning. The experimental results consistently demonstrate competitive performance.
\end{abstract}

\section{Introduction}
\label{sec:intro}

Real-life scenarios often involve the need to handle multiple distinct tasks concurrently, typically achieved by designing task-specific models to ensure satisfactory performance. However, this approach becomes impractical as the number of tasks grows, as it would require significant computational resources and memory. To address this challenge and establish an efficient multi-task learning (MTL) framework, recent research has focused on developing a single model capable of performing well on all target tasks.

Currently, research on MTL can be broadly categorized into two frameworks: architecture-based~\cite{liu2019end,ye2022inverted,gao2019nddr,chen2023mod} and optimization-based approaches~\cite{sener2018multi,yu2020gradient,liu2021conflict,navon2022multi,liu2023famo}. The former emphasizes the design of efficient parameter sharing architectures for multiple tasks, whereas the latter typically employs a fixed architecture and focuses on developing optimization strategies to extract task-shared representations. In this paper, we exclusively introduce and compare our method with optimization-based approaches, as our proposed method falls within this framework.

\begin{figure}
    \centering
    \includegraphics[width=\linewidth]{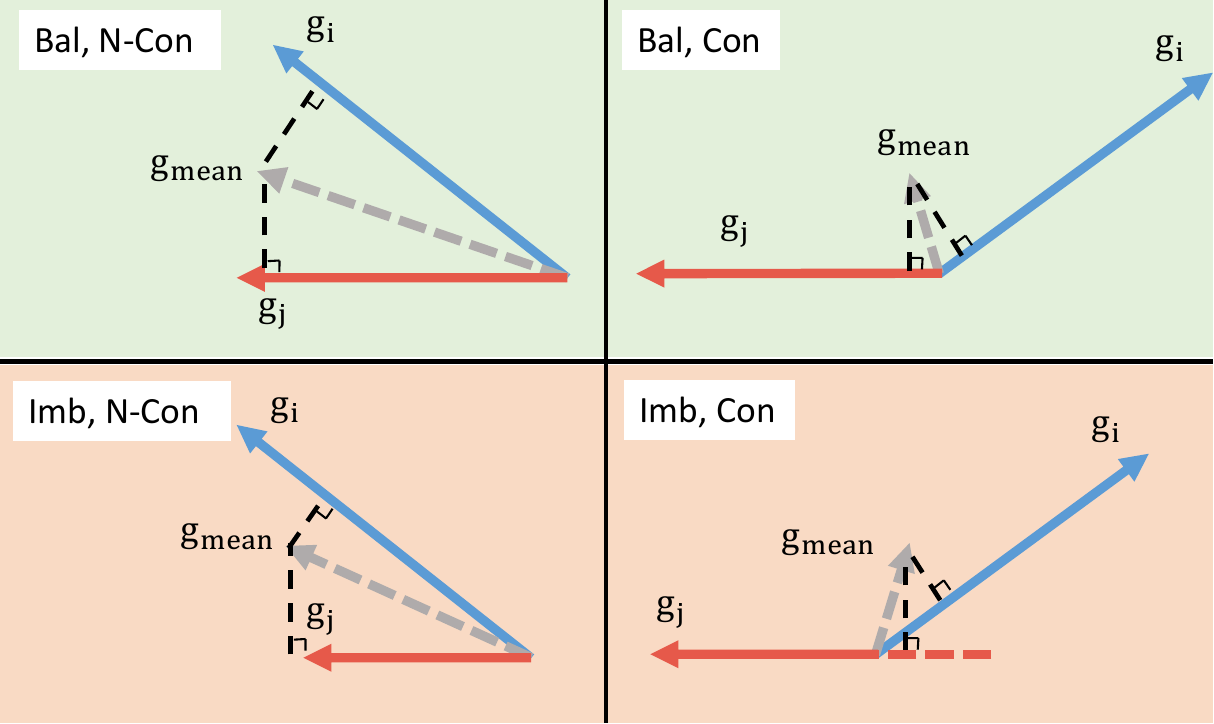}
    \caption{Illustration of imbalance and conflicting issue in multi-task learning. `Bal' and `Imb' represent balanced and imbalanced, while `N-Con' and `Con' represent non-conflicting and conflicting.}
    \label{fig:illu}
\end{figure}
% \usepackage{multirow}
\begin{table*}[h]
\setlength\tabcolsep{4pt}
\centering
\caption{Conflict-averse and imbalance-sensitive comparison for mainstream optimization-based MTL. Note that those which are imbalance-sensitive mean that their solution can tackle the imbalance issue.}
\label{table:moo_para_comp}
\footnotesize
\begin{tabular}{lllllllll|l}
\toprule
          &     GD           &   GradDrop            &    MGDA          &   PCGrad   &     IMTL       &  CAGrad   &  Nash-MTL  &   MoCo     &  \texttt{IMGrad}  \\    \midrule
Conflict-averse   &     \falsesign    &   \falsesign  &    \rightsign  &  \rightsign      &  \falsesign  & \rightsign  &   \falsesign   &  \rightsign    &   \rightsign  \\
Imbalance-sensitive &     \falsesign  &   \falsesign    &    \falsesign    &   \falsesign   &  \rightsign  &   \falsesign  &   \rightsign  &   \falsesign    &  \rightsign  \\ 
%Decoupled Objective &     \falsesign    &   \falsesign  &    \falsesign  &  \falsesign      &  \falsesign  & \rightsign  &   \rightsign   &  \rightsign  &    \rightsign  \\
\bottomrule
\end{tabular}
\end{table*}

In the realm of optimization-based methods, particularly those involving gradient manipulation, a shared paradigm is commonly followed, where task gradients are combined to achieve Pareto optimality for individual tasks. Despite the high performance demonstrated by these methods, the literature has predominantly overlooked the significance of the inherent imbalance nature among individuals (see \textbf{Definition~\ref{def:imb}}). This oversight can be attributed to the greater emphasis placed on addressing the conflict issue. However, it is important to note that the conflict issue alone may not be the fundamental obstacle hindering optimization in MTL. As illustrated in Figure~\ref{fig:illu}, a na\"{\i}ve linear scalarization (LS) strategy ($\bm{g_{mean}}$) effectively improves all individuals when they are balanced, regardless of conflicts. But it proves ineffective when both imbalance and conflict coexist, underscoring the importance of addressing conflicts that arise solely from imbalances. Furthermore, imbalanced task gradients can introduce optimization preferences and lead to imbalanced progress even in the absence of conflicts~\cite{liu2023famo}. Although previous solutions, such as IMTL~\cite{liutowards} and Nash-MTL~\cite{navon2022multi} illustrated in Table~\ref{table:moo_para_comp}, have somewhat mitigated the imbalance/dominance issue, they neither explicitly provide evidence to demonstrate the importance of the imbalance issue nor consider both conflict and imbalance issues simultaneously.

In this paper, we begin by empirically highlighting the significance of the imbalance issue in MTL and elucidate the advantages of incorporating imbalance sensitivity into baseline methods as our primary motivation. Subsequently, we enhance the well-established baseline method by injecting imbalance sensitivity through the imposition of constraints on the projected norms. Convergence and speedup analysis are provided in the \textbf{Appendix}.
In a nutshell, we summarize our contributions as three-fold:
\begin{itemize}
    \item We place significant emphasis on and empirically identify that the primary challenge in optimization-based MTL lies more in the aspect of imbalance rather than conflict. To the best of our knowledge, we are the first to explicitly assert this claim. 
    \item To introduce the imbalance sensitivity into the existing paradigm, we integrate the projected norm constraint into the objectives. This incorporation allows for a dynamic equilibrium between Pareto property (see \textbf{Definition~\ref{def:pareto}}) and convergence (two decoupled objectives), thereby enhancing the combined gradients and optimization trajectories. %Importantly, this approach is scalable and applicable to MTL methods with decoupled objectives. %Furthermore, in-depth theoretical analysis is provided to demonstrate the convergence guarantee of this approach.   
    \item The extensive experimental results present compelling evidence that \texttt{IMGrad} consistently enhances its baselines and surpasses the current advanced gradient manipulation methods in a diverse range of evaluations, e.g., supervised learning tasks, and reinforcement learning benchmarks.
\end{itemize}
\section{Related Work}
% introduce MTL methods generally
Currently, MTL approaches can be broadly categorized into two groups: architecture-based and optimization-based methods. Architecture-based approaches encompass various paradigms, including hard parameter sharing~\cite{heuer2021multitask,kokkinos2017ubernet}, soft parameter sharing~\cite{yang2016deep,gao2019nddr}, modulation and adapters~\cite{he2021towards,liu2022few}, and mixture of experts (MoE)~\cite{chen2023mod,fan2022m3vit}, etc. On the other hand, optimization-based MTL methods primarily focus on learning paradigms rather than structural designs or parameter sharing strategies. These methods aim to optimize all individual tasks to extract task-shared representations.

One classical optimization-based MTL approach is MGDA~\cite{sener2018multi}, which seeks a combined gradient with minimal norm using the Frank-Wolfe algorithm~\cite{jaggi2013revisiting}. PCGrad~\cite{yu2020gradient} addresses the conflict issue by projecting individual gradients onto orthogonal directions with respect to others. CAGrad~\cite{liu2021conflict} considers preserving both the Pareto property and global optimization, ultimately striving for a balance between the two objectives using a hyper-parameter. Nash-MTL~\cite{navon2022multi} negotiates to reach an agreement on a joint direction of parameter update, enabling all individual tasks to achieve more balanced progress. MoCo~\cite{fernando2022mitigating} tackles the problem of biased gradient directions in previous solutions by developing tracking parameters for correction. Our method falls within the realm of optimization-based MTL, with a specific focus on addressing the issue of imbalance-sensitivity, which is largely lacking in the aforementioned solutions. 

% Clearly claim our methods
\noindent \underline{\textbf{Discussion with Counterparts}}: To the best of our knowledge, IMTL~\cite{liutowards}, Nash-MTL~\cite{navon2022multi}, and FAMO~\cite{liu2023famo} are three recent works that explicitly consider the imbalance issue. However, all three works fail to provide evidence demonstrating the importance of the imbalance issue. Moreover, none of these approaches possess conflict-averse properties. Thus, there is still room for improvement. Although Nash-MTL appears to be designed to avoid conflicts, its practical implementation does not achieve this goal. Please refer to the \textbf{Appendix} for more discussion. 
% Nash-MTL~\cite{navon2022multi}, which incorporates constraints to ensure individual progress regardless of the level of imbalance. However, the key distinction between our approach and Nash-MTL lies in the fact that Nash-MTL does not actively leverage imbalance-sensitivity, resulting in a higher occurrence of Pareto failures. Our method, on the other hand, effectively mitigates these failures by actively incorporating imbalance-sensitivity.

% Imbalance sensitivity on other application
% Imbalance-sensitivity is a concept actively employed in the field of imbalanced learning, which aims to address the imbalance between different categories by utilizing pre-defined~\cite{zhou2023class,du2023no} or dynamically calculated imbalance ratios~\cite{ma2022delving,sinha2023difficulty}. As highlighted in Section 3.1, the imbalance issue plays a crucial role in MTL. Therefore, drawing inspiration from imbalanced learning and actively leveraging imbalance-sensitivity in the context of MTL is a natural and straightforward approach.
% Challenges of Optimization-based MTL
\begin{figure*}[h]
    \centering
    \subfloat[LS]{\includegraphics[width = 0.2\textwidth]{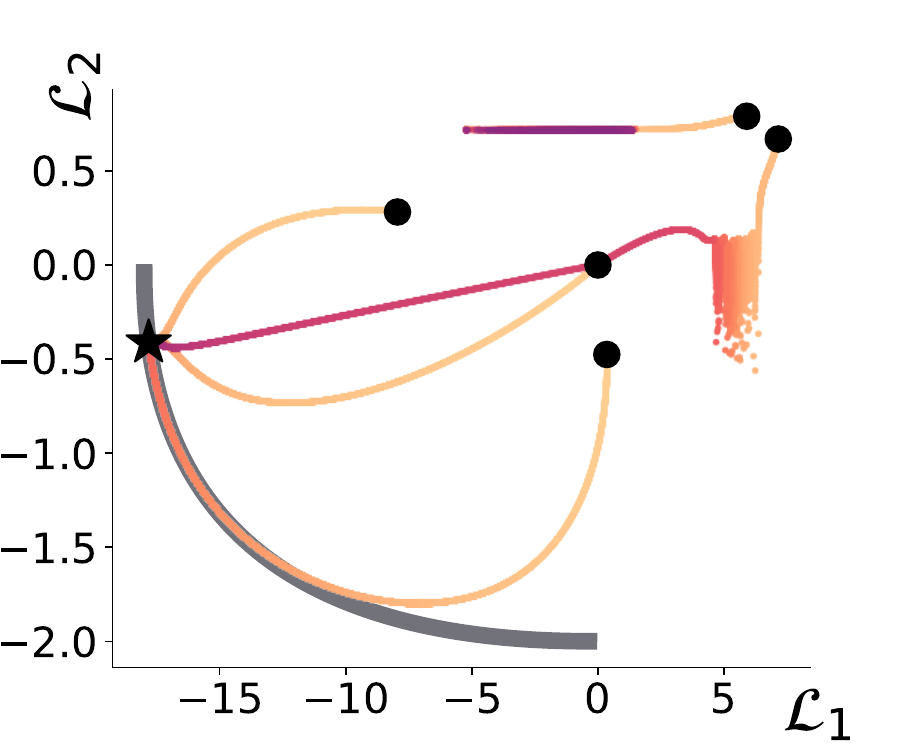}}
    \hfill
    \subfloat[PCGrad]{\includegraphics[width = 0.2\textwidth]{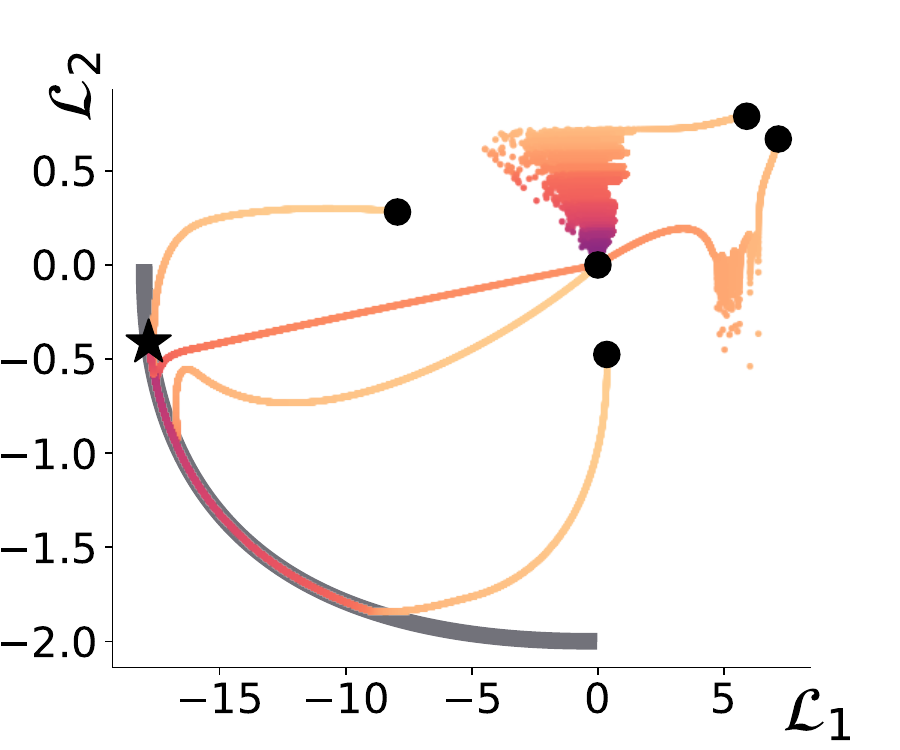}}     
    \hfill
    \subfloat[CAGrad]{\includegraphics[width = 0.2\textwidth]{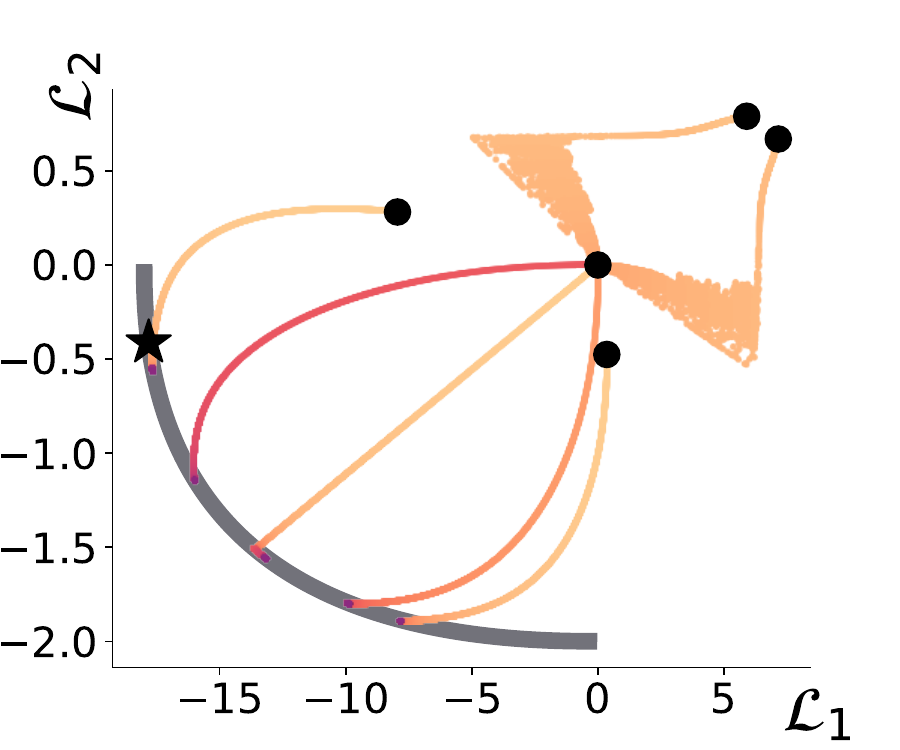}}
    \subfloat[Nash-MTL]{\includegraphics[width = 0.2\textwidth]{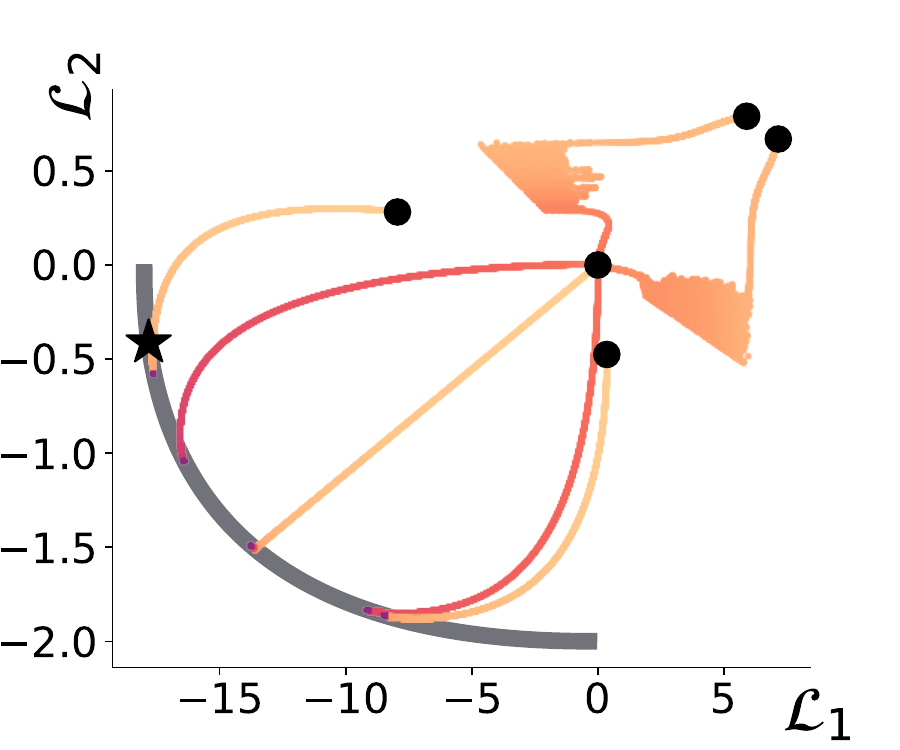}}
    \hfill
    \subfloat[\texttt{IMGrad}]{\includegraphics[width = 0.2\textwidth]{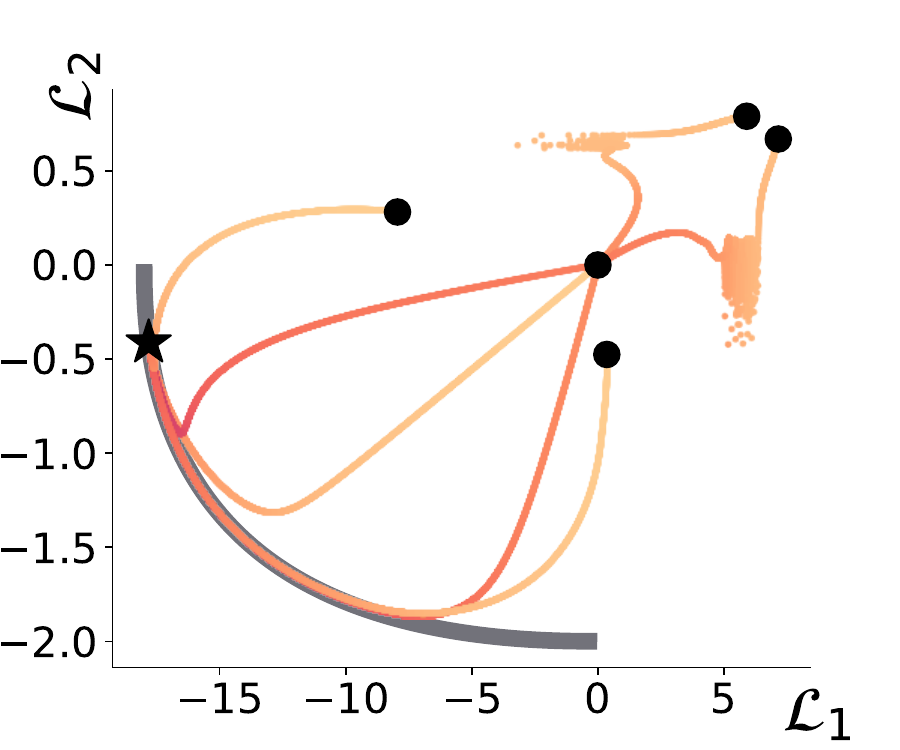}}
    \caption{Comparison of MTL approaches on the imbalanced synthetic two-task benchmark. $\bullet $ and $\star $ represent the starting point and global optimum, respectively, and grey line ${\color{gray}\rule[2pt]{0.5cm}{0.2em}}$ represents the Pareto front. Two objectives are extremely imbalanced weighted, i.e., $(0.9*\mathcal{L}_1, 0.1*\mathcal{L}_2)$. Please refer to the \textbf{Appendix} for more optimization trajectories under various pre-defined task weights.}
    \label{fig:toy_imb}
\end{figure*}
\section{Preliminary}

\subsection{Setup of Optimization-based MTL}
As mentioned, optimization-based MTL approaches operate under the assumption that the model consists of a task-shared backbone network alongside task-specific branches. Consequently, the primary objective of these approaches is to devise gradient combination strategies that optimize the backbone network to yield benefits across all tasks. Let us consider a scenario where there are $K \ge 2$ tasks available, each associated with a differentiable loss function $\mathcal{L}_i(\bm{\Theta})$, where $\bm{\Theta}$ represents the task-shared parameters. The goal of optimization-based MTL is to search for the optimal $\bm{\Theta^*} \in \mathbb{R}^m$ that minimizes the losses for all tasks. However, it is widely recognized that a simplistic linear scalar strategy, $\mathcal{L}_0(\bm{\Theta}) = \frac{1}{K}\sum_{i=1}^{K}\mathcal{L}_i(\bm{\Theta})$, fails to achieve satisfactory performance due to the conflict and imbalance issue.

\subsection{Pareto Concept}  % copy from PLOT ICLR'24, should be modified!!!!!!!
Formally, let us assume the weighted loss as $\mathcal{L}_{\bm{\omega}} = \sum_{i=1}^{K}\omega_i \mathcal{L}_i(\bm{\Theta})$, where $\bm{\omega} \in \bm{\mathcal{W}}$ and $\bm{\mathcal{W}}$ represents the probability simplex on $[K]$. A point $\bm{\Theta'}$ is said to Pareto dominate $\bm{\Theta}$ if and only if $\forall i, \mathcal{L}_i(\bm{\Theta'}) \leq \mathcal{L}_i(\bm{\Theta})$. Consequently, the Pareto optimal situation arises when no $\bm{\Theta'}$ can be found that satisfies $\forall i, \mathcal{L}_i(\bm{\Theta'}) \leq \mathcal{L}_i(\bm{\Theta})$ for the given point $\bm{\Theta}$. All points that meet these conditions are referred to as Pareto sets, and their solutions are known as Pareto fronts. Another concept, known as Pareto stationary, requires ${\rm{min}}_{\bm{\omega} \in \bm{\mathcal{W}}} \left \| \bm{g}_{\bm{\omega}}\right \| = 0$, where $\bm{g}_{\bm{\omega}}$ represents the weighted gradient $\bm{\omega}^{\top}\bm{G}$, and $\bm{G}$ is the gradients matrix whose each row is an individual gradient. We also provide some definitions here for ease of description.

% Formally, assume the weighted loss is $\mathcal{L}_{\bm{\omega}} = \sum_{i=1}^{N}\omega_i \mathcal{L}_i(\bm{\theta}), \bm{\omega} \in \bm{\mathcal{W}}$, where $\bm{\mathcal{W}}$ is the probability simplex on $[K]$. A point $\bm{\theta'}$ is said to Pareto dominate $\bm{\theta}$, only if $\forall i, \mathcal{L}_i(\bm{\theta'}) \leq \mathcal{L}_i(\bm{\theta})$. And therefore the Pareto optimal is the situation that no $\bm{\theta'}$ can be found that holds $\forall i, \mathcal{L}_i(\bm{\theta'}) \leq \mathcal{L}_i(\bm{\theta})$ for the point $\bm{\theta}$. All points that satisfy the above conditions are called Pareto sets, and their solutions are so-called Pareto fronts. Another concept called Pareto stationary, which requires ${\rm{min}}_{\bm{\omega} \in \bm{\mathcal{W}}} \left \| \bm{g}_{\bm{\omega}}\right \| = 0$, where $\bm{g}_{\bm{\omega}}$ is the weighted gradient. 

\begin{definition}[\textbf{Gradient Similarity}] \label{def:conf}
    Denote $\phi_{ij}$ as the angle between two task gradients $\bm{g_i}$ and $\bm{g_j}$, then we define the gradient similarity as $\cos \varphi_{ij}$ and the gradients as conflicting when $\cos \phi_{ij} < 0$.
\end{definition}

\begin{definition}[\textbf{Imbalance of Individuals}] \label{def:imb}
Assume the gradient owns the maximal norm in $\bm{G}$ is $\bm{g}_{max}$, and the corresponding minimal one is $\bm{g}_{min}$. We define the imbalance ratio of $\bm{G}$ as $r = \frac{\left \| \bm{g}_{max} \right\|}{\left \| \bm{g}_{min} \right\|}$. If $r > 1$, we call it's imbalanced.
\end{definition}

% \begin{definition}[\textbf{Imbalance \& Conflict Sensitivity}, \textit{informal}] \label{def:imb_sen}
% The explicit incorporation of the imbalance ratio in the objective of MTL is referred to as \textbf{Imbalance-Sensitivity}. Likewise, the explicit inclusion of the conflict ratio in the MTL objective is referred to as \textbf{Conflict-Sensitivity}.
% \end{definition}

\begin{definition}[\textbf{Pareto Property}] \label{def:pareto}
 For each training step, the combined optimization direction strives to promote all individuals simultaneously  (or at the very least, not cause detriment), i.e. for $\forall i$, the gradient similarity between $\bm{g_i}$ and the combined gradient $\bm{g_{\omega}}$ satisfies $\cos \phi_{\bm{\omega} i} \ge 0$. When this condition is not met, it is referred to as \textbf{Pareto failure}. 
\end{definition}

% \begin{definition}[\textbf{Decoupled Objective}, \textit{informal}] \label{def:decoup}
% The optimization objective of the MTL approach can be decoupled into multiple components that are responsible for different optimization goals, e.g., $\min_{\bm{\omega}} A(\bm{\omega}) + B(\bm{\omega})$, where $A(\bm{\omega})$ and $B(\bm{\omega})$ are functions optimized for different goals.      
% \end{definition}

\begin{figure}
    \centering
    \subfloat[PCGrad]{\includegraphics[width = 0.155\textwidth]{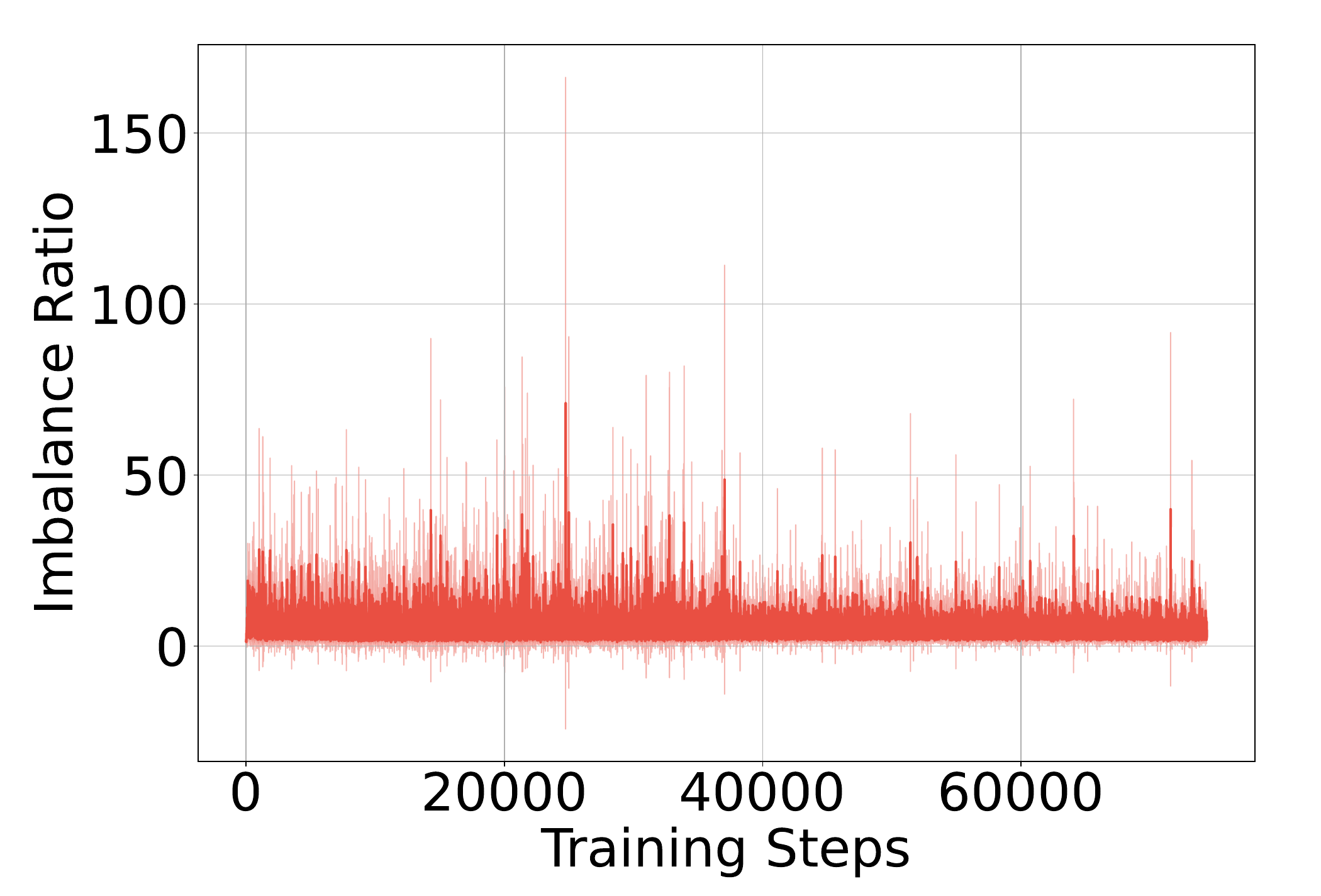}}
    \hfill
    \subfloat[CAGrad]{\includegraphics[width = 0.155\textwidth]{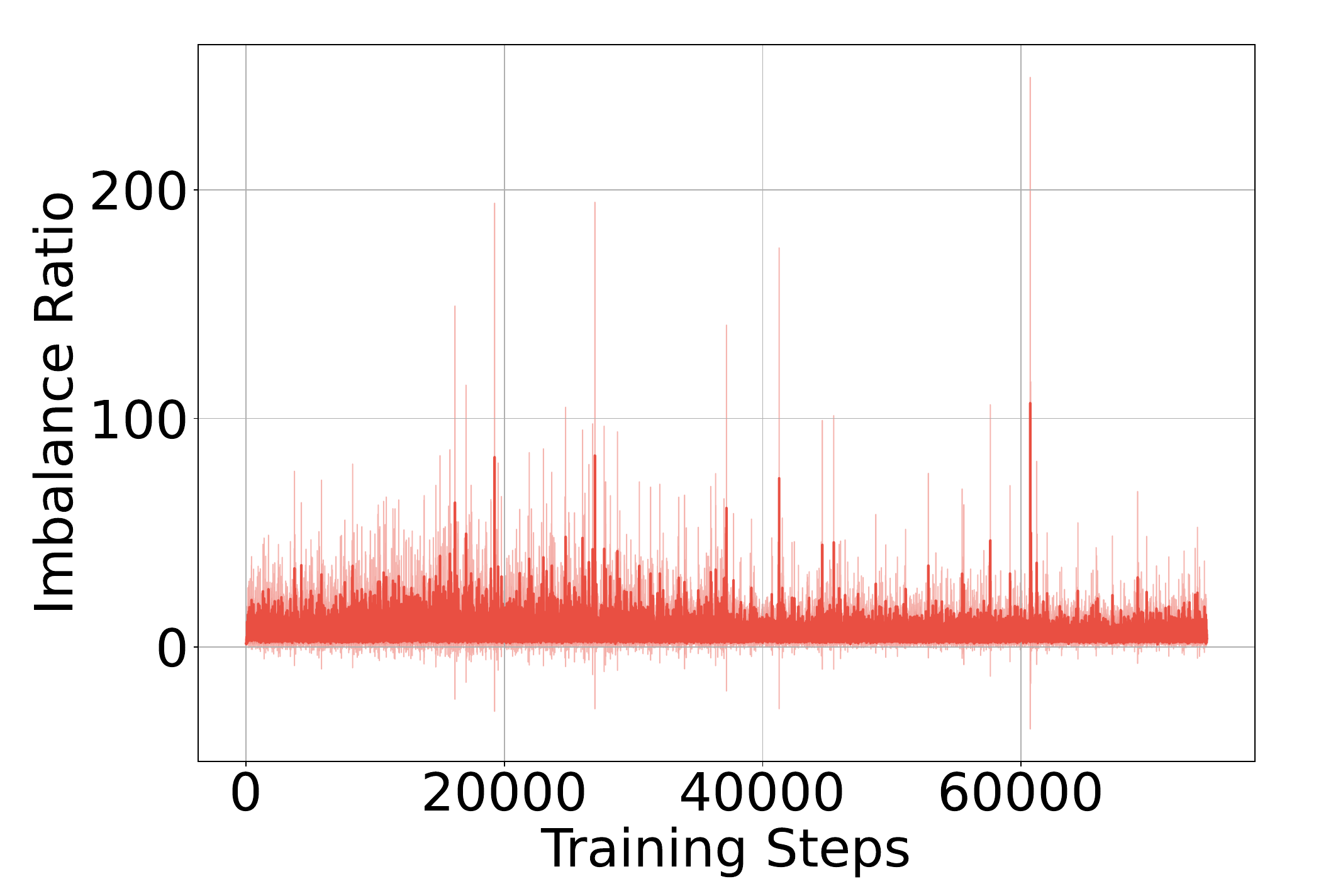}}
    \hfill
    \subfloat[Nash-MTL]{\includegraphics[width = 0.155\textwidth]{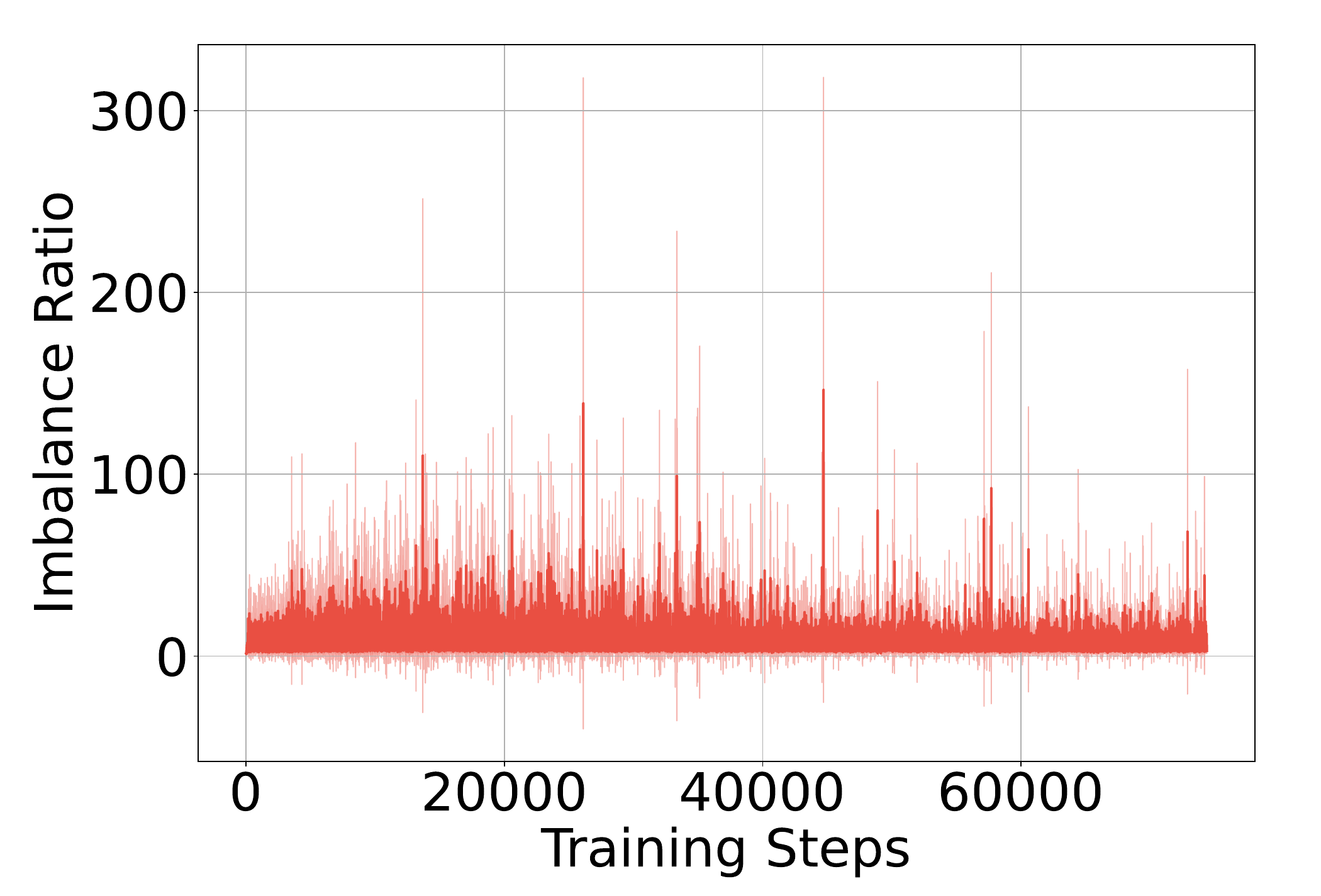}}
    \caption{Statistical imbalance ratios of MTL approaches.}
    \label{fig:imb_stat_city}
\end{figure}
\section{Motivation and Observation}
A substantial body of previous studies~\cite{sener2018multi,liu2021conflict,yu2020gradient,navon2022multi} have primarily focused on addressing the conflict issue rather than the imbalance issue. In this section, we aim to provide empirical insights into the significance of imbalance and elucidate how imbalance-sensitivity can bring benefits to current popular optimization-based MTL paradigms. Based on these insights, we naturally deduce our design in the next section. 

% \begin{figure*}
%     \centering
%     \includegraphics[width=0.9\linewidth]{figs/illustration.pdf}
%     \caption{Visualization of the combined updated gradient. More visual illustrations please refer to the \textbf{Appendix}.}
%     \label{fig:illu}
% \end{figure*}

%%%%%%%%%%%%%%%%%%%%%%%%----Motivation%----%%%%%%%%%%%%%%%%%%%%%%%%%%%%%
% generally show the importance of imbalance issue

%%%%%%%%%%%%%%% Sub Sec 1 %%%%%%%%%%%%%%%%
\subsection{Why Does Imbalance Matter More?} 
To begin, we conducted experiments on the CityScapes dataset~\cite{cordts2016cityscapes} to statistically analyze the imbalance ratios of representative optimization-based MTL methods (e.g., PCGrad~\cite{yu2020gradient}, CAGrad~\cite{liu2021conflict}, Nash-MTL~\cite{navon2022multi}). The results of these experiments are presented in Figure~\ref{fig:imb_stat_city}. From the depicted results, it is evident that all the methods exhibit significant imbalance during training, which poses a substantial challenge when attempting to optimize all individuals simultaneously, thereby underscoring the importance of addressing the imbalance issue.
\begin{figure}
    \centering
    \subfloat[Multi-Task Objective]{\includegraphics[width = 0.155\textwidth]{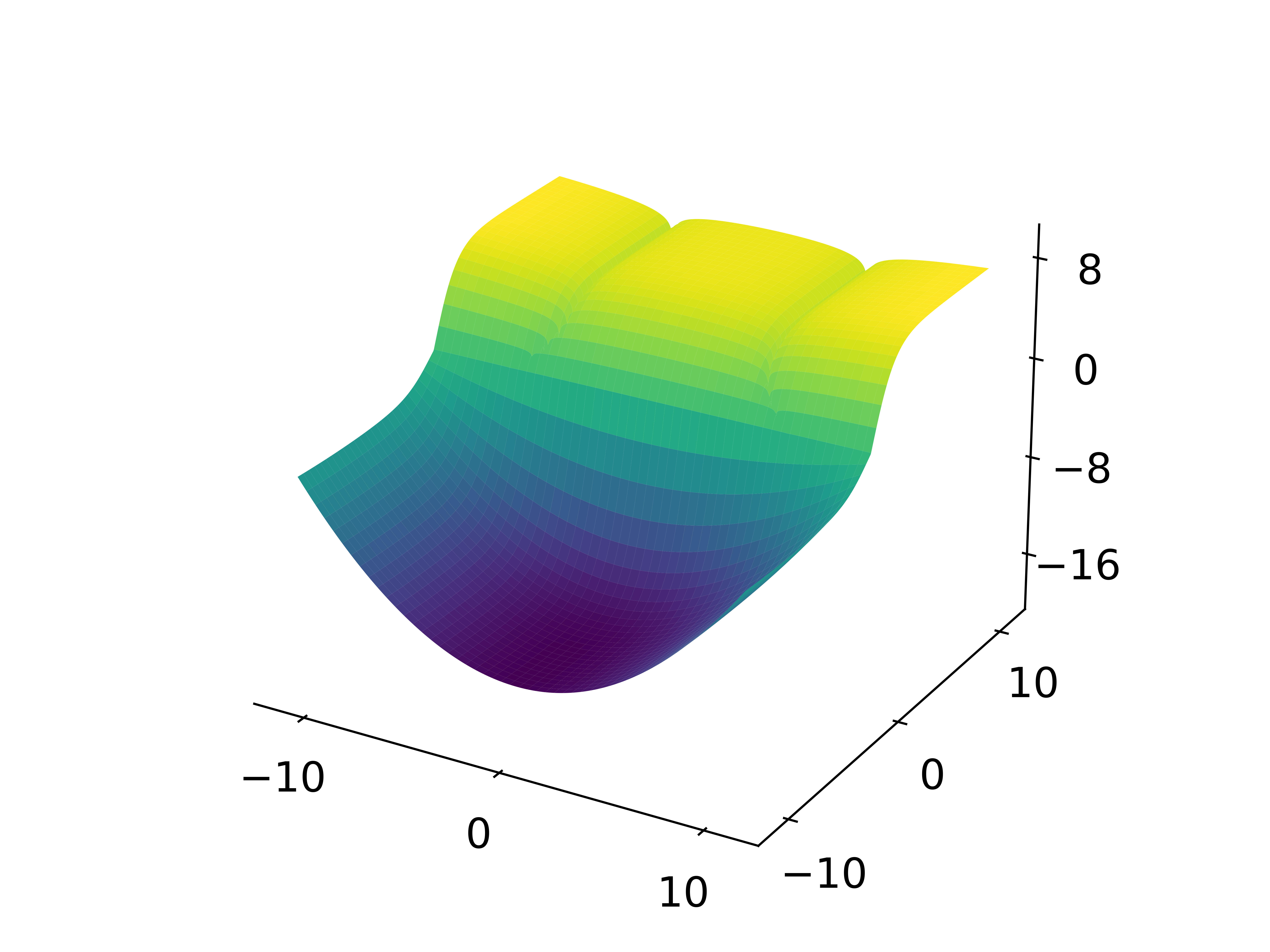}}
    \hfill
    \subfloat[Task 1 Objective]{\includegraphics[width = 0.155\textwidth]{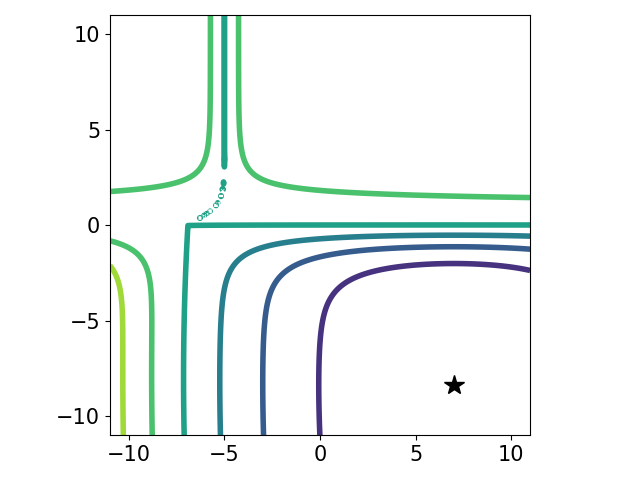}}
    \hfill
    \subfloat[Task 2 Objective]{\includegraphics[width = 0.155\textwidth]{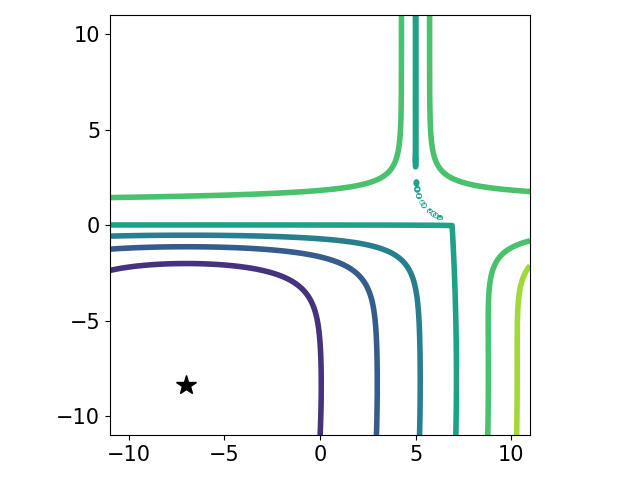}}
    \hfill  \\
    \subfloat[LS]{\includegraphics[width = 0.13\textwidth]{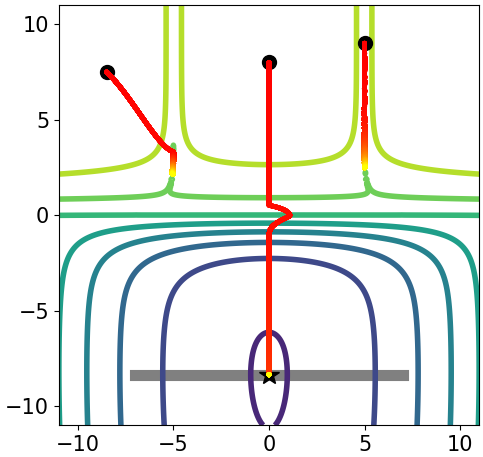}}
    \hfill
    \subfloat[CAGrad]{\includegraphics[width = 0.13\textwidth]{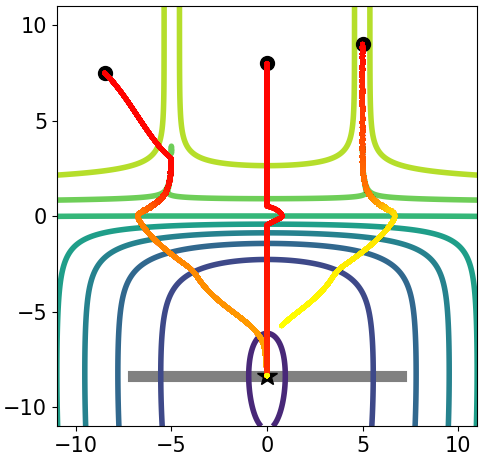}}
    \hfill
    \subfloat[\texttt{IMGrad}]{\includegraphics[width = 0.155\textwidth]{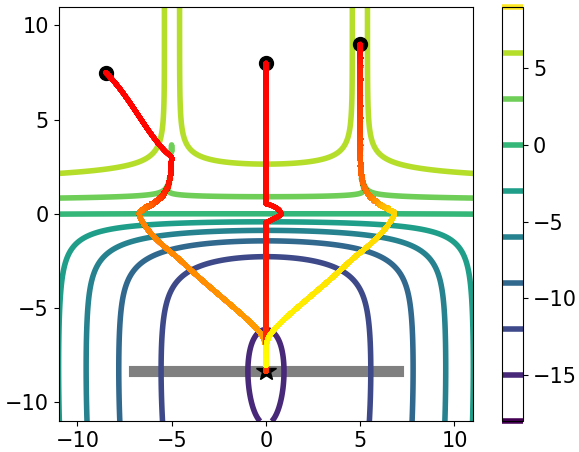}}
    \caption{Comparison of MTL approaches on the toy examples. We use the tool provided CAGrad to generate the synthetic toy examples with two objective shown in (b) and (c). In this case, both objective are equally weighted.} %, i.e., $(0.5*\mathcal{L}_1, 0.5*\mathcal{L}_2)$.}
    \label{fig:toy_bal}
\end{figure}

Secondly, to demonstrate the higher priority of imbalance issue, we show the toy example results that present imbalance and conflict among gradients in the following cases:
\begin{itemize}
    \item Conflict ({\rightsign}); Imbalance ({\falsesign}): In Figure~\ref{fig:toy_bal}, we manually create scenarios where conflict exists but imbalance is absent. By closely examining the \underline{center} trajectories in Figure~\ref{fig:toy_bal} (d)(e)(f), we observe that all methods can easily reach the optimal point when imbalance is absent, regardless of the presence of conflicts. This observation suggests that the sole existence of conflicts has limited impact on optimization, emphasizing the importance of addressing the imbalance issue.
    \item Conflict ({\falsesign}); Imbalance ({\rightsign}): Simulating an optimization trajectory without conflicts among individuals can indeed be challenging. Therefore, we adopt the setting from Nash-MTL~\cite{navon2022multi} to handcraft an imbalance-dominated optimization scenario. The resulting trajectories are depicted in Figure~\ref{fig:toy_imb}. It is evident that all the compared approaches fail to converge at the desired global optimum from all initial starts under the extreme imbalance circumstances, though most of them reach the Pareto front. Additionally, the trajectories at the sides in Figure~\ref{fig:toy_bal} (d)(e)(f) also highlight the issue of progress hindered by imbalance. Specifically, CAGrad fails to reach the global optimum compared to \texttt{IMGrad} despite undergoing the same number of optimization steps.%~\footnote{More details of toy experiments please refer to the \textbf{Appendix}}. 
\end{itemize}

%%%%%%%%%%%%%%% Sub Sec 2 %%%%%%%%%%%%%%%%
\subsection{The Impacts of the Imbalance Issue}
In Table~\ref{table:moo_para_comp}, we list and compare mainstream optimization-based MTL approaches. The table focuses on two key properties: conflict-averse and imbalance-sensitive properties. It is observed that most MTL approaches possess the conflict-averse property due to their design nature. However, only a few approaches are imbalance-sensitive~\footnote{We provide imbalance-sensitive analysis for IMTL, Nash-MTL, and FAMO in the \textbf{Appendix}.}, and currently, there are no methods that possess both properties simultaneously. 
\begin{figure}
    \centering
    \subfloat[PCGrad]{\includegraphics[width = 0.155\textwidth]{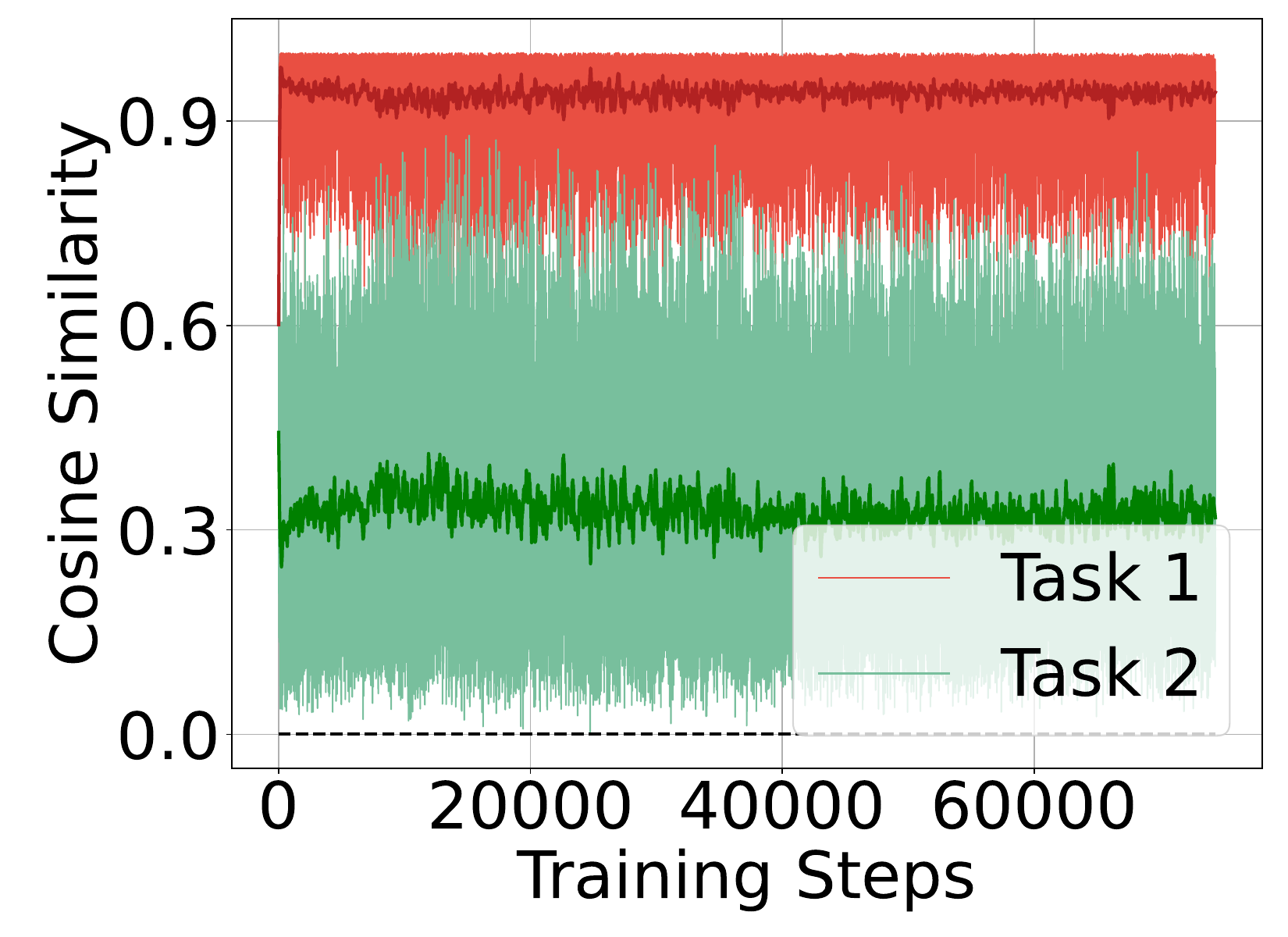}}
    \hfill
    \subfloat[CAGrad]{\includegraphics[width = 0.155\textwidth]{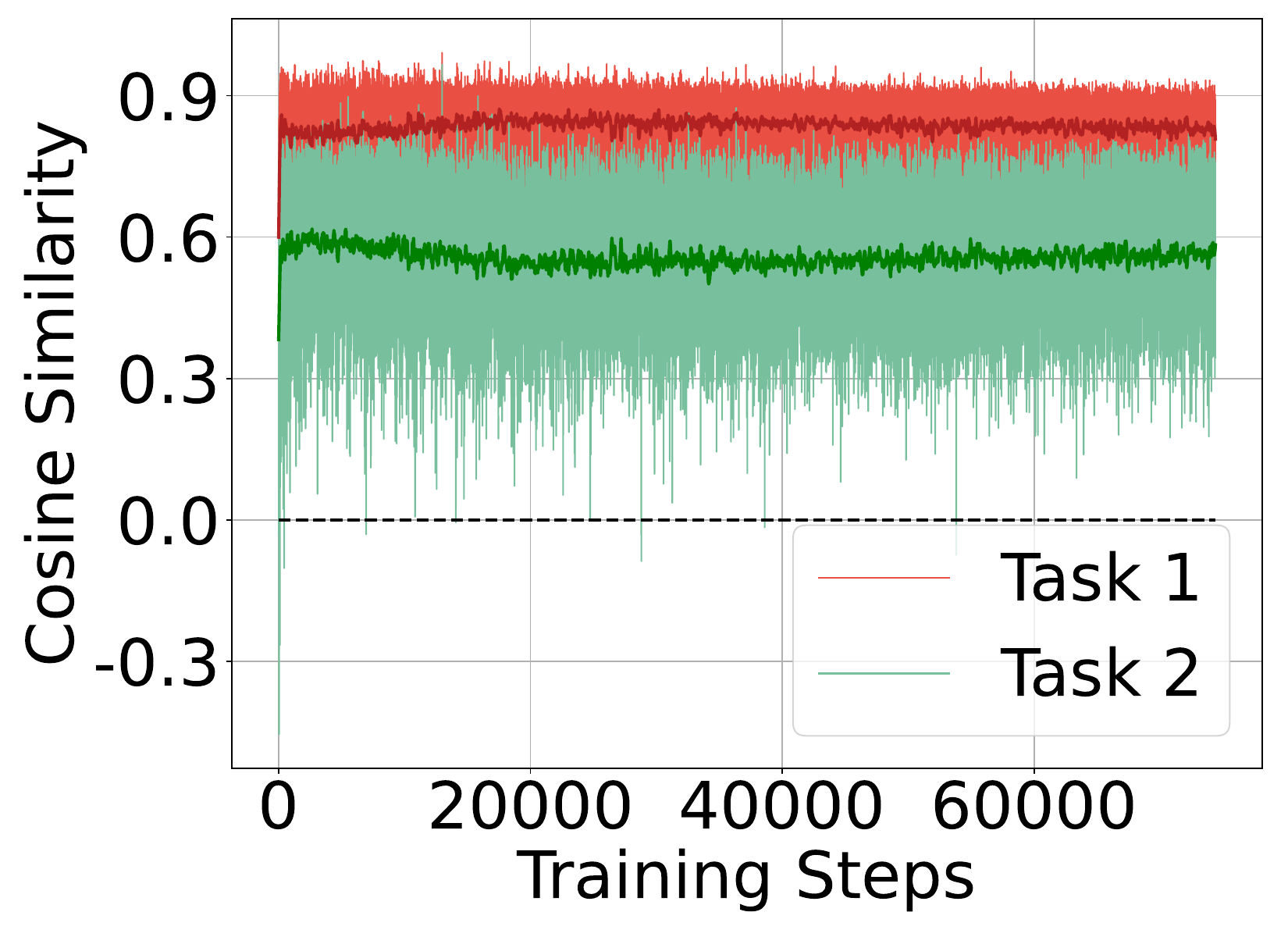}}
    \hfill
    \subfloat[Nash-MTL]{\includegraphics[width = 0.155\textwidth]{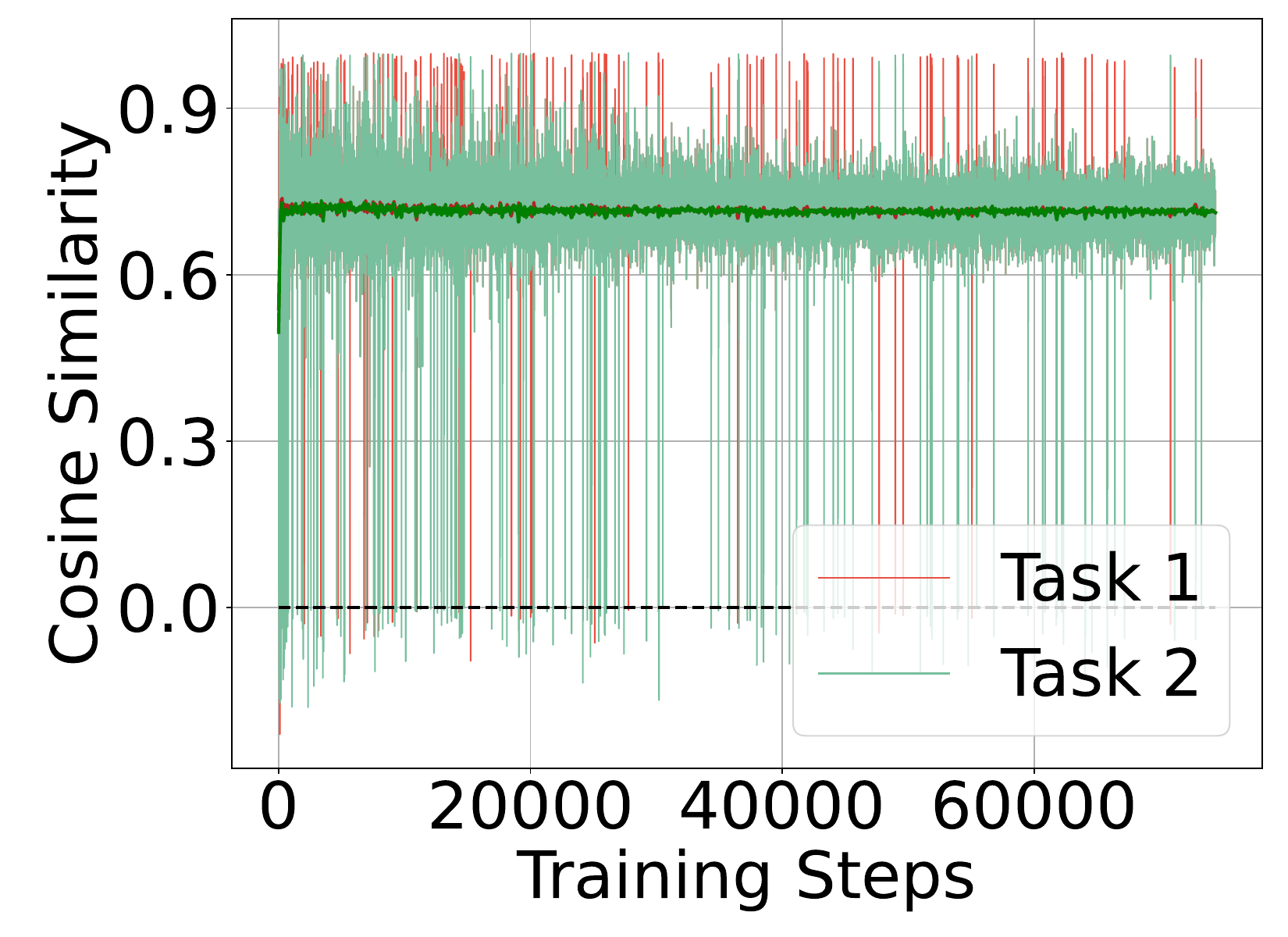}}
    \hfill  \\
    \subfloat[PCGrad ($\Delta m\% = 22.46$)]{\includegraphics[width = 0.155\textwidth]{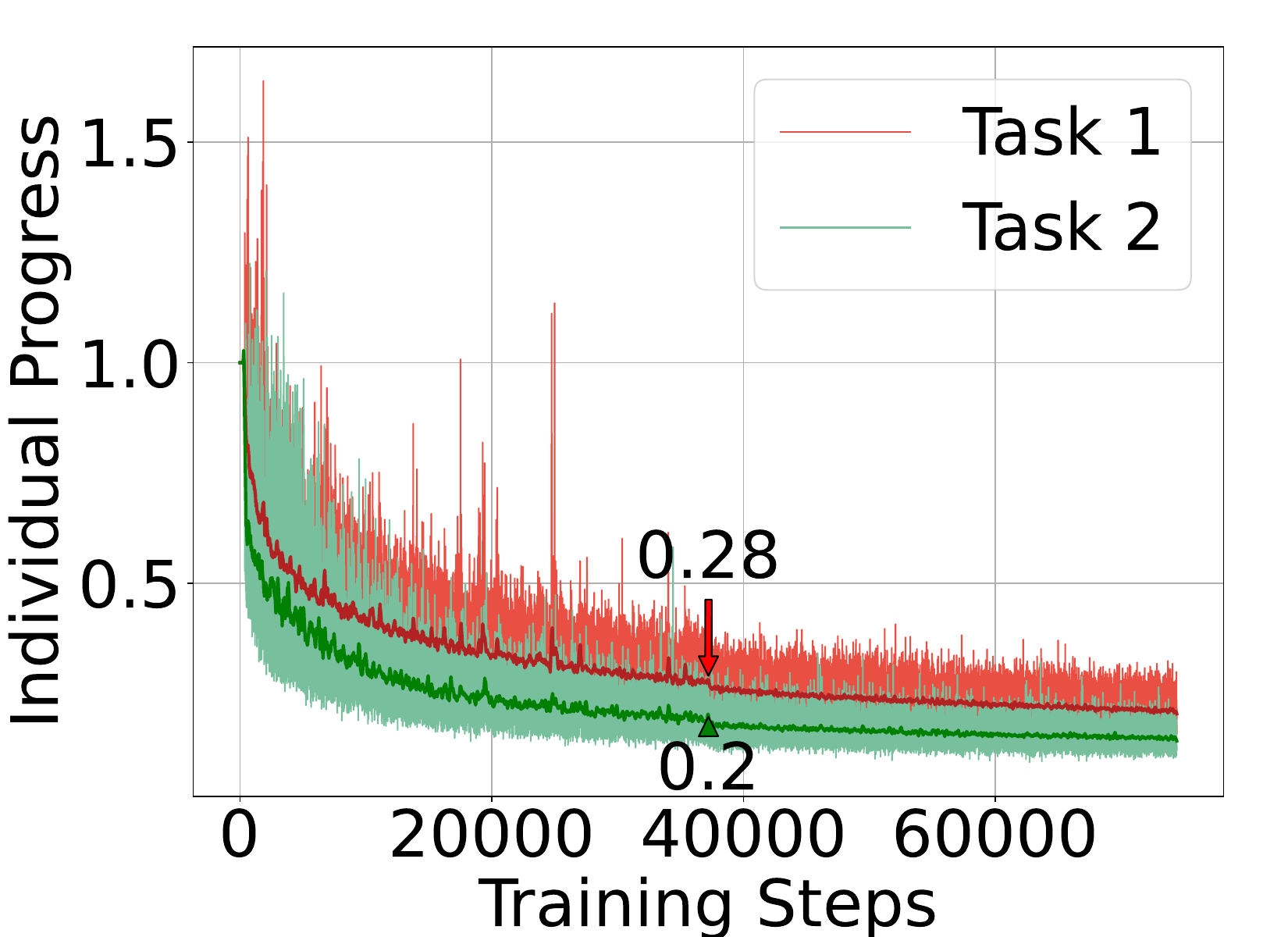}}
    \hfill
    \subfloat[CAGrad ($\Delta m\% = 9.97$)]{\includegraphics[width = 0.155\textwidth]{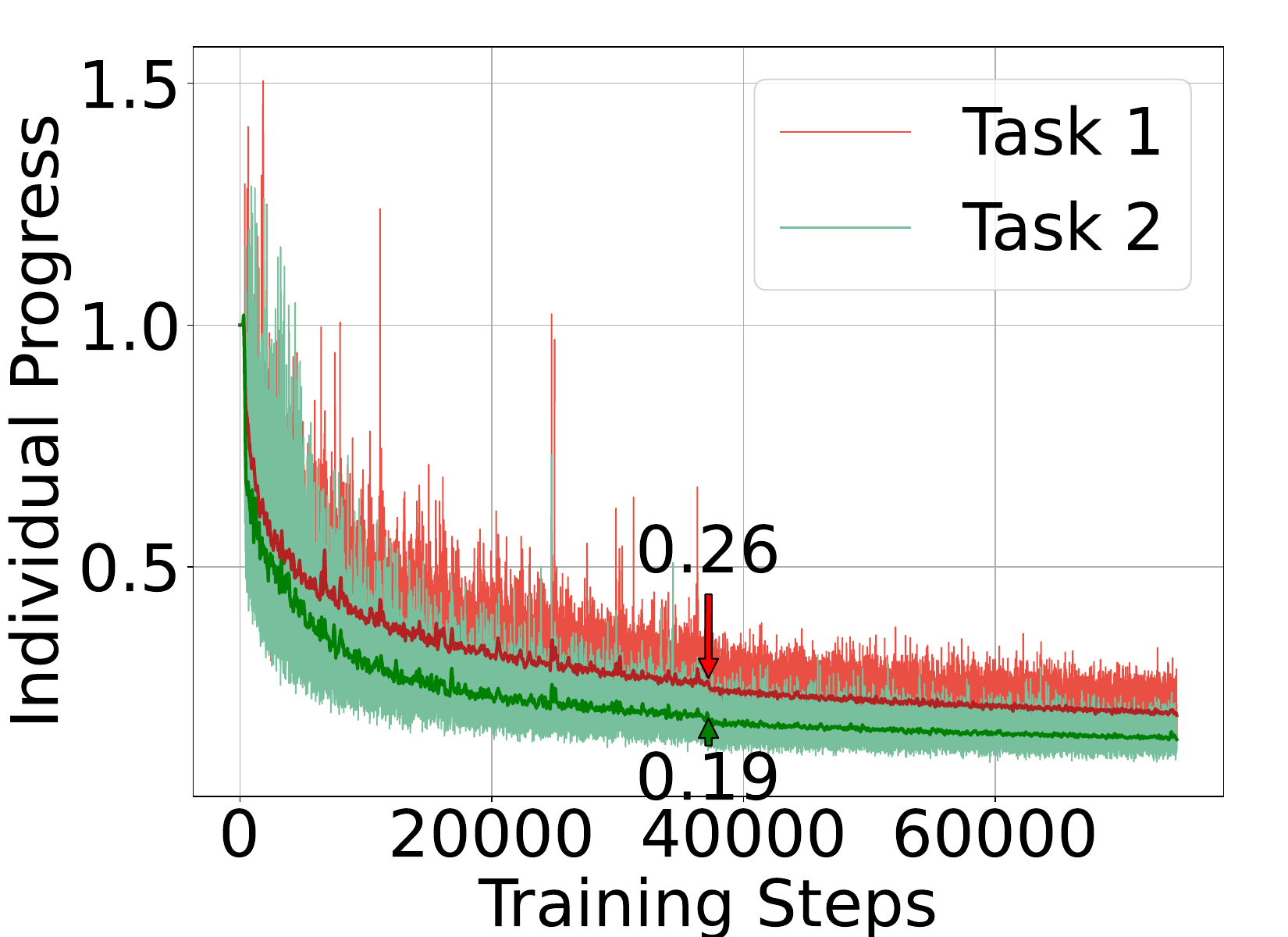}}
    \hfill
    \subfloat[Nash-MTL ($\Delta m\% = 8.20$)]{\includegraphics[width = 0.155\textwidth]{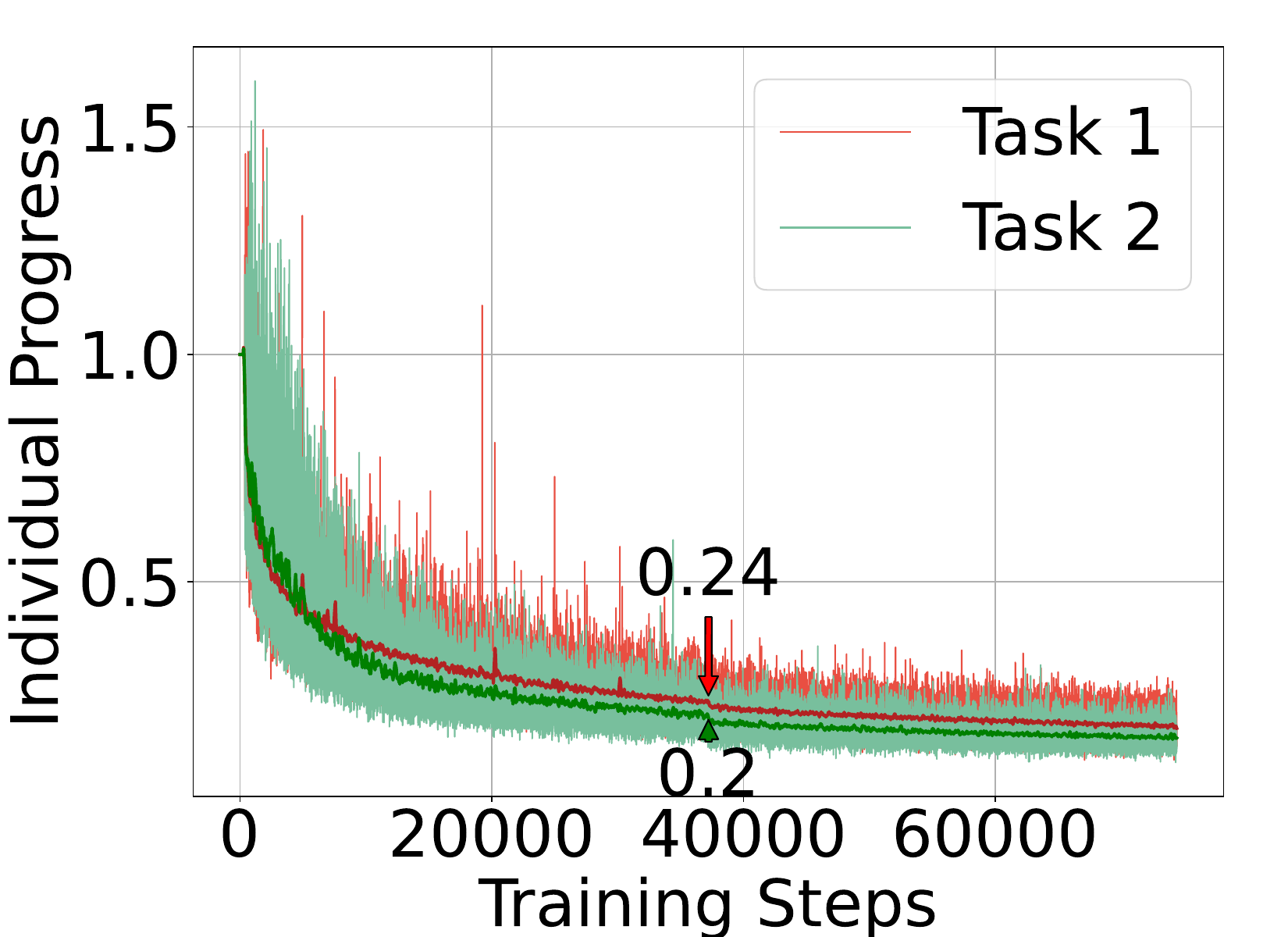}}
    \caption{Individual gradient similarity and progress analysis of MTL algorithms on CityScapes. (a)-(c) show the gradient similarities between individuals and the combined gradient; (d)-(e) present the progress of individuals during optimization.}
    \label{fig:sim_pro}
\end{figure}
Furthermore, we analyze two imbalance-deduced issues that occur and impede past solutions during optimization: Pareto failure and imbalanced individual progress.

% show the ratio
\noindent \textbf{Pareto Failure}: As shown in Figure~\ref{fig:sim_pro} (a)(b)(c), it is evident that PCGrad effectively mitigates conflict issue by projecting individuals onto orthogonal directions with respect to others. In contrast, CAGrad exhibits a certain probability of failing to preserve the conflict issue due to its inherent compromise between conflict-averse and convergence. This compromise is inevitably influenced by the issue of imbalance. As illustrated in Figure~\ref{fig:cagrad_imgrad_illu}, CAGrad tends to prioritize the combined gradient that deviates from the individual with the least norm when encountering imbalanced scenarios, leading to potential conflicts. Surprisingly, although Nash-MTL imposes a strong constraint for the Pareto property, i.e., $\forall i, -\varphi_i(\bm{\omega}) \le 0, \varphi_i(\bm{\omega}) = \log(\omega_i) + \log(\bm{g_i^{\top}}\bm{G}\bm{\omega}), \bm{G} = [\bm{g_1}, \bm{g_2}, ..., \bm{g_K}]$, it often fails to achieve such a guarantee. This failure can be attributed to the presence of negative terms in $\bm{g_i^{\top}}\bm{G}$, indicating conflicts between $\bm{g_i}$ and $\bm{g_j}$
 . Consequently, this leads to infeasible errors in the \textit{cvxpy}~\cite{diamond2016cvxpy} implementation, and the Nash-MTL algorithm chooses to skip the current step when such errors occur. As a result, Nash-MTL frequently encounters Pareto failures due to the co-existence of imbalance and conflict, as depicted in Figure~\ref{fig:sim_pro} (c). \texttt{IMGrad} demonstrates a tendency to acquire a combined gradient that effectively preserves the Pareto property as the imbalance ratio increases. %Please refer to Section~\ref{sec:design} for technique details.

% Comparison of gradient similarity curve, refer to PLOT(ICLR'24), and task progress definition in GradNorm; Currently run on cityscapes, later can provide nyuv2 if time avaiable.
\noindent \textbf{Imbalanced Individual Progress}:
We employ an individual progress metric proposed by~\cite{chen2018gradnorm}, which is defined as follows:
\begin{align}
    r_i(t) = \mathcal{L}_i(t) / \mathcal{L}_i(0)
\end{align}
where $\mathcal{L}_i(t)$ represents the individual loss value at $t$ time. As depicted in Figure~\ref{fig:sim_pro} (d)(e)(f), Nash-MTL demonstrates a narrower gap in terms of individual progress compared to CAGrad and PCGrad. This can be attributed to the more balanced combination employed by Nash-MTL, as indicated by the cosine similarity in (a)(b)(c). Consequently, Nash-MTL exhibits superior overall performance, characterized by a smaller $\Delta m\%$. Specifically, $\Delta m\%$ is widely adopted to evaluate the overall degradation compared to independently trained models, which are considered as the reference oracles. Its formal definition can be found in the \textbf{Performance Evaluation} section. 

Unfortunately, none of the above methods get rid of both Pareto failure and imbalanced individual progress, primarily due to their limited focus on the imbalance issue.

% \begin{remark}[Are Nash-MTL and Align-MTL Imbalance-Sensitive?] It is worth noting that both Nash-MTL and Align-MTL effectively address the issue of imbalance, as evidenced by the reduced dominance of the combined gradient by individuals with larger gradient norms. However, these methods do not fully exploit the imbalance-sensitive property in their objective design. For more detailed information, please refer to the \textbf{Appendix}.
% \end{remark}

%%%%%%%%%%%%%%% Sub Sec 3 %%%%%%%%%%%%%%%%
\subsection{Benefits of Integrating Imbalance-Sensitivity}
% apply MOO method when imbalance ratio > \alpha, otherwise naive apply GD
The toy results depicted in Figure~\ref{fig:toy_imb} and Figure~\ref{fig:toy_bal} demonstrate that among the methods evaluated, only our proposed \texttt{IMGrad}, which incorporates imbalance-sensitivity, consistently arrives at the optimal point from all initial starts.

To further elucidate the advantages of imbalance-sensitivity in optimization-based MTL, we have implemented a na\"{\i}ve method called \textit{Adaptive Threshold}. This baseline selectively applies optimization-based MTL approaches only when the imbalance ratio surpasses a specific threshold. The results of this implementation on CityScapes are presented in Figure~\ref{fig:benefit} (a). It is evident that all baselines exhibit varying performance as the imbalance ratio fluctuates, emphasizing the significance of imbalance-sensitivity. Notably, all baselines outperform their respective vanilla versions under specific threshold conditions, providing additional evidence of the effectiveness of injecting imbalance-sensitivity.
\begin{figure}
    \centering
    \subfloat[Imbalance Sensitivity]{\includegraphics[width = 0.23\textwidth]{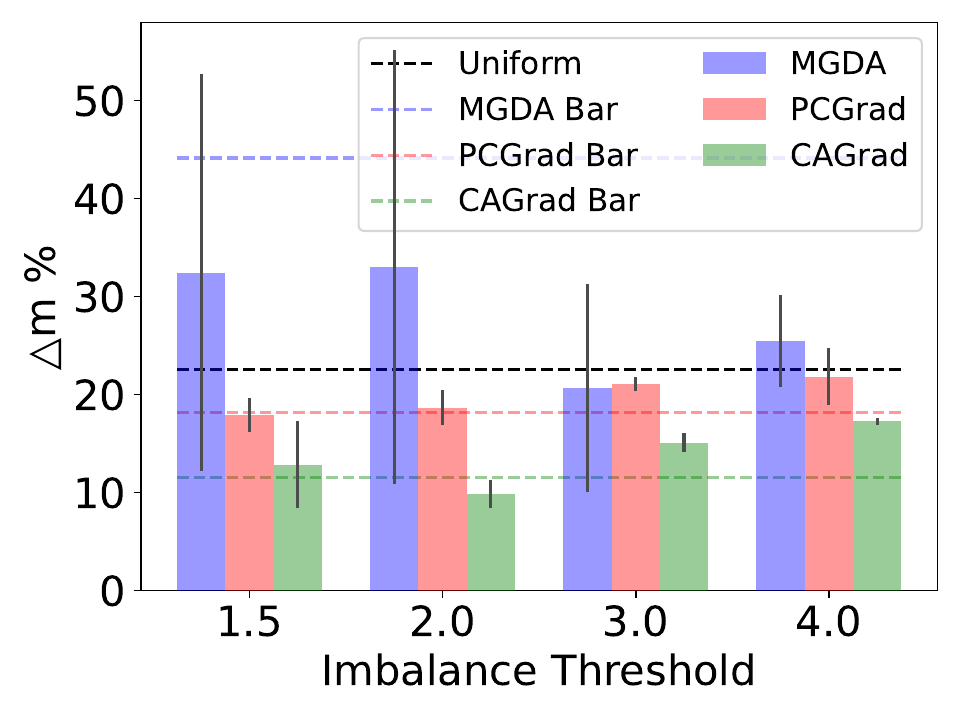}}
    \hfill
    \subfloat[Conflict Sensitivity]{\includegraphics[width = 0.23\textwidth]{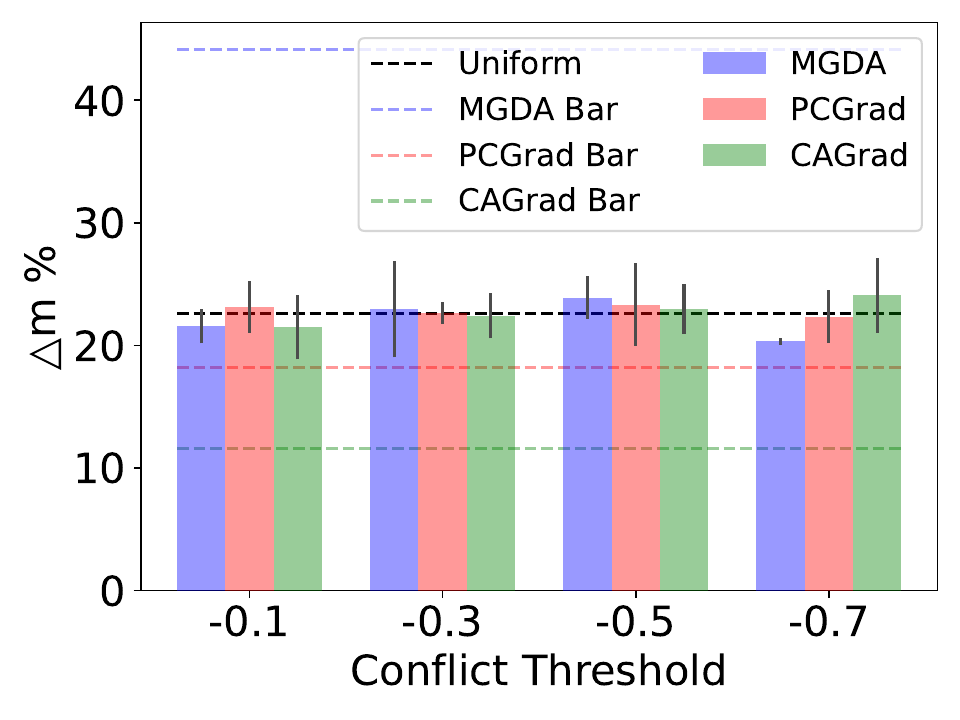}}
    \caption{Imbalance and conflict sensitivity examination.}
    \label{fig:benefit}
\end{figure}

Additionally, we have conducted a series of control group experiments to further support our findings. Similarly, we only apply optimization-based MTL when the gradient similarity falls below a certain threshold. As depicted in Figure~\ref{fig:benefit} (b), all baselines demonstrate relatively stable performance compared to those in (a) and fail to outperform the vanilla version, except for MGDA (which itself performs worse than LS). This outcome further reinforces the claim that imbalance matters more.

\section{Principal Design} \label{sec:design}
In this section, taking CAGrad as the baseline, we present the principal design of \texttt{IMGrad}, encompassing its formulation in the objective function and the practical implementation. And we provide convergence and speedup analysis in the \textbf{Appendix}.

\subsection{Injecting Imbalance-Sensitivity}
As a widely adopted baseline, CAGrad strikes a balance between Pareto property and globe convergence, and its \underline{dual objective} is formulated as follows:
% % Initial Objective
\begin{align} \label{eqn:cagrad}
    \underset{\bm{d} \in \mathbb{R}^m}{\rm max}\underset{\bm{\omega} \in \mathcal{W}}{\rm min}  \bm{g_{\omega}^{\top}}\bm{d}  \ \ \ \ {\rm s.t.} \left \| \bm{d} - \bm{g_0} \right \| \le c\left \| \bm{g_0} \right \|   
\end{align}
where $\bm{d}$ represents the combined gradient, while $\bm{g_0}$ denotes the averaged gradient, and $c$ is the hyper-parameter. 

To alleviate the imbalance-deduced Pareto failures or individual progress issue as illustrated in Figure~\ref{fig:sim_pro}, a logical approach is to maximize the projected norm of the combined gradient across all individuals. To achieve this, we incorporate a stronger constraint ($\bm{g_{i}^{\top}}\bm{d} - \left \| \bm{g_{i}} \right \|^2$) into Eqn.~\ref{eqn:cagrad}, which encourages projected norms that surpass individual norms. This formulation is reflected in our objective presented in Eqn.~\ref{eqn:new_form}, and subsequently, we derive the corresponding Lagrangian equations in Eqn.~\ref{eqn:lag}.
% approaches to the norm of the individuals themselves, i.e., $\bm{g_{i}^{\top}}\bm{d} \le \left \| \bm{g_{i}} \right \|^2$. Therefore, we incorporate this constraint into Eqn.~\ref{eqn:cagrad} to formulate our objective in Eqn.~\ref{eqn:new_form}, and subsequently derive the corresponding Lagrangian equations in Eqn.~\ref{eqn:lag}. 
% La Objective
\begin{align} 
     \underset{\bm{d} \in \mathbb{R}^m}{\rm max}\underset{\bm{\omega} \in \mathcal{W}}{\rm min}  \bm{g_{\omega}^{\top}}\bm{d} - \mu (\bm{g_{\omega}^{\top}}\bm{d} - \left \| \bm{g_{\omega}} \right \|^2 ) \ \ {\rm s.t.} \left \| \bm{d} - \bm{g_0} \right \| \le c\left \| \bm{g_0} \right \|\label{eqn:new_form} \\
    \underset{\bm{d} \in \mathbb{R}^m}{\rm max}\underset{\lambda \ge 0, \bm{\omega} \in \mathcal{W}}{\rm min}  \bm{g_{\omega}^{\top}}\bm{d}  - \lambda (\left \| \bm{d} - \bm{g_0} \right \|^2 - \phi)/2 \ \ \ \ \ \ \ \ \ \ \ \ \ \ \label{eqn:lag} \\
    - \mu (\bm{g_{\omega}^{\top}}\bm{d} - \left \| \bm{g_{\omega}} \right \|^2 ), \ \ \ \lambda > 0, \mu > 0\ \ \ \ \ \ \ \ \ \ \ \ \ \ \ \ \ \  \nonumber 
\end{align}

\begin{figure}
    \centering
    \includegraphics[width=\linewidth]{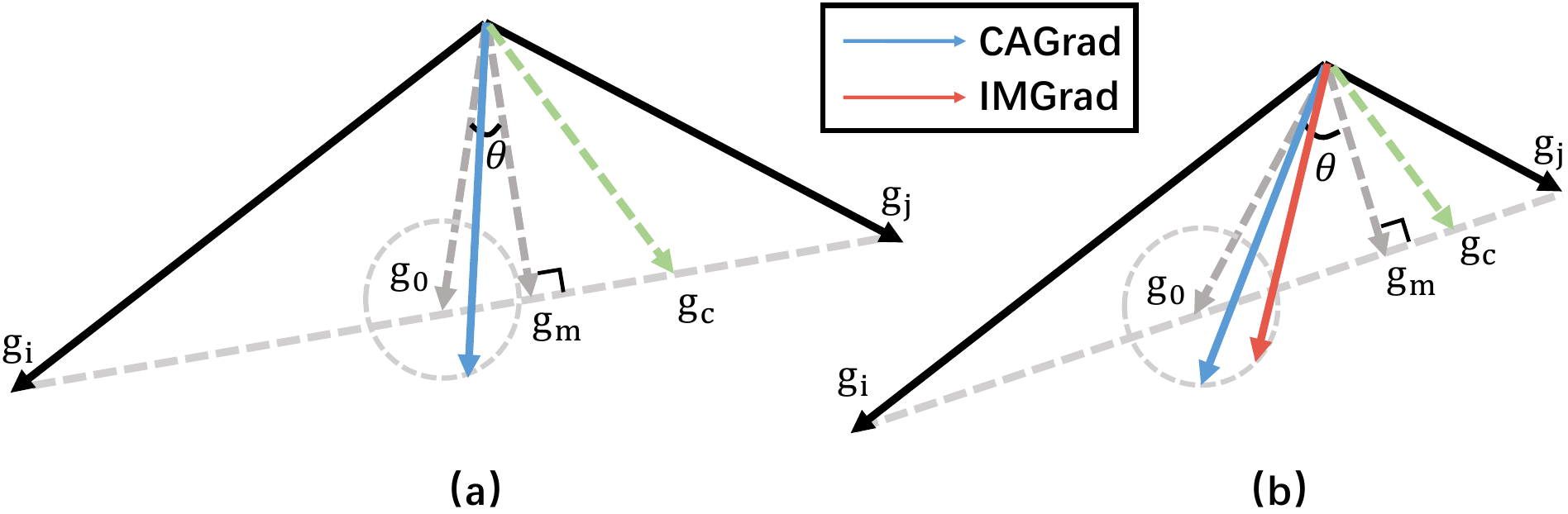}
    \caption{Multi-objective optimization Comparison between CAGrad and \texttt{IMGrad}. Here we suppose the angles between $\bm{g_i}$ and $\bm{g_j}$ in (a) and (b) are same. $\bm{g_m}$ can be obtained via MGDA.}
    \label{fig:cagrad_imgrad_illu}
\end{figure}

The strong duality property holds for the aforementioned objective, as supported by convex programming principles and the fulfillment of Slater's condition. Consequently, we interchange the positions of the minimum and maximum operators:
\begin{align} \label{obj:exg}
    \underset{\lambda \ge 0, \bm{\omega} \in \mathcal{W}}{\rm min} \underset{\bm{d} \in \mathbb{R}^m}{\rm max} 
    (1-\mu)\bm{g_{\omega}^{\top}}\bm{d} \\
    - \frac{\lambda}{2} (\left \| \bm{d} - \bm{g_0} \right \|^2 - \phi) + \mu  \left \| \bm{g_{\omega}} \right \|^2 \nonumber
\end{align}

\noindent With $\lambda, \omega$ fixing, the optimal $\bm{d}$ is achieved when $\bm{d} = \bm{g_0} + \frac{(1-\mu)\bm{g_{\omega}}}{\lambda}$. Substitude the optimal $\bm{d}$ into Eqn.~\ref{obj:exg}, yielding the following problem:
\begin{align}
    \underset{\lambda \ge 0, \bm{\omega} \in \mathcal{W}}{\rm min}
    (1-\mu)\bm{g_{\omega}^{\top}}\bm{g_0} + \mu  \left \| \bm{g_{\omega}} \right \|^2  \\
    + \frac{(1-\mu)^2}{2\lambda} \left \| \bm{g_{\omega}} \right \|^2 + \frac{\lambda}{2} \phi \nonumber
\end{align}
After optimizing out the $\lambda$ we have
\begin{align}
    \underset{\bm{\omega} \in \mathcal{W}}{\rm min} (1-\mu)\bm{g_{\omega}^{\top}}\bm{g_0} + \mu  \left \| \bm{g_{\omega}} \right \|^2 + (1 - \mu)\sqrt{\phi}\left \| \bm{g_{\omega}} \right \|
\end{align}
where $\lambda = (1 - \mu) \left \| \bm{g_{\omega}} \right \| / \phi^{1/2}$, and finally we have the optimization objective in Eqn.~\ref{eqn:final}. By solving this objective, we can obtain $\bm{g_{\omega}}$ and have $\bm{d} = \bm{g_0} + \frac{\phi^{1/2}}{\left \| \bm{g_{\omega}} \right \|}\bm{g_{\omega}}$.
\begin{align} \label{eqn:final}
    \underset{\bm{\omega} \in \mathcal{W}}{\rm min} (1-\mu)(\underbrace{\bm{g_{\omega}^{\top}}\bm{g_0} + \sqrt{\phi}\left \| \bm{g_{\omega}} \right \|}_{\rm \color{cvprblue} CAGrad}) + \mu \underbrace{ \left \| \bm{g_{\omega}} \right \|^2}_{\rm \color{cvprblue} \sim MGDA}
\end{align}
Upon careful examination of Eqn.~\ref{eqn:final}, it becomes evident that the final objective can be decomposed into two distinct components: CAGrad and MGDA. As depicted in Figure~\ref{fig:cagrad_imgrad_illu} (a), the gradient obtained by solving the practical objective in Eqn.~\ref{eqn:cagrad_prac}, denoted as $\bm{g_c}$ (represented by the green dotted line), predominantly resides within the region bounded by $\bm{g_m}$ and $\bm{g_j}$. However, in the case of an extreme imbalance scenario, as illustrated in Figure~\ref{fig:cagrad_imgrad_illu} (b), the corresponding $\bm{g_c}$ tends to lean towards the dominant gradient $\bm{g_i}$, thereby increasing the risk of conflicting with $\bm{g_j}$ and resulting in Pareto failures. When confronted with such a situation characterized by varying imbalances, it is desirable for $\mu$ to adaptively adjust $\bm{g_c}$ to consistently avoid Pareto failures while still promoting individual progress when the imbalance is less pronounced. Consequently, we establish a connection between $\mu$ and the gradient imbalances, effectively controlling the constraint ($\bm{g_{i}^{\top}}\bm{d} - \left \| \bm{g_{i}} \right \|^2$) adaptively based on the imbalance circumstances.

Multiple alternatives exist for quantifying the imbalance ratio among individuals~\footnote{Please refer to the \textbf{Appendix} for additional alternatives.}. We here choose to compute $\cos\theta$ to represent the imbalance ratio (see negative correlation between imbalance ratio and $\cos\theta$ in the \textbf{Appendix}), where $\theta$ denotes the angle between $\bm{g_0}$ and $\bm{g_m}$. As a result, Eqn.~\ref{eqn:final} can be re-written as:
% Take a close look at the Eqn.~\ref{eqn:final}, the final objective can mainly be decoupled as two components: CAGrad and MGDA. As illustrated in Figure~\ref{fig:cagrad_imgrad_illu} (a), the gradient obtained by solving Eqn.~\ref{eqn:cagrad} $\bm{g_c}$ (green dotted line) roughly lies in the area between $\bm{g_m}$ and $\bm{g_j}$. When it comes to the extreme imbalance scenario as illustrated in Figure~\ref{fig:cagrad_imgrad_illu} (b), the corresponding $\bm{g_c}$ would be leaning to the dominated one $\bm{g_i}$, and the obtained $\bm{d}$ is risking to be conflicting with $\bm{g_j}$, resulting in Pareto failures. When facing such an imbalance-varying circumstance, we expect $\mu$ could be adaptive to adjust $\bm{g_c}$ so that always avoid Pareto failures, and still promote the individual progress when less imbalanced. Therefore, we choose to build the connection between $\mu$ and imbalance of gradients. There are several alternatives to obtain the imbalance ratio among tasks~\footnote{Please refer to the \textbf{Appendix} for more alternatives.}, we choose to compute $\cos(\theta)$ to represent the imbalance status, where $\theta$ is the angle between $\bm{g_0}$ and $\bm{g_m}$. And Eqn.~\ref{eqn:final} can be re-wriiten as:
\begin{align} \label{eqn:rewritten}
    \underset{\bm{\omega} \in \mathcal{W}}{\rm min} (1-\cos\theta)(\bm{g_{\omega}^{\top}}\bm{g_0} + \sqrt{\phi}\left \| \bm{g_{\omega}} \right \|) + \cos\theta \left \| \bm{g_{\omega}} \right \|^2
\end{align}

%\subsection{Simplification and Scalability}
\noindent \textbf{Simplification}: As a matter of fact, CAGrad itself contains decoupled components in its \underline{practical objective}:
\begin{align} \label{eqn:cagrad_prac}
    \underset{\bm{\omega} \in \mathcal{W}}{\rm min} \underbrace{\bm{g_{\omega}^{\top}}\bm{g_0}}_{\rm \color{cvprblue} Push\ Away\ from\ \bm{g_0}} + \underbrace{\sqrt{\phi}\left \| \bm{g_{\omega}} \right \|}_{\rm \color{cvprblue} \sim MGDA}
\end{align}
where $\bm{g_{\omega}^{\top}}\bm{g_0}$ tends to push away from $\bm{g_0}$ and $\sqrt{\phi}\left \| \bm{g_{\omega}} \right \|$ plays the role of MGDA does. Thus we can simplify the Eqn.~\ref{eqn:rewritten} as:
\begin{align} \label{eqn:final_simp}
    \underset{\bm{\omega} \in \mathcal{W}}{\rm min} (1-\cos\theta)\bm{g_{\omega}^{\top}}\bm{g_0} + \cos\theta \sqrt{\phi} \left \| \bm{g_{\omega}} \right \|^2
\end{align}

\subsection{Augment Nash-MTL with Imbalance Sensitivity}
As stated in the previous \textbf{Pareto Failure} analysis, while Nash-MTL effectively addresses the imbalance issue and appears to be naturally conflict-averse, its implementation often leads to frequent Pareto failures. To address this problem, let's first examine its decoupled objective:
\begin{align}
    \underset{\omega}{\rm min} \underbrace{\sum_{i} \bm{g_i^{\top}}\bm{G}\bm{\omega}}_{\rm \color{cvprblue} Push\ Away\ from\ \bm{g_0}} + \underbrace{\varphi(\bm{\omega})}_{\rm \color{cvprblue} Strike\ balance\ among\ individuals}\\
    {\rm s.t.}\  \forall i, -\varphi_i(\bm{\omega}) \le 0,\ \ \omega_i > 0 \ \ \ \ \nonumber
\end{align}
where $\varphi_i(\bm{\omega}) = \log(\omega_i) + \log(\bm{g_i^{\top}}\bm{G}\bm{\omega}), \bm{G} = [\bm{g_1}, \bm{g_2}, ..., \bm{g_K}]$. $\sum_{i} \bm{g_i^{\top}}\bm{G}\bm{\omega}$ tends to push away from $\bm{g_0}$ and $\varphi(\bm{\omega})$ strives balance among individuals. Intuitively, we expect to preserve the Pareto property when encountering extremely imbalanced scenarios; therefore, $\sum_{i} \bm{g_i^{\top}}\bm{G}\bm{\omega}$ should be given more weight:
\begin{align}
    \underset{\omega}{\rm min}\ (1 - \cos\theta)\sum_{i} \bm{g_i^{\top}}\bm{G}\bm{\omega} + \cos\theta \varphi(\bm{\omega})
\end{align}
With the proper assumption of H-Lipschitz on gradients, we can still avoid Pareto failure with the derived weights among individuals from the last step. In a word, we augment Nash-MTL by injecting imbalance sensitivity to reduce Pareto failures. Please refer to the \textbf{Appendix} for more details. 

% \subsection{Convergence and Speedup}
% By introducing the adaptive factor `$\cos\theta$' to balance their decoupled objectives, the convergence of the methods is still ensured. Taking CAGrad as an example, the inequality in Eqn.~\ref{conv:cagrad} remains valid when $\mathcal{L}$ is lower bounded and the individual gradients $\nabla \mathcal{L}_i$ are all \textit{H-Lipschitz}. This property is independent of `$\rm cos\theta$' and indicates that the methods readily converge to the Pareto-stationary point.
% \begin{align} \label{conv:cagrad}
%     \min_{t \le T}\left \| \bm{\nabla \mathcal{L}} \right \|^2 \le \frac{1}{T+1}\sum_{t=0}^T \left \| \bm{\nabla \mathcal{L}} \right \|^2 \le \frac{2(\mathcal{L}_{0} - \mathcal{L}_{T+1})}{\alpha (1-c^2)(T+1)}    
% \end{align}

% While the primary focus of this paper does not revolve around achieving speedup, it is worth mentioning that general speedup approaches, such as sampling a subset of tasks~\cite{liu2021conflict}, utilizing feature-level gradients~\cite{sener2018multi}, and updating $\bm{\omega}$ periodically instead of every iteration~\cite{navon2022multi}, continue to be applicable to methods augmented with \texttt{IMGrad}. Further details and experimental results regarding these approaches are provided in the \textbf{Appendix}. 

\subsection{Implementation}
We implement our code with Python 3.8, PyTorch 1.4.0 and cvxpy 1.3.1, while all experiments are carried out on Tesla V100 GPUs. We follow the setting and general implementation of~\cite{liu2021conflict}, and the toy example generation is borrowed from~\cite{navon2022multi,senushkin2023independent}. See more implementation details in the \textbf{Appendix}.

\section{Performance Evaluation}  \label{sec:delta_m}
Following the evaluation protocol in~\cite{navon2022multi} and taking it as the baseline, we conduct experiments under the supervised learning and reinforcement learning scenarios. Specifically, two scene understanding and one image classification benchmarks are involved in supervised learning, and the classical MT10 benchmark is adopted for reinforcement learning. The examination of Pareto failures, individual task progress, a sensitivity analysis of $\mu$, the verification of negative correlation between imbalance ratio and $\cos \theta$, speed analysis, and more visualizations are also provided in the \textbf{Appendix}, please refer them for more details.

\noindent \textbf{Evaluation metric}. 
In addition to reporting individual performance, we also incorporate a widely used metric, $\Delta m\%$~\cite{maninis2019attentive}, which evaluates the overall degradation compared to independently trained models that are considered as the reference oracles. The formal definition of $\Delta m\%$ is given as: $\Delta m\% = \frac{1}{K}\sum_{k=1}^K (-1)^{\delta _k}(M_{m,k} - M_{b,k}) / M_{b,k}$. $M_{m,k}$ and $M_{b,k}$ represent the metric $M_k$ for the compared method and the independent model, respectively. The value of $\delta _k$ is assigned as 1 if a higher value is better for $M_k$, and 0 otherwise.

\subsection{Supervised Learning}
Customary evaluation in supervised learning for MTL involves assessing the ability of MTL approaches to handle multiple scene understanding and classification tasks. For scene understanding tasks, we follow previous studies~\cite{liu2021conflict,liutowards,navon2022multi} and employ a Multi-Task Attention Network (MTAN)\cite{liu2019end} as the fundamental architecture for all MTL approaches. Our experiments are conducted on two well-established datasets: NYUv2\cite{silberman2012indoor} and CityScapes~\cite{cordts2016cityscapes}. To ensure fair comparisons, we adopt the same training strategy as described in prior works~\cite{liu2021conflict,navon2022multi}. Specifically, models are trained for 200 epochs using the Adam optimizer, with an initial learning rate of 1e-4, which decays to 5e-5 after 100 epochs. For the image classification task, we utilize a 9-layer convolutional neural network (CNN) as the backbone, with linear layers serving as task-specific heads, and conduct experiments on CelebA~\cite{liu2015deep}. The model is trained using the Adam optimizer for 15 epochs, with an initial learning rate of 3.0e-4 and a batch size of 256.
% \usepackage{multirow}
\begin{table}[h]
\setlength\tabcolsep{3pt}
\centering
\caption{\textbf{Scene understanding} (\textit{CityScapes}, 2 tasks). We report MTAN model performance averaged over 3 random seeds.}
\label{table:city}
\footnotesize
\begin{tabular}{lllll|l}
\hline \toprule 
\multicolumn{1}{c}{\multirow{3}{*}{Method}} &
  \multicolumn{2}{c}{Segmentation} &
  \multicolumn{2}{c|}{Depth} &   
  \multicolumn{1}{c}{\multirow{3}{*}{$\Delta$ m\% $\downarrow$}} \\ \cmidrule(r){2-3} \cmidrule(r){4-5} %\cline{2-5}
\multicolumn{1}{c}{} &
  \multicolumn{2}{c}{\multirow{1}{*}{(Higher Better)}} &
  \multicolumn{2}{c|}{\multirow{1}{*}{(Lower Better)}} &
  \multicolumn{1}{c}{} 
\\ \cmidrule(r){2-3} \cmidrule(r){4-5}
\multicolumn{1}{c}{} &
  \multicolumn{1}{c}{mIoU} &
  Pix. Acc. &
  \multicolumn{1}{c}{Abs. Err.} &
  \multicolumn{1}{c|}{Rel. Err.} &
  \multicolumn{1}{c}{} \\ \cmidrule(r){1-6}  %\hline
Independent &
  \multicolumn{1}{c}{74.01}  &
                     93.16   &                  
  \multicolumn{1}{c}{0.0125} &
                     27.77  &                 
                     -
   \\ \cmidrule(r){1-6} 
LS &
  \multicolumn{1}{c}{75.18}  &
                     93.49   &
  \multicolumn{1}{c}{0.0155} &
                     46.77  &
                     22.60
   \\ %\hline
RLW &
  \multicolumn{1}{c}{74.57}  &
                     93.41   &
  \multicolumn{1}{c}{0.0158} &
                     47.79  &
                     24.37
   \\ %\hline
DWA &
  \multicolumn{1}{c}{75.24}  &
                     93.52   &
  \multicolumn{1}{c}{0.0160} &
                     44.37  &
                     21.43
   \\ %\hline
%Uncertainty &
%  \multicolumn{1}{c}{72.02}  &
%                     92.85   &
%  \multicolumn{1}{c}{0.0140} &
%                     \cellcolor{mygrey}30.13  &
%                     5.88
%   \\ %\hline
MGDA &
  \multicolumn{1}{c}{68.84}  &
                     91.54   &
  \multicolumn{1}{c}{0.0309} &
                     33.50  &
                     44.14
   \\ %\hline
% GradNorm &
%   \multicolumn{1}{c}{73.72}  &
%                      93.04   &
%   \multicolumn{1}{c}{\cellcolor{mygrey}0.0124} &
%                      34.11  &
%                      \cellcolor{mygrey}5.63
%    \\ %\hline
GradDrop &
  \multicolumn{1}{c}{75.27}  &
                     93.53   &
  \multicolumn{1}{c}{0.0157} &
                     47.54  &
                     23.67
   \\ %\hline
PCGrad &
  \multicolumn{1}{c}{75.13}  &
                     93.48   &
  \multicolumn{1}{c}{0.0154} &
                     42.07  &
                     18.21
\\
CAGrad &
  \multicolumn{1}{c}{75.16}  &
                     93.48   &
  \multicolumn{1}{c}{0.0141} &
                     37.60  &
                     11.58
   \\ %\hline
IMTL &
  \multicolumn{1}{c}{75.33}  &
                     93.49   &
  \multicolumn{1}{c}{0.0135} &
                     38.41  &
                     11.04
   \\  
Nash-MTL &
  \multicolumn{1}{c}{75.41}  &
                   \cellcolor{mygrey}93.66   &
  \multicolumn{1}{c}{0.0129} &
                     35.02  &
                     6.82
   \\  
MoCo &
  \multicolumn{1}{c}{\cellcolor{mygrey}75.42}  &
                     93.55   &
  \multicolumn{1}{c}{0.0149} &
                     34.19  &
                     9.90
   \\
FAMO &
  \multicolumn{1}{c}{74.54}  &
                     93.29   &
  \multicolumn{1}{c}{0.0145} &
                     \cellcolor{mygrey}32.59  &
                     8.13
   \\  
   \cmidrule(r){1-6}  %\hline
\texttt{IMGrad} &
  \multicolumn{1}{c}{75.13}  &
                   93.45   &
  \multicolumn{1}{c}{\cellcolor{mygrey}0.0128 } &
                     34.95  &
                     \cellcolor{mygrey}6.61 
   \\  
 \bottomrule \hline
\end{tabular}
\end{table}

% \usepackage{multirow}
\begin{table*}[h]
\centering
\setlength\tabcolsep{5pt}
\caption{\textbf{Scene understanding} (\textit{NYUv2}, 3 tasks). We report MTAN model performance averaged over 3 random seeds.} %More results and details please refer to the \textbf{Appendix}.}
\label{table:nyu}
%\footnotesize
\begin{tabular}{llllllllll|l}
\hline \toprule 
\multicolumn{1}{c}{\multirow{4}{*}{Method}} &
  \multicolumn{2}{c}{Segmentation} &
  \multicolumn{2}{c}{Depth} &
  \multicolumn{5}{c|}{Surface Normal} &   
  \multicolumn{1}{c}{\multirow{4}{*}{$\Delta$ m\% $\downarrow$}} \\ \cmidrule(r){2-3} \cmidrule(r){4-5} \cmidrule(r){6-10}
\multicolumn{1}{c}{} &
  \multicolumn{2}{c}{\multirow{2}{*}{(Higher Better)}} &
  \multicolumn{2}{c}{\multirow{2}{*}{(Lower Better)}} &
  \multicolumn{2}{c}{Angle Distance}  &
  \multicolumn{3}{c|}{Within $t^{\circ}$} &
  \multicolumn{1}{c}{} \\ \cmidrule(r){6-10}
\multicolumn{1}{c}{} &
  \multicolumn{2}{c}{} &
  \multicolumn{2}{c}{} &
  \multicolumn{2}{c}{(Lower Better)} &
  \multicolumn{3}{c|}{(Higher Better)} &
  \multicolumn{1}{c}{} \\ \cmidrule(r){2-3} \cmidrule(r){4-5} \cmidrule(r){6-10}
\multicolumn{1}{c}{} &
  \multicolumn{1}{c}{mIoU} &
  Pix. Acc. &
  \multicolumn{1}{c}{Abs Err} &
  Rel Err &
  \multicolumn{1}{c}{Mean} &
  \multicolumn{1}{c}{Median} &
  \multicolumn{1}{c}{11.25} &
  \multicolumn{1}{c}{22.5} &
  30 &
  \multicolumn{1}{c}{} \\ \cmidrule(r){1-11} %\hline
Independent &
  \multicolumn{1}{c}{38.30}  &
                     63.76   &                  
  \multicolumn{1}{c}{0.68} &
                     0.28  &                 
  \multicolumn{1}{c}{25.01}  &
  \multicolumn{1}{c}{19.21}  &
  \multicolumn{1}{c}{30.14}  &
  \multicolumn{1}{c}{57.20}  &
                     69.15   &
                     -
   \\ \cmidrule(r){1-11}
LS &
  \multicolumn{1}{c}{39.29}  &
                     65.33   &                  
  \multicolumn{1}{c}{0.55} &
                     0.23  &                 
  \multicolumn{1}{c}{28.15}  &
  \multicolumn{1}{c}{23.96}  &
  \multicolumn{1}{c}{22.09}  &
  \multicolumn{1}{c}{47.50}  &
                     61.08   &
                     \ 5.46    
   \\
RLW &
  \multicolumn{1}{c}{37.17}  &
                     63.77   &                  
  \multicolumn{1}{c}{0.58} &
                     0.24  &                 
  \multicolumn{1}{c}{28.27}  &
  \multicolumn{1}{c}{24.18}  &
  \multicolumn{1}{c}{22.26}  &
  \multicolumn{1}{c}{47.05}  &
                     60.62   &
                     \ 7.67
   \\
DWA &
  \multicolumn{1}{c}{39.11}  &
                     65.31   &                  
  \multicolumn{1}{c}{0.55} &
                     0.23  &                 
  \multicolumn{1}{c}{27.61}  &
  \multicolumn{1}{c}{23.18}  &
  \multicolumn{1}{c}{24.17}  &
  \multicolumn{1}{c}{50.18}  &
                     62.39   &
                     \ 3.49
   \\
% Uncertainty &
%   \multicolumn{1}{c}{36.87}  &
%                      63.17   &                  
%   \multicolumn{1}{c}{0.54} &
%                      0.23  &                 
%   \multicolumn{1}{c}{27.04}  &
%   \multicolumn{1}{c}{22.61}  &
%   \multicolumn{1}{c}{23.54}  &
%   \multicolumn{1}{c}{49.05}  &
%                      63.65   &
%                      \ 4.01
%    \\
MGDA &
  \multicolumn{1}{c}{30.47}  &
                     59.90   &
  \multicolumn{1}{c}{0.61} &
                     0.26  &
  \multicolumn{1}{c}{\cellcolor{mygrey}24.88}  &
  \multicolumn{1}{c}{\cellcolor{mygrey}19.45}  &
  \multicolumn{1}{c}{29.18}  &
  \multicolumn{1}{c}{\cellcolor{mygrey}56.88}  &
                     \cellcolor{mygrey}69.36   &   
                     \ 1.47
   \\ 
% GradNorm &
%   \multicolumn{1}{c}{20.09}  &
%                      52.06   &
%   \multicolumn{1}{c}{0.72} &
%                      0.28  &
%   \multicolumn{1}{c}{\cellcolor{mygrey}24.83}  &
%   \multicolumn{1}{c}{\cellcolor{mygrey}18.86}  &
%   \multicolumn{1}{c}{\cellcolor{mygrey}30.81}  &
%   \multicolumn{1}{c}{\cellcolor{mygrey}57.94}  &
%                      {\cellcolor{mygrey}69.73}   & 
%                      \ 7.22
%    \\ 
GradDrop &
  \multicolumn{1}{c}{39.39}  &
                     65.12   &
  \multicolumn{1}{c}{0.55} &
                     0.23  &
  \multicolumn{1}{c}{27.48}  &
  \multicolumn{1}{c}{22.96}  &
  \multicolumn{1}{c}{23.38}  &
  \multicolumn{1}{c}{49.44}  &
                     62.87   &
                     \ 3.61
   \\ %\hline
PCGrad &
  \multicolumn{1}{c}{38.06}  &
                     64.64   &
  \multicolumn{1}{c}{0.56} &
                     0.23  &
  \multicolumn{1}{c}{27.41}  &
  \multicolumn{1}{c}{22.80}  &
  \multicolumn{1}{c}{23.86}  &
  \multicolumn{1}{c}{49.83}  &
                     63.14   &
                     \ 3.83
   \\ %\hline
CAGrad &
  \multicolumn{1}{c}{39.79}  &
                     65.49   &
  \multicolumn{1}{c}{0.55} &
                     0.23  &
  \multicolumn{1}{c}{26.31}  &
  \multicolumn{1}{c}{21.58}  &
  \multicolumn{1}{c}{25.61}  &
  \multicolumn{1}{c}{52.36}  &
                     65.58   &
                     \ 0.29
   \\ 
IMTL &
  \multicolumn{1}{c}{39.35}  &
                     65.60   &
  \multicolumn{1}{c}{0.54} &
                     0.23  &
  \multicolumn{1}{c}{26.02}  &
  \multicolumn{1}{c}{21.19}  &
  \multicolumn{1}{c}{26.20}  &
  \multicolumn{1}{c}{53.13}  &
                     66.24   &
                     -0.59
   \\
Nash-MTL &
  \multicolumn{1}{c}{40.13}  &
                     65.93   &
  \multicolumn{1}{c}{0.53} &
                     0.22  &
  \multicolumn{1}{c}{25.26}  &
  \multicolumn{1}{c}{20.08}  &
  \multicolumn{1}{c}{28.40}  &
  \multicolumn{1}{c}{55.47}  &
                     68.15   &
                     -4.04
   \\  
MoCo &
  \multicolumn{1}{c}{\cellcolor{mygrey}40.30}  &
                     66.07   &
  \multicolumn{1}{c}{0.56} &
                     {\cellcolor{mygrey}0.21}  &
  \multicolumn{1}{c}{26.67}  &
  \multicolumn{1}{c}{21.83}  &
  \multicolumn{1}{c}{25.61}  &
  \multicolumn{1}{c}{51.78}  &
                     64.85   &
                     \ 0.16
   \\   
FAMO &
  \multicolumn{1}{c}{38.88}  &
                     64.90   &
  \multicolumn{1}{c}{0.55} &
                     0.22  &
  \multicolumn{1}{c}{25.06}  &
  \multicolumn{1}{c}{19.57}  &
  \multicolumn{1}{c}{\cellcolor{mygrey}29.21}  &
  \multicolumn{1}{c}{56.61}  &
                     68.98   &
                     -4.10
   \\
   \cmidrule(r){1-11}

\texttt{IMGrad} &
  \multicolumn{1}{c}{40.20 }  &
                     {\cellcolor{mygrey}66.19}  &
  \multicolumn{1}{c}{\cellcolor{mygrey}0.52 } &
                     0.22  &
  \multicolumn{1}{c}{25.15 }  &
  \multicolumn{1}{c}{19.94 }  &
  \multicolumn{1}{c}{28.69 }  &
  \multicolumn{1}{c}{55.80 }  &
                     68.44   &
                     {\cellcolor{mygrey}-4.57}
   \\
   \bottomrule \hline
\end{tabular}
\end{table*}
 
\noindent\textbf{NYUv2.} NYUv2 is a widely used indoor scene understanding dataset for MTL benchmarking, encompassing three tasks: semantic segmentation, depth estimation, and surface normal prediction. The results, presented in Table~\ref{table:nyu}, show that \texttt{IMGrad} surpasses the previous SOTA in terms of $\Delta m\%$, highlighting the effectiveness of incorporating imbalance sensitivity. Specifically, \texttt{IMGrad} achieves best performance on segmentation and depth tasks without much promise on other tasks. 

\noindent \textbf{CityScapes.} The CityScapes dataset is used for MTL evaluation, focusing on semantic segmentation and depth estimation tasks. Following the previous experimental setup, we utilize a coarser version that categorizes segmentation into 7 classes. The results, presented in Table~\ref{table:city}, indicate that \texttt{IMGrad} exhibits a similar trend to its performance on NYUv2 and achieves SOTA results in terms of $\Delta m\%$.

% \usepackage{multirow}
\begin{table}[h]
\setlength\tabcolsep{5pt}
\centering
\caption{\textbf{Reinforcement learning} (\textit{MT10}, 10 tasks) and \textbf{image classification} (\textit{CelebA}, 40-task).} %Average success rate on validation over 10 seeds.}
\label{table:mtrl}
\footnotesize
\begin{tabular}{l|l||l|l}
\hline \toprule 
\multicolumn{2}{c}{\textbf{MT10}} & \multicolumn{2}{c}{\textbf{CelebA}} \\
\cmidrule(lr){1-2} \cmidrule(lr){3-4} %\\
Method  &   Success $\pm$ SEM $\uparrow$  & Method & \textbf{$\Delta$m\% $\downarrow$}\\ \midrule 
LS                  &  0.49 $\pm$ 0.070  & LS & 4.15 \\
STL SAC &  0.90 $\pm$ 0.032  & SI & 7.20
   \\ \cmidrule(lr){1-2}
MTL SAC &   0.49 $\pm$ 0.073 & RLW & 1.46
   \\ %\hline
MH SAC &  0.54 $\pm$ 0.047 & DWA & 3.20
   \\ %\hline
SM &  0.73 $\pm$ 0.043  & UW & 3.23
   \\ %\hline
% Uncertainty & 0.77 $\pm$ 0.050 
% \\
CARE &  0.84 $\pm$ 0.051  & MGDA & 14.85
   \\ %\hline
PCGrad & 0.72 $\pm$ 0.022 & PCGrad & 3.17
   \\ %\hline
CAGrad &  0.83 $\pm$ 0.045 & CAGrad & 2.48
   \\ 
Nash-MTL & 0.91 $\pm$ 0.031 & Nash-MTL & 2.84
\\ 
FAMO & 0.83 $\pm$ 0.050 & FAMO & 1.21
\\ \midrule
\rowcolor{mygrey} \texttt{IMGrad} &   0.93 $\pm$ 0.068 {\color{cvprblue}(\textbf{+0.10})}  &  \texttt{IMGrad} &    1.27
   \\
 \bottomrule \hline
\end{tabular}
\end{table}  
\noindent \textbf{CelebA.} CelebA is a widely used face attributes dataset containing over 200,000 images annotated with 40 attributes. Recently, it has been adopted as a 40-task MTL benchmark to evaluate a model’s ability to handle a large number of tasks. The results, presented in Table~\ref{table:mtrl}, are averaged over three random seeds. While \texttt{IMGrad} does not achieve the best performance, it consistently ranks among the top methods, underscoring the importance of imbalance sensitivity for learning massive tasks.

\subsection{Reinforcement Learning}
Reinforcement learning is another domain where MTL is often essential, as it seeks to acquire a policy capable of succeeding across various manipulation tasks. To evaluate the generalizability of our proposed method, we use CAGrad as the baseline and conduct experiments on the MT10 environment from the Meta-World benchmark~\cite{yu2020meta}. The results, presented in Table~\ref{table:mtrl}, report the average success rate on the validation set over 10 random seeds. Consistent with the improvements observed in supervised learning evaluations, \texttt{IMGrad} enhances CAGrad by over 0.10, achieving SOTA performance on this benchmark. It is worth noting that Nash-MTL does not provide an official implementation for reinforcement learning benchmarks. As a result, we did not augment it for evaluation in this context.  

\section{Conclusion}
In this paper, we begin by empirically demonstrating the significance of addressing the imbalance issue in optimization-based MTL. We assert that incorporating imbalance-sensitivity is crucial for avoiding Pareto failures and promoting balanced individual progress. Building upon this motivation, we propose \texttt{IMGrad}, a method derived from a projection norm constraint, which is further simplified as an adaptive balancer between decoupled objectives. Through extensive experiments, we validate the effectiveness of our proposed approach. We believe that our explicit emphasis on the imbalance issue, rather than the conflict issue, provides valuable insights for the future development of optimization-based MTL.

\bibliographystyle{named}
\bibliography{ijcai24} 

\end{document}

\maketitle

\section{Implementation Details} \label{sec:imp}

\subsection{MTL Baselines}
\noindent \textbf{CAGrad}: CAGrad strikes a balance between Pareto optimality and globe convergence by regulating the combined gradients in proximity to the average gradient:
\begin{align} \label{app_eqn:cagrad}
    \underset{\bm{d} \in \mathbb{R}^m}{\rm max}\underset{\bm{\omega} \in \mathcal{W}}{\rm min}  \bm{g_{\omega}^{\top}}\bm{d}  \ \ \ \ {\rm s.t.} \left \| \bm{d} - \bm{g_0} \right \| \le c\left \| \bm{g_0} \right \|   
\end{align}

In our experimental setup, we utilize the official code~\footnote{\url{https://github.com/Cranial-XIX/CAGrad}} that encompasses the implementations of MGDA, PCGrad, and CAGrad. Additionally, we extend this implementation to include \texttt{IMGrad}. For more comprehensive information, please consult the official implementation.

\noindent \textbf{Nash-MTL}: Nash-MTL provides the individual progress guarantee via the following objective:
\begin{align} \label{eqn:nash}
    \min_{\bm{\omega}} \sum_i \beta_i(\bm{\omega}) + \varphi_i(\bm{\omega}) \\
    s.t. \forall i, - \varphi_i(\bm{\omega}) \le 0, \ \ \ \omega_i > 0.
\end{align}
where $\varphi_i(\bm{\omega}) = \log(\omega_i) + \log(\bm{g_i^{\top}}\bm{G}\bm{\omega}), \ \bm{G} = [\bm{g_1}, \bm{g_2}, ..., \bm{g_K}]$. As demonstrated, the individual progress is ensured through the projection, subject to the constraint $\beta_i = \bm{g_i^{\top}}\bm{G}\bm{\omega} \ge \frac{1}{\omega_i}$. \uline{However, in cases where conflicting individuals arise, such as when $\bm{g_i}$ conflicts with $\bm{g_j}$, the official implementation of Nash-MTL opts to skip the current step, which poses a potential risk in handling these conflicting individuals.} It is important to acknowledge that the experiments conducted on CityScapes in the main text regarding Nash-MTL were performed using our re-implementation, as the official version does not offer the corresponding implementation. 

\noindent \textbf{\underline{Our Extension to Nash-MTL}}:
As previously stated, Nash-MTL opts to skip the current step and retain the weights from the previous step only when negative terms exist in $\bm{g_i^{\top}}\bm{G}$, indicating a conflict issue in the current step. Moreover, the co-existence of both imbalance and conflict issues necessitates a specialized design for MTL. By assuming H-Lipschitz continuity on gradients, we expect a similar imbalance status between consecutive steps. To address the conflict issue, we incorporate imbalance sensitivity, which allows greater weighting on $\sum_{i} \bm{g_i^{\top}}\bm{G}\bm{\omega}$.

\subsection{Datasets}
\noindent \textbf{NYUv2}: NYUv2 is an indoor scene dataset comprising 1449 RGBD images, accompanied by dense per-pixel labeling encompassing 13 classes. This dataset is commonly employed for tasks such as semantic segmentation, depth estimation, and surface normal prediction.

\noindent \textbf{CityScapes}: CityScapes is a renowned benchmark dataset for multi-task learning, comprising 5000 high-resolution street-view images accompanied by dense per-pixel labels for tasks such as semantic segmentation and depth estimation. In accordance with the previous configuration~\cite{liu2021conflict}, we resized the images to 128*256 dimensions prior to inputting them into the model, aiming to enhance computational efficiency.

\noindent \textbf{MT10}: MT10 is a widely recognized benchmark for multi-task reinforcement learning, encompassing 10 robot manipulation tasks. For visual references, please consult~\citep{liu2021conflict}. Our implementation builds upon the official implementation of CAGrad and leverages the well-established MTRL environment~\footnote{\url{https://github.com/facebookresearch/mtrl}}.

\subsection{Synthetic Examples}
We adopt the toy optimization example introduced in~\citep{liu2021conflict,navon2022multi}, which involves two objectives that consider the variables $\bm{\vartheta} = (\vartheta_1, \vartheta_2)$:
\begin{small}
\begin{align}
    \mathcal{L}_1(\bm{\vartheta}) = c_1(\bm{\vartheta})f_1(\bm{\vartheta}) + c_2(\bm{\vartheta})g_1(\bm{\vartheta})\ \ \ \ \ \ \ \ \ \ \ \ \ \ \ \ \ \ \ \ \ \ \ \\
    \mathcal{L}_2(\bm{\vartheta}) = c_1(\bm{\vartheta})f_2(\bm{\vartheta}) + c_2(\bm{\vartheta})g_2(\bm{\vartheta})\ \ \ \ \ \ \ {\rm where}  \ \ \ \ \ \  \nonumber \\
    f_1(\bm{\vartheta}) = \log (\max(| 0.5(-\vartheta_1 -7) - \tanh(\vartheta_2) |, 5.e-6)) + 6, \nonumber  \ \ \ \ \ \ \\
    f_2(\bm{\vartheta}) = \log (\max(| 0.5(-\vartheta_1 +3) - \tanh(\vartheta_2) + 2|, 5.e-6)) + 6, \nonumber  \\ 
    g_1(\bm{\vartheta}) = ((-\vartheta_1 + 7)^2 + 0.1 * (-\vartheta_2-8)^2)/10 - 20, \ \ \ \ \ \ \ \ \ \nonumber  \\
    g_2(\bm{\vartheta}) = ((-\vartheta_1 - 7)^2 + 0.1 * (-\vartheta_2-8)^2)/10 - 20, \ \ \ \ \ \ \ \ \ \nonumber  \\    
    c_1(\bm{\vartheta}) = \max (\tanh(0.5*\vartheta_2), 0), \ \ \ \ \ \ \ \ \ \ \ \ \ \ \ \ \ \ \ \ \ \ \ \ 
 \nonumber \\
    c_2(\bm{\vartheta}) = \max(\tanh(-0.5*\vartheta_2), 0). \ \ \ \ \ \ \ \ \ \ \ \ \ \ \ \ \ \ \ \ \ \nonumber
\end{align}
\end{small}

The initial points in Figure 2 of the main text are $\{ (-8.5, 7.5), (-8.5, 5), (0, 0), (9, 9), (10, -8)\}$, while those in Figure 4 are $\{(-8.5, 7.5), (0, 8), (5, 9)\}$.

\section{Convergence Analysis} \label{app_sec:conv}
 Let's first recall the convergence analysis of CAGrad in Theorem~\ref{thm:2}.
\begin{theorem}(Convergence of CAGrad) \label{thm:2}
With a fix step size $\alpha$ and the assumption of H-Lipschitz ($0<H\le 1/\alpha$) on gradients, i.e., $\left \| \nabla  \mathcal{L}_i(\bm{\theta}) - \nabla \mathcal{L}_i(\bm{\theta'}) \right \| \leq H \left \|  \bm{\theta} - \bm{\theta'} \right \|$ for i = 1, 2, ..., K. Denote $\bm{d^*(\theta_t)}$ as the optimization direction of \texttt{IMGrad} at step $t$, then we have:

(1) If $0\le c< 1$, then CAGrad converges to stationary points of $\mathcal{L_}_0$ convergence rate in that
\begin{align}
    \sum_{t=0}^{T}\left \| \bm{g_0(\theta_t)} \right \|^2 \le \frac{2(\mathcal{L}(\bm{\theta_{0}}) - \mathcal{L}(\bm{\theta_{T+1}}))}{\alpha(1-c^2)}
\end{align}
\end{theorem}
 
\begin{proof}
\begin{align*}
    \mathcal{L}(\bm{\theta_{t+1}}) &- \mathcal{L}(\bm{\theta_{t}})  = \mathcal{L}(\bm{\theta_{t}} -\alpha \bm{d^*(\theta_t)}) - \mathcal{L}(\bm{\theta_{t}}) \\
    &\leq -\alpha \bm{g_0(\theta_t)}^\top \bm{d^*(\theta_t)} + \frac{H\alpha^2}{2}\left \| \bm{d^*(\theta_t)} \right \|^2    \\
    &\leq -\alpha \bm{g_0(\theta_t)}^\top \bm{d^*(\theta_t)} + \frac{\alpha}{2}\left \| \bm{d^*(\theta_t)} \right \|^2 \\ &\ \ \ \ \ {\color{magenta} //\alpha \le 1/H}    \nonumber  \\
    &\leq -\frac{\alpha}{2}\left (  \left \| \bm{g_0(\theta_t)} \right \|^2 + \left \| \bm{d^*(\theta_t)} \right \|^2 \\&- \left \| \bm{g_0(\theta_t)}  - \bm{d^*(\theta_t)}\right \|^2 \right ) + \frac{\alpha}{2}\left \| \bm{d^*(\theta_t)} \right \|^2  \nonumber  \\
    &= - \frac{\alpha}{2} \left ( \left \| \bm{g_0(\theta_t)} \right \|^2 - \left \| \bm{g_0(\theta_t)}  - \bm{d^*(\theta_t)}\right \|^2 \right )    \\
    &\leq - \frac{\alpha}{2}(1-c^2)\left \| \bm{g_0(\theta_t)} \right \|^2 \\ &\ \ \ \ \ {\color{magenta} // {\rm CAGrad's\ Bound}} \nonumber
\end{align*}
Using telescoping sums, we have $\mathcal{L}(\bm{\theta_{t+1}}) - \mathcal{L}(0) = -(\alpha\beta/2)(1-c^2)\sum_{t=0}^T\left \| \bm{g_m(\theta_t)} \right \|^2$. Thus, we further have
\begin{align}
    \underset{t \le T}{\rm min} \left \| \bm{g_0(\theta_t)} \right \|^2 &\le \frac{1}{T+1}\sum_{t=0}^{T}\left \| \bm{g_0(\theta_t)} \right \|^2 \nonumber \\ &\le \frac{2(\mathcal{L}(\bm{\theta_{0}}) - \mathcal{L}(\bm{\theta_{T+1}}))}{\alpha(1-c^2)(T+1)}
\end{align}
Generally, if $\mathcal{L}$ is lower bounded, then $\underset{t \le T}{\rm min} \left \| \bm{g_0(\theta_t)} \right \|^2 = \mathcal{O}(1/T)$. Besides, the upper bound of $\underset{t \le T}{\rm min} \left \| \bm{g_0(\theta_t)} \right \|^2$ is harnessed by $\beta$, which is positively related to the imbalance ratio of individuals. 
\end{proof}

Take a close at the above convergence analysis, we can find that the proof only leverages the property of $\left \| \bm{g_0(\theta_t)}  - \bm{d^*(\theta_t)}\right \| \le c\left \| \bm{g_0(\theta_t)}\right \|$, ignoring the solved $\bm{g_{\omega}}$, which is essential for mitigating Pareto failures and imbalanced individual progress claimed in the main text. Hence we re-organize the proof as following:
\begin{proof}
\begin{align*}
    \mathcal{L}(\bm{\theta_{t+1}}) &- \mathcal{L}(\bm{\theta_{t}}) = \mathcal{L}(\bm{\theta_{t}} -\alpha \bm{d^*(\theta_t)}) - \mathcal{L}(\bm{\theta_{t}}) \\
    &\leq -\alpha \bm{g_0(\theta_t)}^\top \bm{d^*(\theta_t)} + \frac{H\alpha^2}{2}\left \| \bm{d^*(\theta_t)} \right \|^2    \\
    &\leq -\alpha \bm{g_0(\theta_t)}^\top \bm{d^*(\theta_t)} + \frac{\alpha}{2}\left \| \bm{d^*(\theta_t)} \right \|^2 \\ &\ \ \ \ \ {\color{magenta} //\alpha \le 1/H} \nonumber \\
    & = - \alpha(1 + c\cos\varphi)\left \| \bm{g_0(\theta_t)} \right \|^2 + \frac{\alpha}{2}\left \| \bm{g_0(\theta_t)} \right \|^2 \\ &+ \frac{\alpha}{2}\left \| \bm{d^*(\theta_t)} \right \|^2 - \frac{\alpha}{2}\left \| \bm{g_0(\theta_t)} \right \|^2 \\
    &\leq - \alpha(1 + c\cos\varphi)\left \| \bm{g_0(\theta_t)} \right \|^2 + \frac{\alpha}{2}\left \| \bm{g_0(\theta_t)} \right \|^2 \\ &+ \frac{\alpha}{2}\left \| \bm{d^*(\theta_t)} -\bm{g_0(\theta_t)} \right \|^2 \\
    & = - \alpha(1 + c\cos\varphi)\left \| \bm{g_0(\theta_t)} \right \|^2 + \frac{\alpha}{2}\left \| \bm{g_0(\theta_t)} \right \|^2 \\ &+ \frac{\alpha c^2}{2}\left \| \bm{g_0(\theta_t)} \right \|^2 \\
    & = - \frac{\alpha}{2}(1-c^2 + 2c\cos\varphi)\left \| \bm{g_0(\theta_t)} \right \|^2 \\ &\ \ \ \ \ {\color{magenta} // {\rm Our\ Bound}}
\end{align*} 
where $\varphi$ is the angle between $\bm{g_0(\bm{\theta_t})}$ and $\bm{g_{\omega}(\bm{\theta_t})}$. And it should noted that although $\left \| \bm{g_0(\theta_t)}  - \bm{d^*(\theta_t)}\right \| \le c\left \| \bm{g_0(\theta_t)}\right \|$ is constrained in CAGrad, the practical implementation always satisfied $\left \| \bm{g_0(\theta_t)}  - \bm{d^*(\theta_t)}\right \| = c\left \| \bm{g_0(\theta_t)}\right \|$, thus we directly apply it in the above analysis. According to the final derived equation, the convergence rate of CAGrad is harnessed by $\cos \varphi$. Specifically, when all individuals are balanced, i.e., $\cos \varphi = 1$, then our bound degrades to the CAGrad's bound. 

In particular, we provide a tighter bound when $0^{\circ } < \varphi \le 90^{\circ }$, but a looser bound when $\varphi \ge 90^{\circ }$ (extreme imbalance). However, the percentage of extreme cases is very small (see Figure 5(b) in the main text) and therefore does not affect its convergence generally. Besides, according to the simplified objective in  Eqn. 11 in the main text, \texttt{IMGrad} tends to choose a larger $\varphi$, leading to a slower convergence to guarantee a more balanced individual progress.

\end{proof}
\section{Additional Experiments} \label{sec:add_exp}

% \subsection{Results on MT50}
% \input{tables/mt50}

\subsection{Sensitivity of $\mu$} \label{subsec:mu}
One concern arises regarding whether $\mu\ ({\rm cos}\theta)$ effectively adjusts the combined gradient. In other words, does $\mu$ exhibit sufficient sensitivity to address varying imbalance circumstances? To address this question, we provide statistics of $\mu$ for both NYUv2 and CityScapes datasets, which are presented in Figure~\ref{fig:mu}. The results demonstrate that $\mu$ exhibits a wide range of values that effectively cater to the manipulation requirements.
\begin{figure}[h]
    \centering
    \subfloat[NYUv2]{\includegraphics[width = 0.24\textwidth]{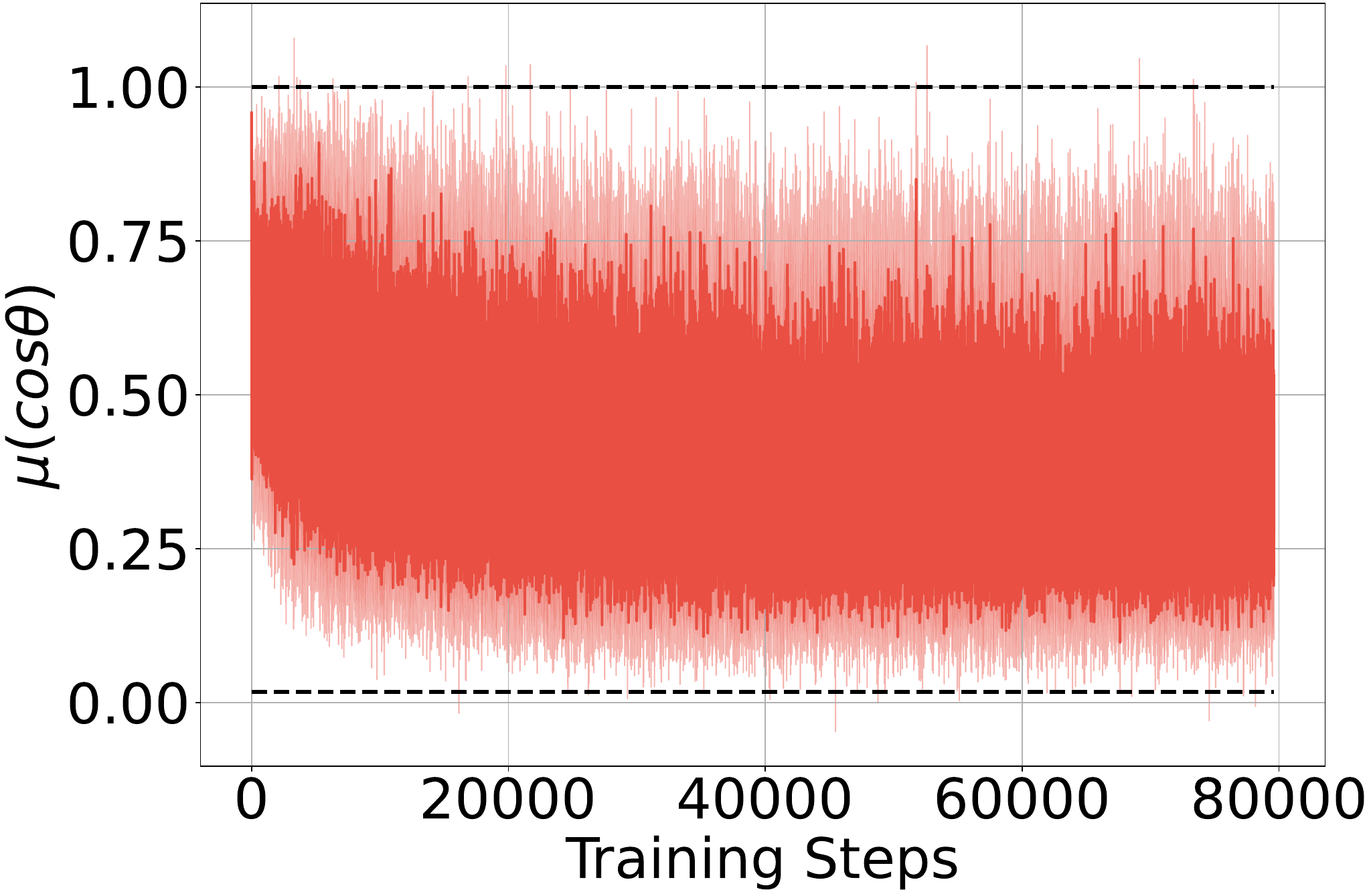}}
    \subfloat[CityScapes]{\includegraphics[width = 0.24\textwidth]{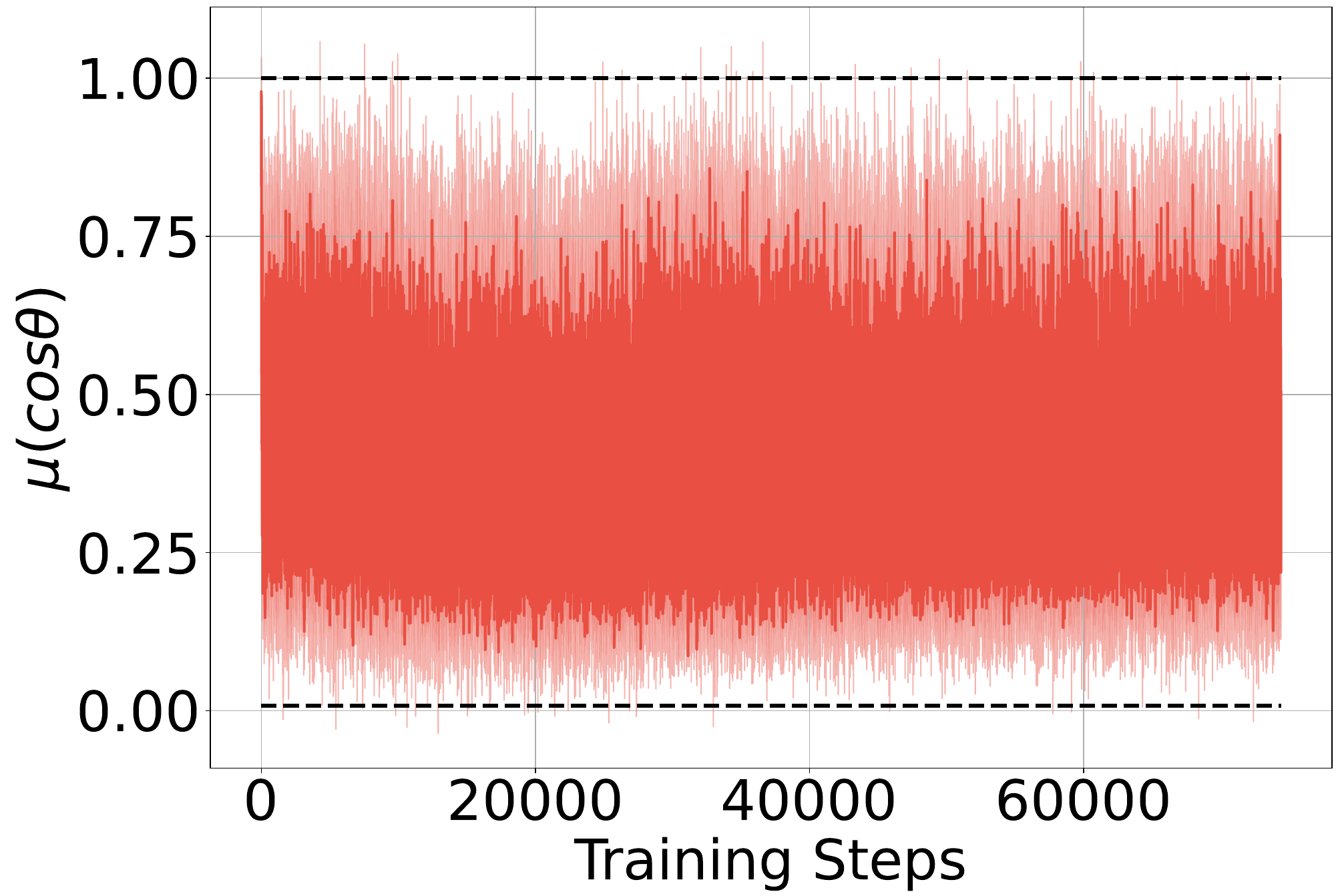}}
    \caption{Sensitivity examinations of $\mu$.}
    \label{fig:mu}
\end{figure}

% \usepackage{multirow}
\begin{table*}[h]
\centering
\caption{$\mu$ alternatives comparison on CityScapes. The GPU time reported here refers to the cumulative duration of the entire run.}
\label{table:alter}
%\scriptsize
\begin{tabular}{lllllll}
\hline
\multicolumn{1}{c}{\multirow{4}{*}{Method}} &
  \multicolumn{2}{c}{Segmentation} &
  \multicolumn{2}{c}{Depth} &   
  \multicolumn{1}{c}{\multirow{4}{*}{$\Delta$ m\% $\downarrow$}} & 
  \multicolumn{1}{c}{\multirow{4}{*}{GPU Time (h) $\downarrow$}} \\ \cline{2-5}
\multicolumn{1}{c}{} &
  \multicolumn{2}{c}{\multirow{2}{*}{(Higher Better)}} &
  \multicolumn{2}{c}{\multirow{2}{*}{(Lower Better)}} &
  
% \multicolumn{1}{c}{} &
%  \multicolumn{2}{c}{} &
  \multicolumn{2}{c}{} &
  \multicolumn{1}{c}{} 
\\ \cline{2-5}
\multicolumn{1}{c}{} &
  \multicolumn{1}{c}{mIoU} &
  Pix. Acc. &
  \multicolumn{1}{c}{Abs Err} &
  Rel Err &
  \multicolumn{1}{c}{} \\ \hline
CAGrad &
  \multicolumn{1}{c}{75.16}  &
                     93.48   &
  \multicolumn{1}{c}{0.0141} &
                     37.60  &
                     11.58   &
                     13.83
   \\ \hline
DC &
  \multicolumn{1}{c}{74.20}  &
                     93.36   &
  \multicolumn{1}{c}{0.0132} &
                     34.27  &
                     7.23   &
                     \cellcolor{mygrey}16.90
   \\ %\hline
LM &
  \multicolumn{1}{c}{\cellcolor{mygrey}74.85}  &
                    \cellcolor{mygrey}93.47   &
  \multicolumn{1}{c}{0.0132} &
                     36.56  &
                     8.83     &
                     17.40
   \\ %\hline
PM &
  \multicolumn{1}{c}{74.84}  &
                     93.39   &
  \multicolumn{1}{c}{\cellcolor{mygrey}0.0132} &
                    \cellcolor{mygrey}33.99  &
                    \cellcolor{mygrey}6.58   &
                    20.25
   \\ %\hline
   \hline
\end{tabular}
\end{table*}
\subsection{Verification of Negative Correlation between Imbalance Ratio and \cos$\theta$} \label{subsec:corre}

To justify our choice of using \cos$\theta$ as an alternative for the imbalance ratio in implementing \texttt{IMGrad}, it is necessary to establish their negative correlation. To demonstrate this, we track the imbalance ratio and \cos$\theta$ while running CAGrad on NYUv2 and CityScapes datasets, presenting their statistical results in Figure~\ref{fig:corre}. The results indicate a positive correlation between $\rm 1 / Imb.$ and \cos$\theta$ during optimization, regardless of whether it involves two (CityScapes) or three (NYUv2) tasks. Additionally, we observe that the values of \cos$\theta$ are confined within the range of [0, 1], while the imbalance ratio can vary significantly. This discrepancy in stability makes \cos$\theta$ a more suitable choice for modulation.

\begin{figure}[h]
    \centering
    \subfloat[NYUv2]{\includegraphics[width = 0.23\textwidth]{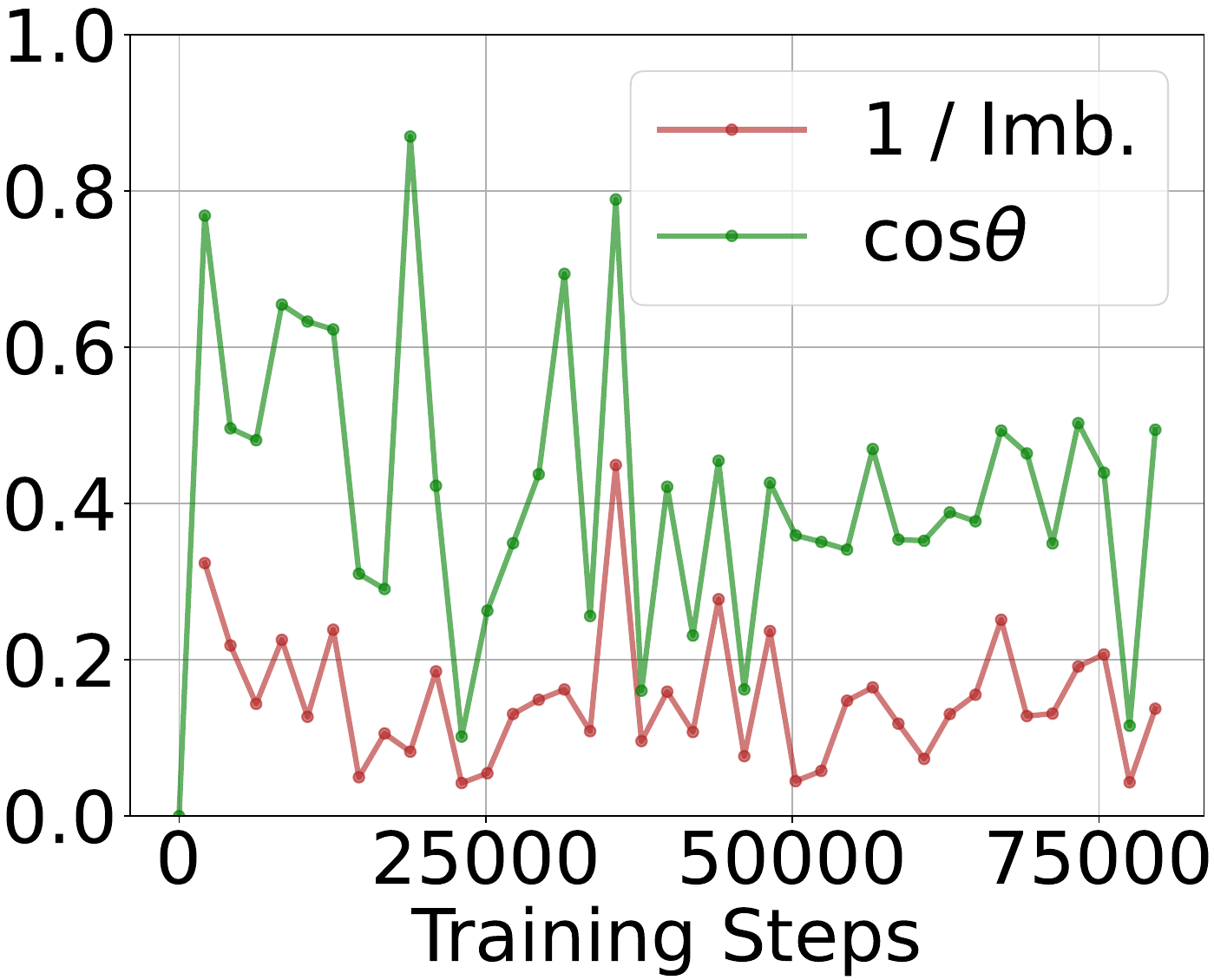}}
    \subfloat[CityScapes]{\includegraphics[width = 0.24\textwidth]{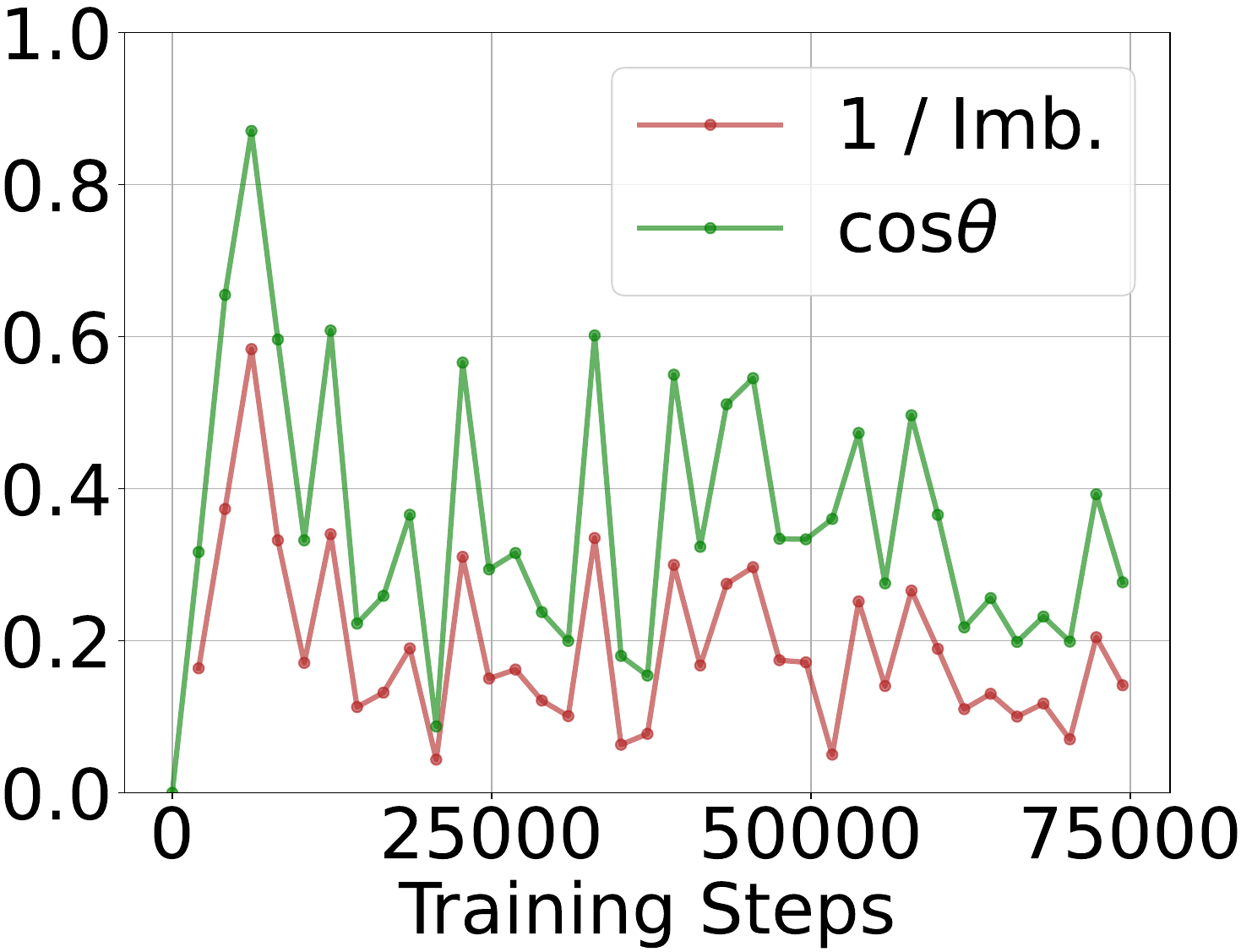}}
    \caption{Negative correlation examination between imbalance ratio and \cos$\theta$.}
    \label{fig:corre}
\end{figure}

\subsection{Alternatives of $\mu$} \label{subsec:alter}
There are several alternatives to quantify the imbalance ratio among tasks, we provide the following choices:
\begin{itemize}
    \item \textbf{\underline{Direct Calculation (DC)}}: DC is the standard operation to obtain the imbalance ratio by directly calculate the their norm ratio, i.e., $\left | \bm{g_i} \right \| / \left | \bm{g_j} \right \|$. Specifically, to control the range of $\mu$, we adopt $\mu = 1 / (1 + \log(\left | \bm{g_i} \right \| / \left | \bm{g_j} \right \|))$.
    \item \textbf{\underline{Least Mean (LM)}}: LM quantifies the imbalance ratio by calculating the cosine similarity between $\bm{g_0}$ and the individual with the smallest norm. The utilization of $\theta$ in LM is depicted in Figure~\ref{fig:mu_illu} (b).
    \item \textbf{\underline{Plumbline Mean (PM)}}: PM is just the operation that adopted in the main text, and is illustrated in Figure~\ref{fig:mu_illu} (a).  %just as the main text described 
\end{itemize}

The corresponding results are presented in Table~\ref{table:alter}, revealing that PM exhibits the most favorable overall performance. However, LM demonstrates competitiveness in the context of semantic segmentation tasks, while DC shows moderate performance. This outcome can be attributed to the wide range of imbalance ratios, which makes it challenging to achieve adaptive balancing of the decoupled objectives. The inferior performance of LM can be attributed to the variation of $\cos\theta$, which can range between -1 and 1, making it difficult to tune hyper-parameters effectively. 

\begin{figure}[h]
    \centering
    \includegraphics[width = 0.48\textwidth]{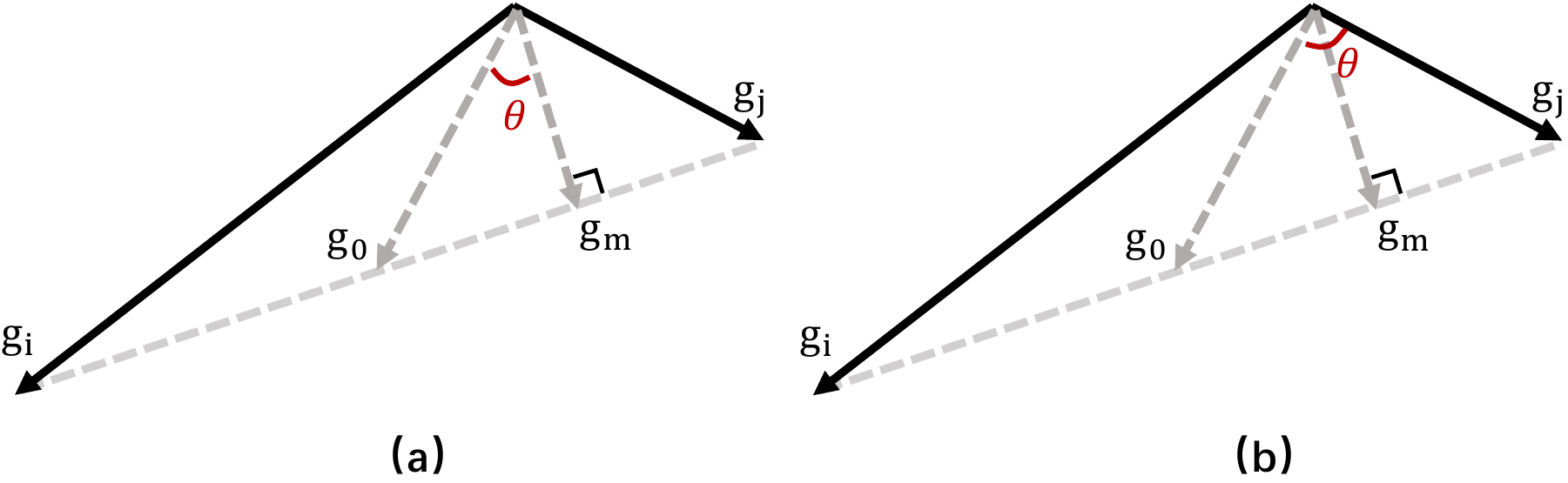}
    \caption{The illustration of different choice of $\mu$. (a) represents the PM, while represents the LM.}
    \label{fig:mu_illu} 
\end{figure}
\begin{figure}[h]
    \centering
    \includegraphics[width = 0.4\textwidth]{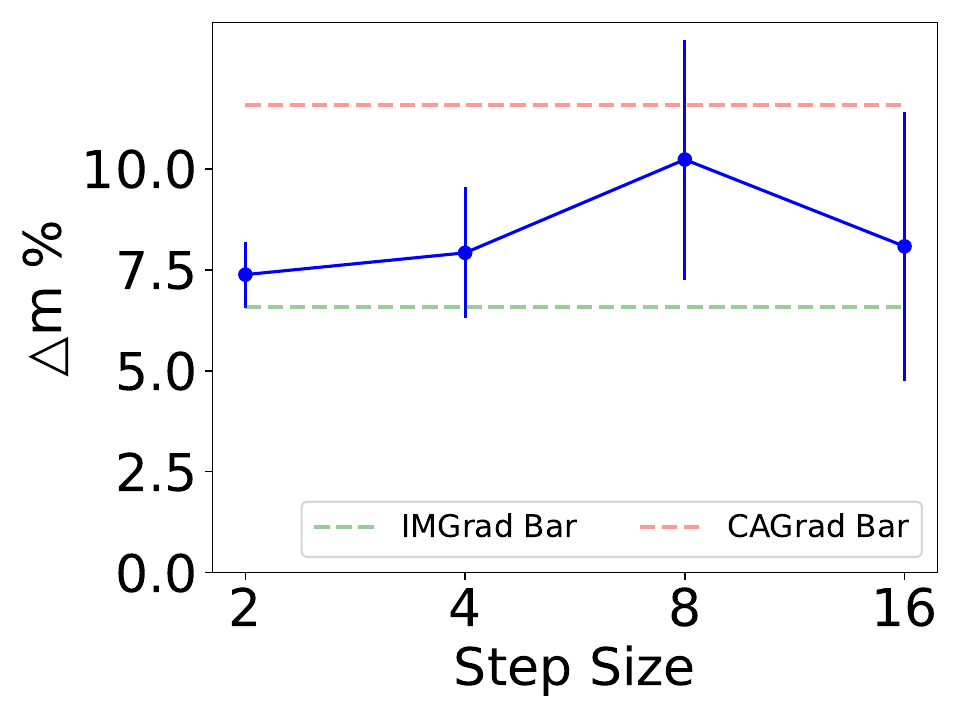}
    \caption{The impact of step size to the overall performance.}
    \label{fig:step_size}
\end{figure}

\begin{figure}[h]
    \centering
    \subfloat[MGDA]{\includegraphics[width = 0.235\textwidth]{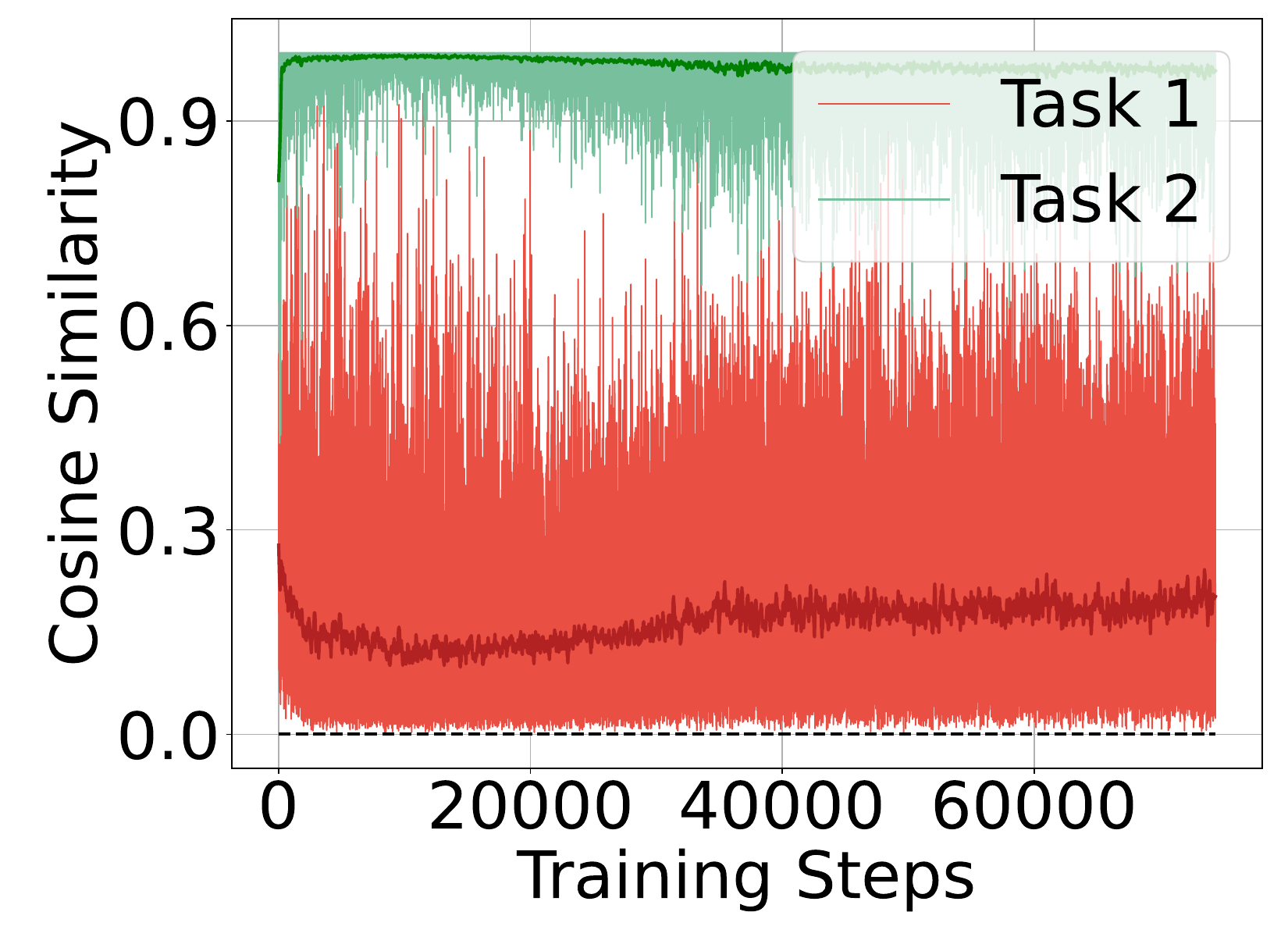}}
    \hfill
    \subfloat[GradDrop]{\includegraphics[width = 0.235\textwidth]{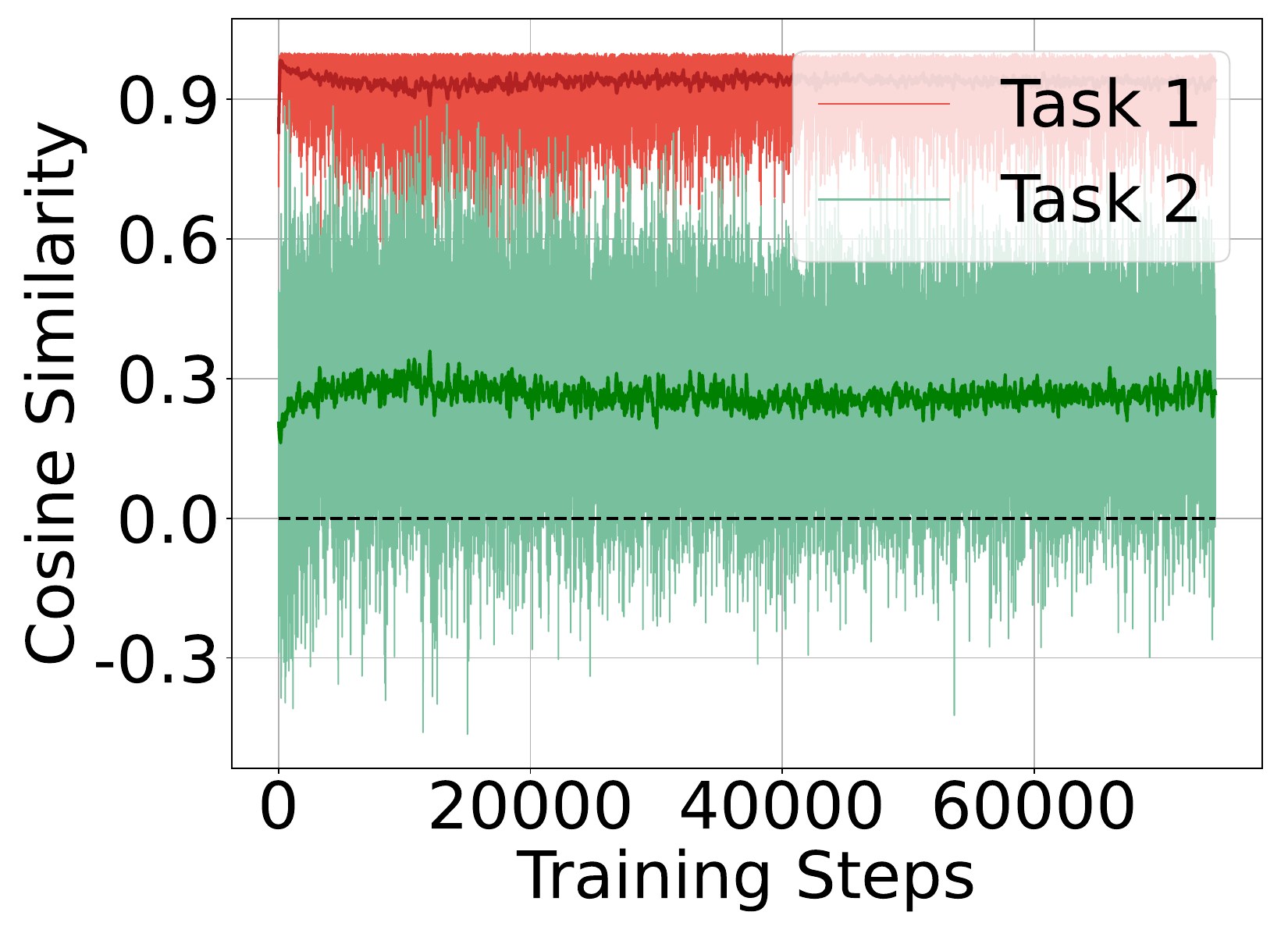}}
    \hfill 
    \subfloat[MGDA ($\Delta m\% = 17.69$)]{\includegraphics[width = 0.235\textwidth]{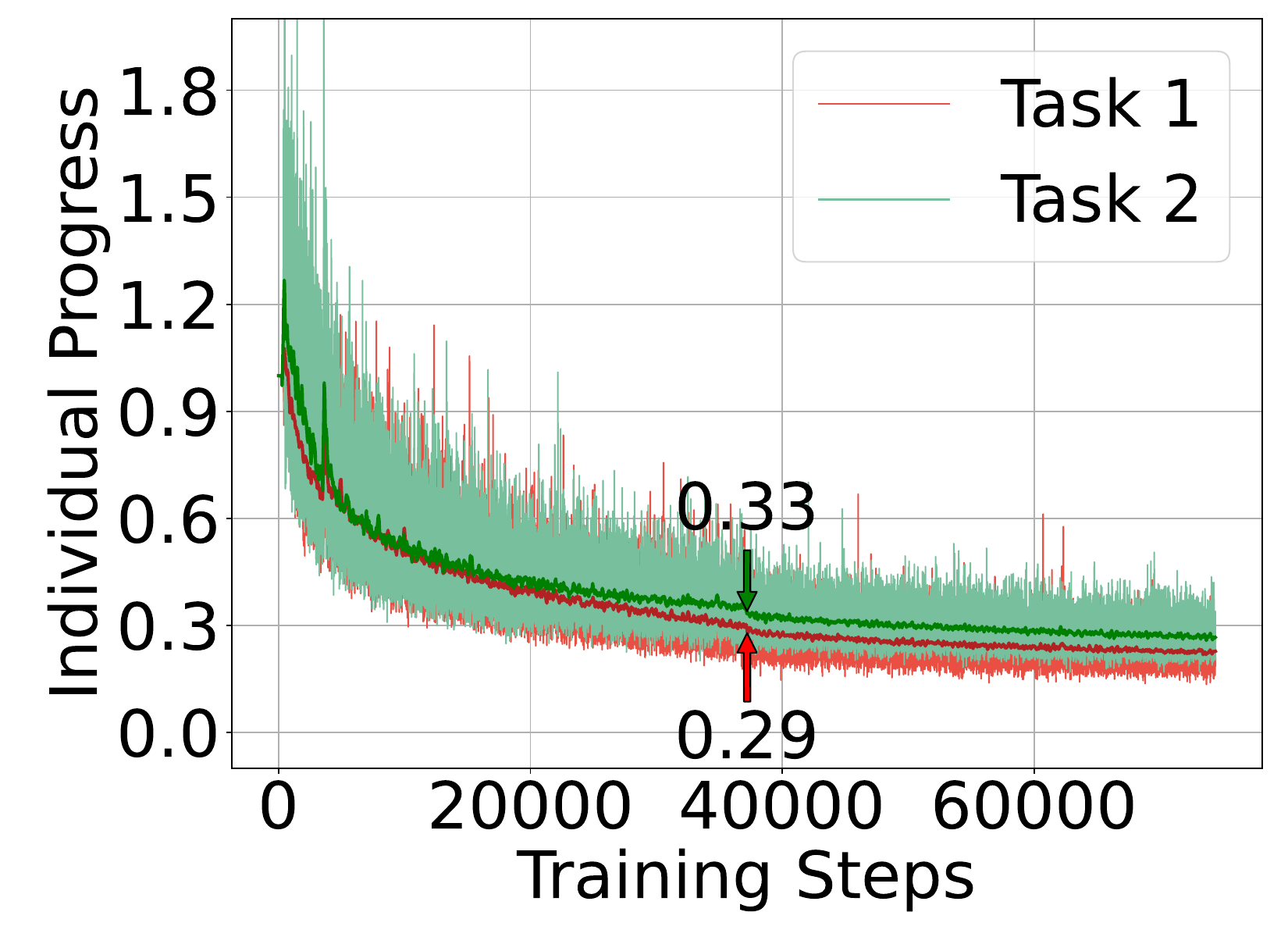}}
    \hfill
    \subfloat[GradDrop ($\Delta m\% = 21.18$)]{\includegraphics[width = 0.235\textwidth]{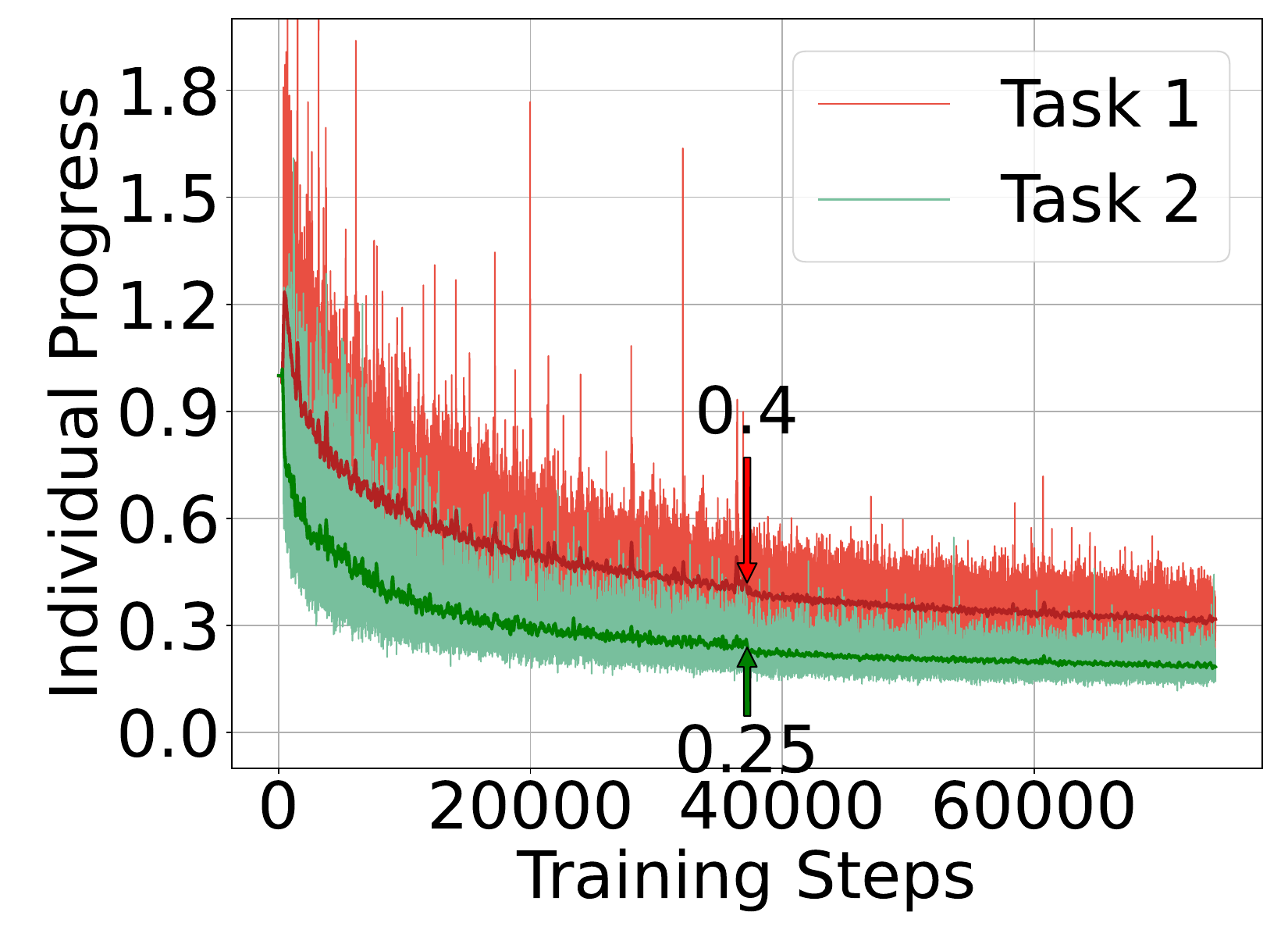}}
    \caption{Pareto failures and individual progress examination. The top row presents the gradient similarity, while the bottom row illustrates the individual progress for MGDA and GradDrop, respectively.}
    \label{fig:app_sim_pro}
\end{figure}
\subsection{Speedup Analysis} \label{app_sec:speed}
While the primary focus of this paper does not revolve around achieving speedup, it is worth mentioning that general speedup approaches, such as sampling a subset of tasks~\citep{liu2021conflict}, utilizing feature-level gradients~\citep{sener2018multi}, and updating $\bm{\omega}$ periodically instead of every iteration~\citep{navon2022multi}, continue to be applicable to methods augmented with \texttt{IMGrad}.
\begin{figure*}[h]
    \centering
    \subfloat[Sim: CAGrad]{\includegraphics[width = 0.235\textwidth]{figs/sim_pro/cagrad_new_cityscapes_sim.pdf}}
    \hfill
    \subfloat[Pro: CAGrad]{\includegraphics[width = 0.235\textwidth]{figs/sim_pro/cagrad_new_cityscapes_pro.pdf}}
    \hfill 
    \subfloat[Sim: Nash-MTL]{\includegraphics[width = 0.235\textwidth]{figs/sim_pro/nash_new_cityscapes_sim.pdf}}
    \hfill
    \subfloat[Pro: Nash-MTL]{\includegraphics[width = 0.235\textwidth]{figs/sim_pro/nash_new_cityscapes_pro.pdf}}
    \hfill  \\
    \subfloat[Sim: CAGrad+\texttt{IMGrad}]{\includegraphics[width = 0.235\textwidth]{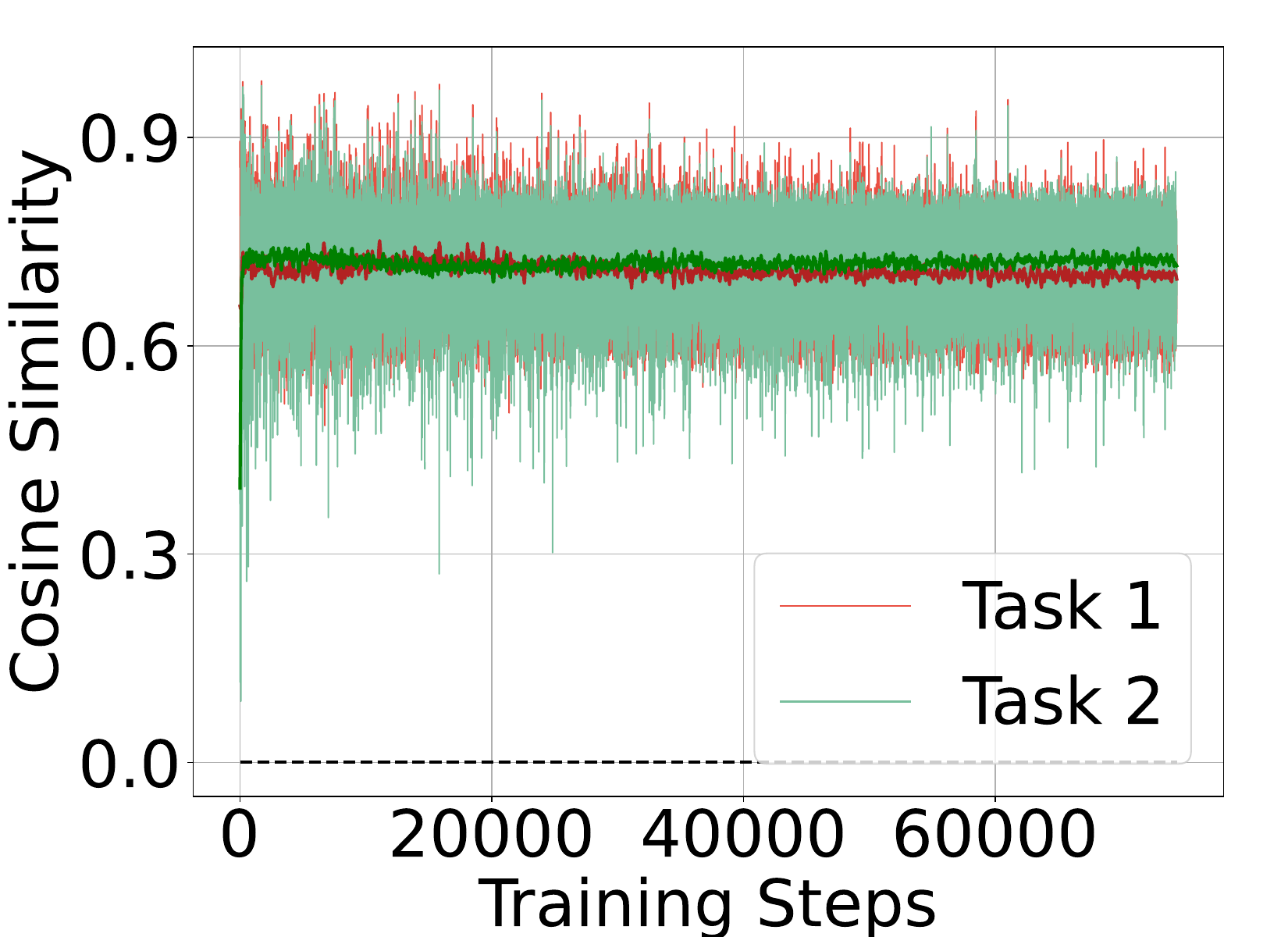}}
    \hfill
    \subfloat[Pro: CAGrad+\texttt{IMGrad}]{\includegraphics[width = 0.235\textwidth]{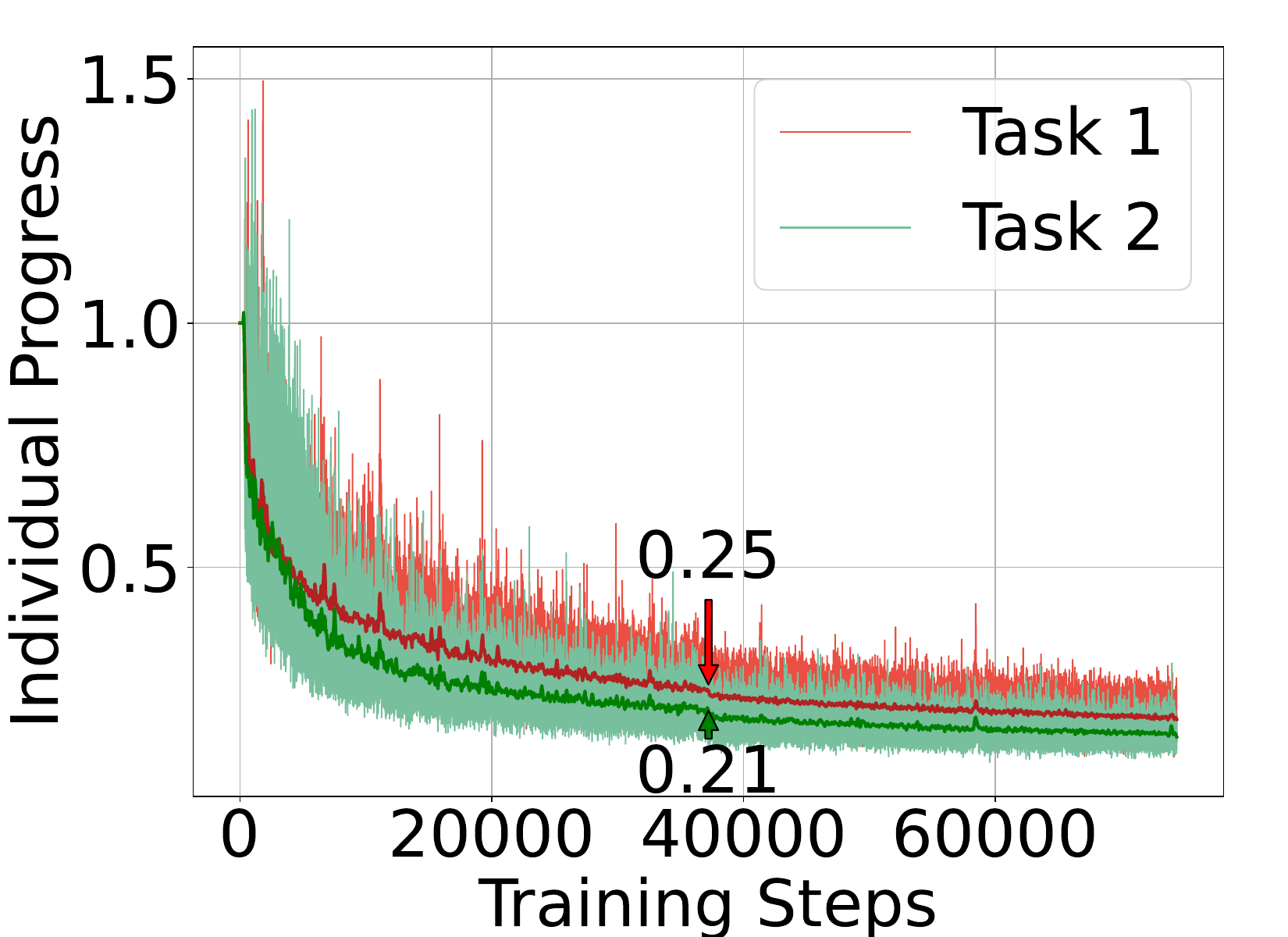}}
    \hfill
    \subfloat[Sim: Nash-MTL+\texttt{IMGrad}]{\includegraphics[width = 0.235\textwidth]{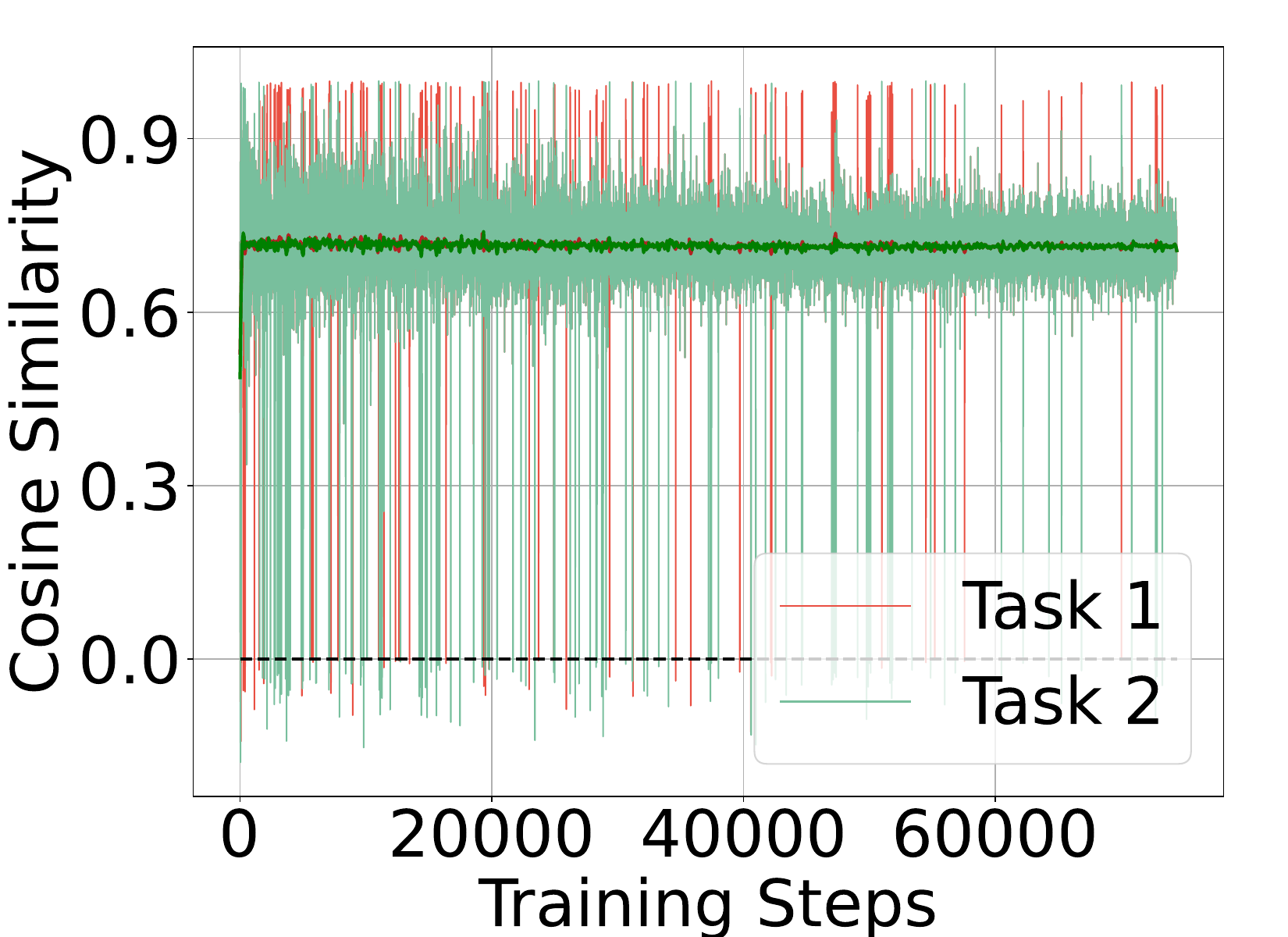}}
    \hfill
    \subfloat[Pro: Nash-MTL+\texttt{IMGrad}]{\includegraphics[width = 0.235\textwidth]{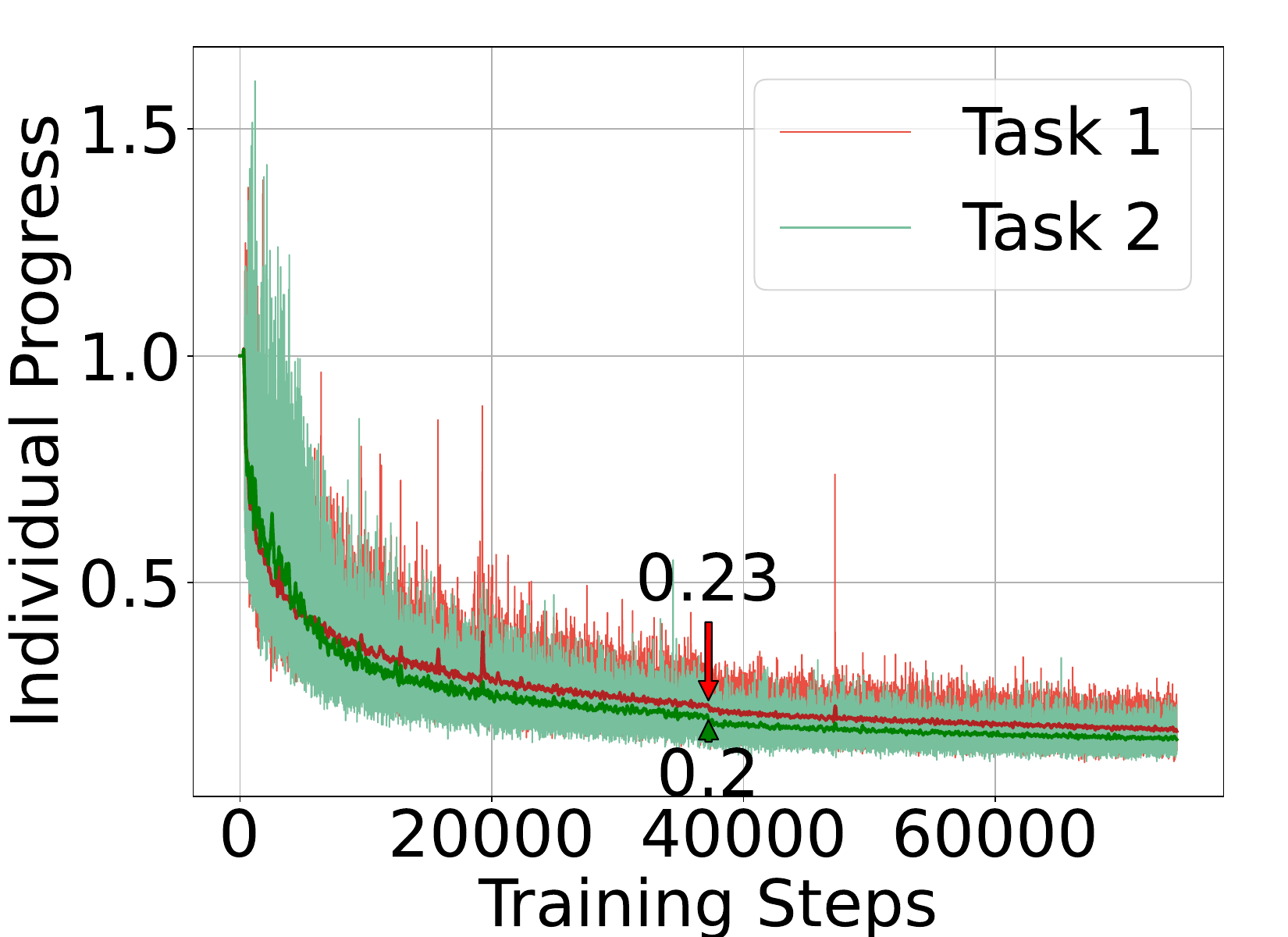}}
    \caption{Pareto failures and individual progress examinations. `Sim' and `Pro' are short for similarity and progress, respectively.}
    \label{fig:sim_pro_imgrad}
\end{figure*}

On the other hand, based on the findings presented in Table~\ref{table:alter}, we have introduced two additional computation-efficient alternatives for $\mu$. These alternatives offer reduced GPU time requirements while still maintaining an acceptable level of performance sacrifice.

We have also implemented a commonly used practical speedup approach, i.e., updating the weights less frequently. In this regard, we have chosen the update step sizes from the set [2, 4, 8, 16], and the corresponding results are presented in Figure~\ref{fig:step_size}. Generally, as the step size increases, the overall performance ($\Delta m\%$) tends to deteriorate, although it still demonstrates improvements compared to the vanilla CAGrad. Interestingly, the performance appears to be better when the step size is set to 16 compared to 4 and 8, suggesting that frequent weight updates may not be necessary. However, further exploration of this aspect is left for future research.

\section{Visualizations}   \label{sec:visual}
\subsection{Synthetic Examples} \label{app_sec:syn}
In order to further illustrate the effectiveness of \texttt{IMGrad} in different imbalance scenarios, we present additional comparative synthetic examples in Table~\ref{table:app_syn}. The results reveal that LS, PCGrad, and CAGrad exhibit failure cases across nearly all scenarios, while IMTL~\citep{liu2021towards} and Nash-MTL fail to reach the global optimum, although they do converge to the Pareto stationary. In contrast, \texttt{IMGrad} consistently achieves the optimal solution in all scenarios, demonstrating its robust and stable performance. This underscores the significance of imbalance-sensitivity in enhancing optimization outcomes.
% \usepackage{multirow}
\begin{table*}[h]
\centering
\caption{Comparison of MTL optimization methods on synthetic two-task benchmark. $(a_1, a_2)$ denotes $a_1 * \mathcal{L}_1 + a_2 * \mathcal{L}_2$.}
\label{table:app_syn}
\footnotesize
\setlength\tabcolsep{1pt}
\begin{tabular}{lllllll}
\toprule
  & \multicolumn{1}{c}{LS}  &  \multicolumn{1}{c}{PCGrad}  &  \multicolumn{1}{c}{CAGrad}  &  \multicolumn{1}{c}{IMTL}  &  \multicolumn{1}{c}{Nash-MTL} & \multicolumn{1}{c}{\texttt{IMGrad}}  \\ \midrule
$(0.1, 0.9)$ &  
    \begin{minipage}[b]{0.30\columnwidth}
		\raisebox{-.5\height}{\includegraphics[width=\linewidth]{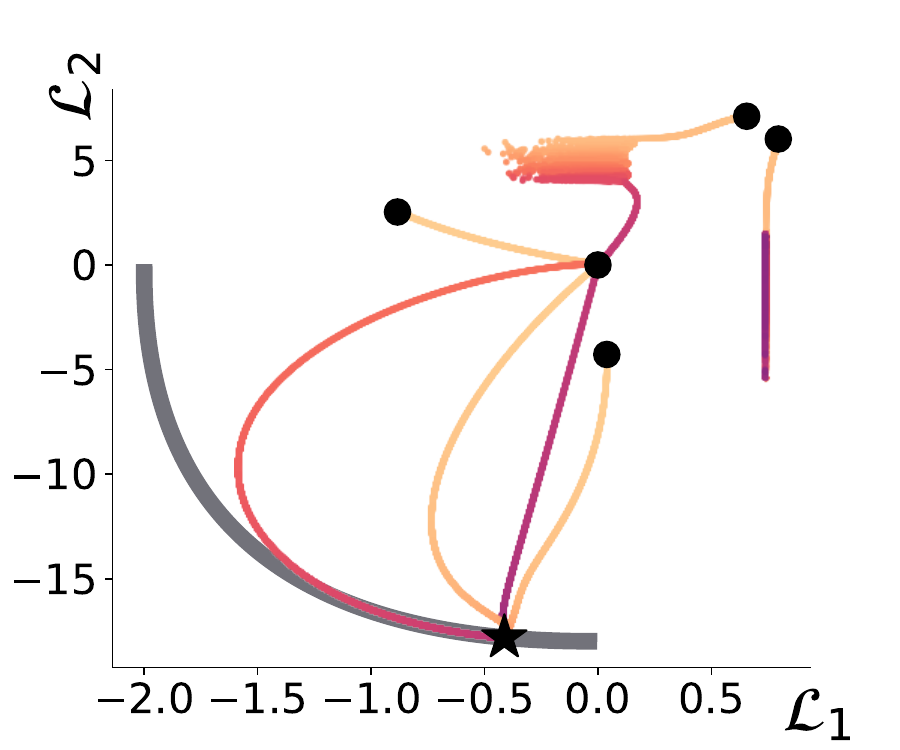}}
	\end{minipage} &
     \begin{minipage}[b]{0.30\columnwidth}
		\raisebox{-.5\height}{\includegraphics[width=\linewidth]{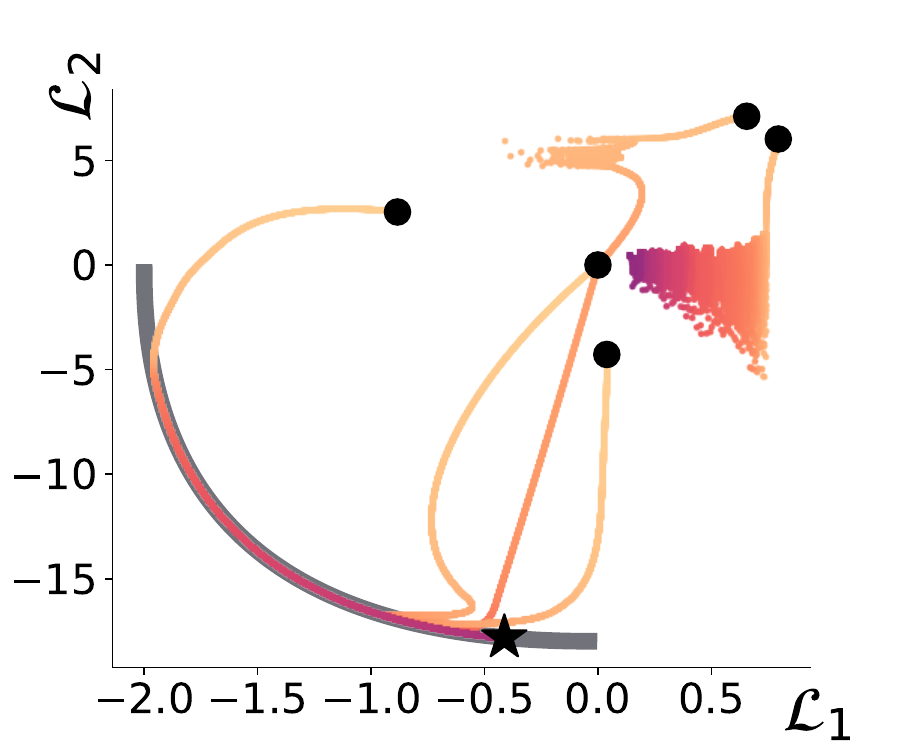}}
	\end{minipage} &
    \begin{minipage}[b]{0.30\columnwidth}
		\raisebox{-.5\height}{\includegraphics[width=\linewidth]{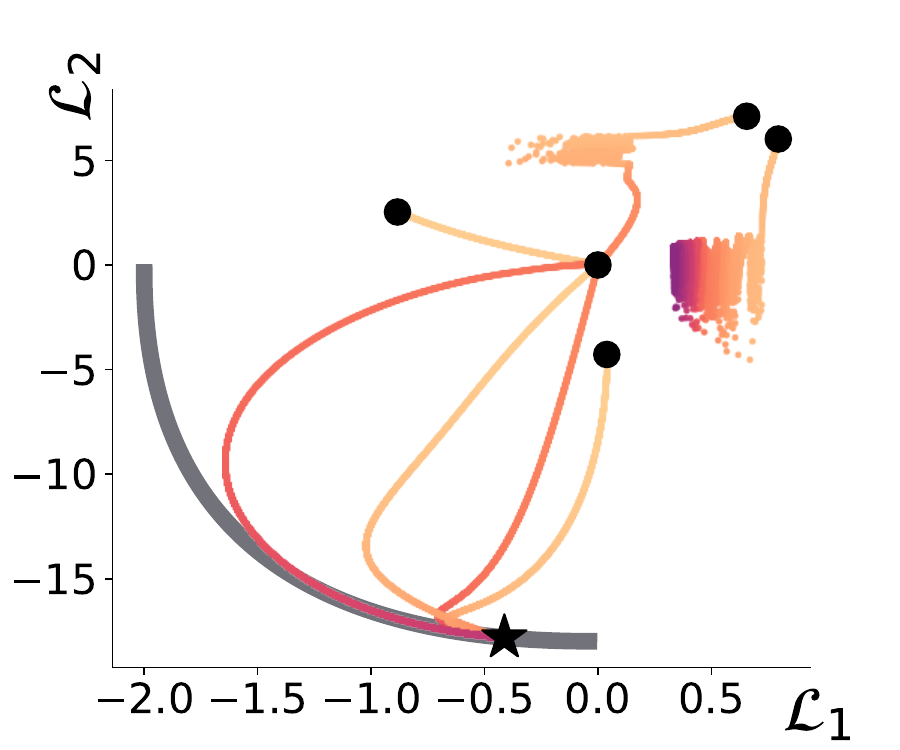}}
	\end{minipage} &  
    \begin{minipage}[b]{0.30\columnwidth}
		\raisebox{-.5\height}{\includegraphics[width=\linewidth]{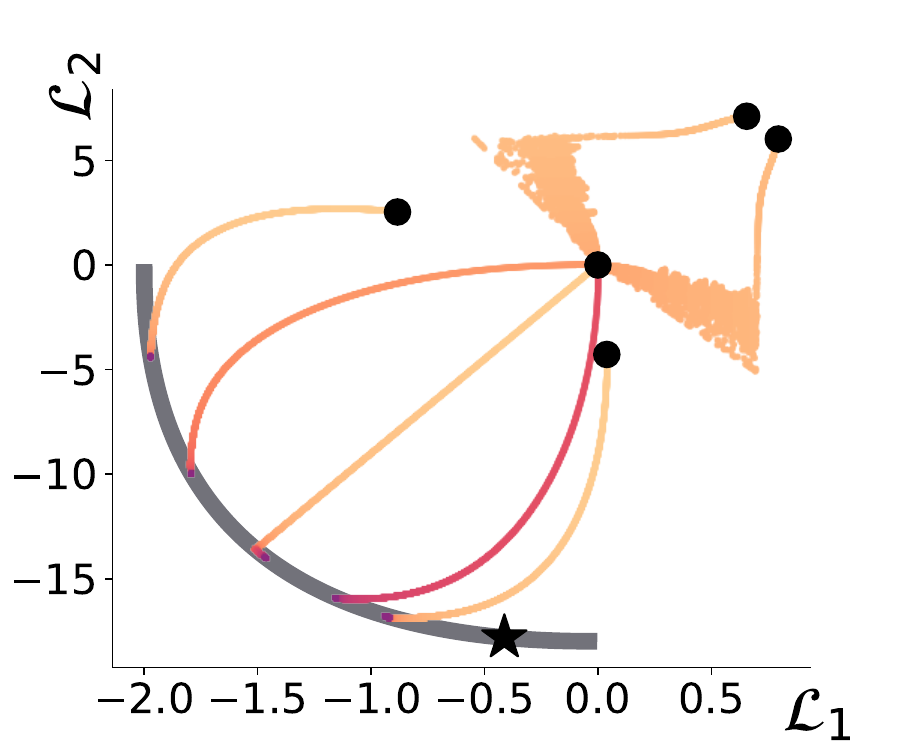}}
	\end{minipage} &
    \begin{minipage}[b]{0.30\columnwidth}
		\raisebox{-.5\height}{\includegraphics[width=\linewidth]{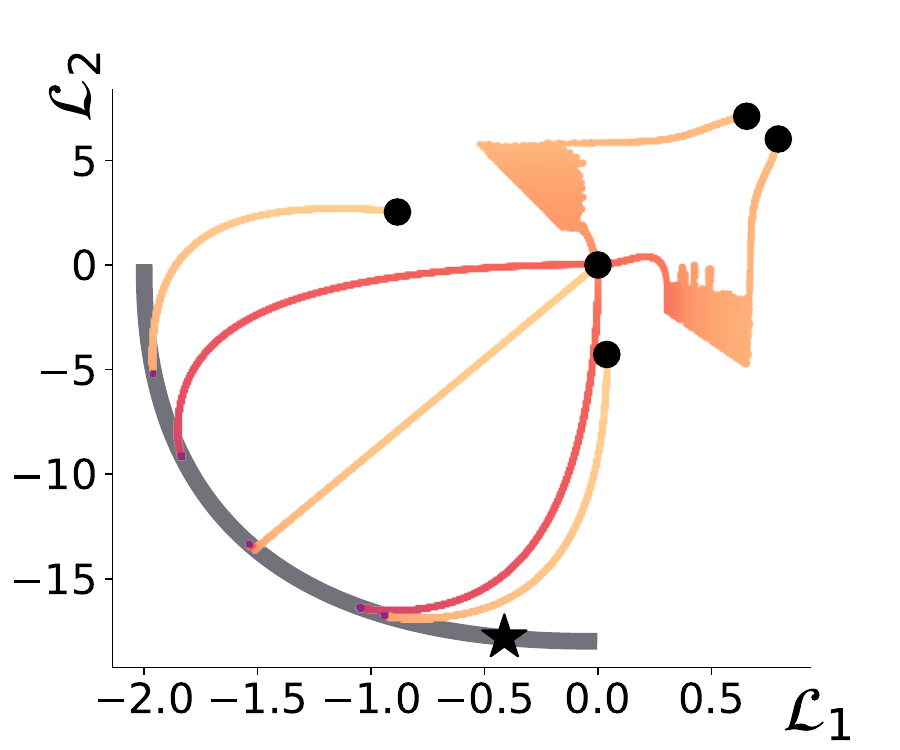}}
	\end{minipage} &
    \begin{minipage}[b]{0.30\columnwidth}
		\raisebox{-.5\height}{\includegraphics[width=\linewidth]{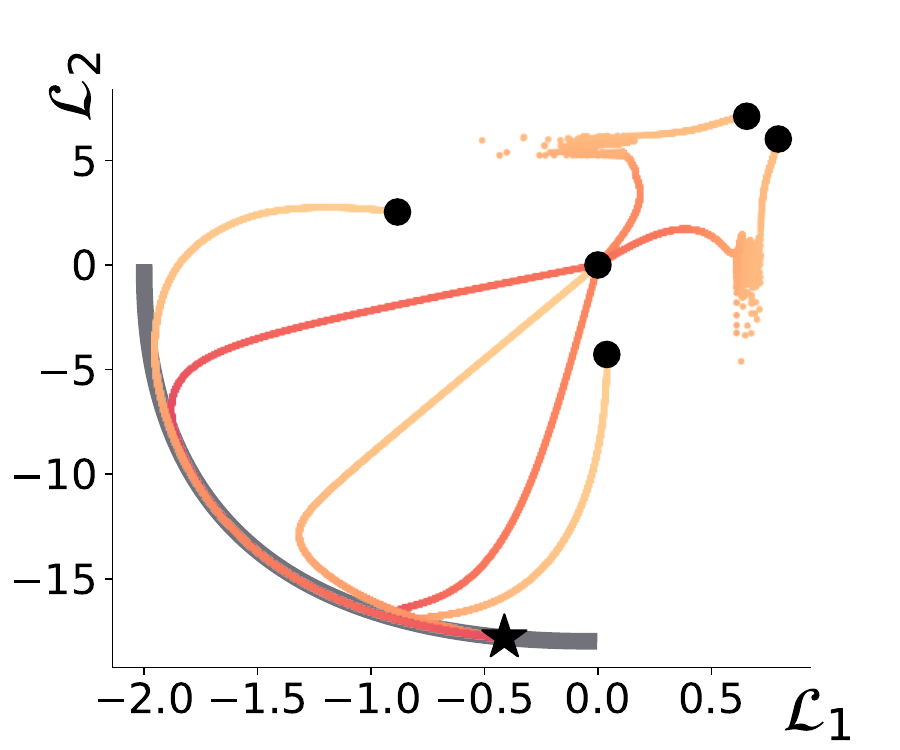}}
	\end{minipage} 
   \\ 

$(0.3, 0.7)$ &  
    \begin{minipage}[b]{0.30\columnwidth}
		\raisebox{-.5\height}{\includegraphics[width=\linewidth]{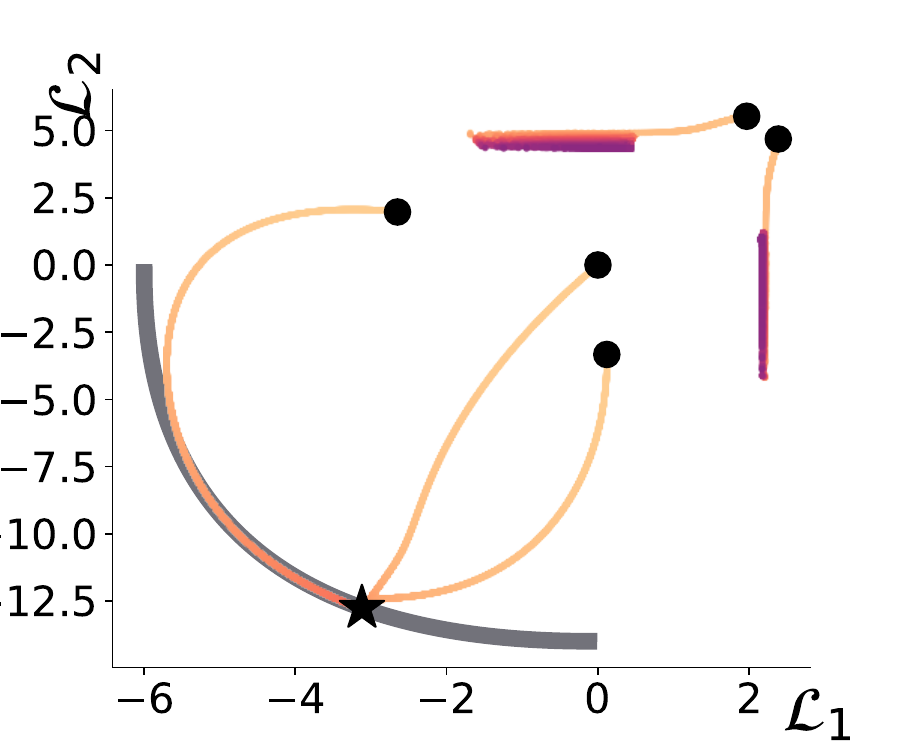}}
	\end{minipage} &
     \begin{minipage}[b]{0.30\columnwidth}
		\raisebox{-.5\height}{\includegraphics[width=\linewidth]{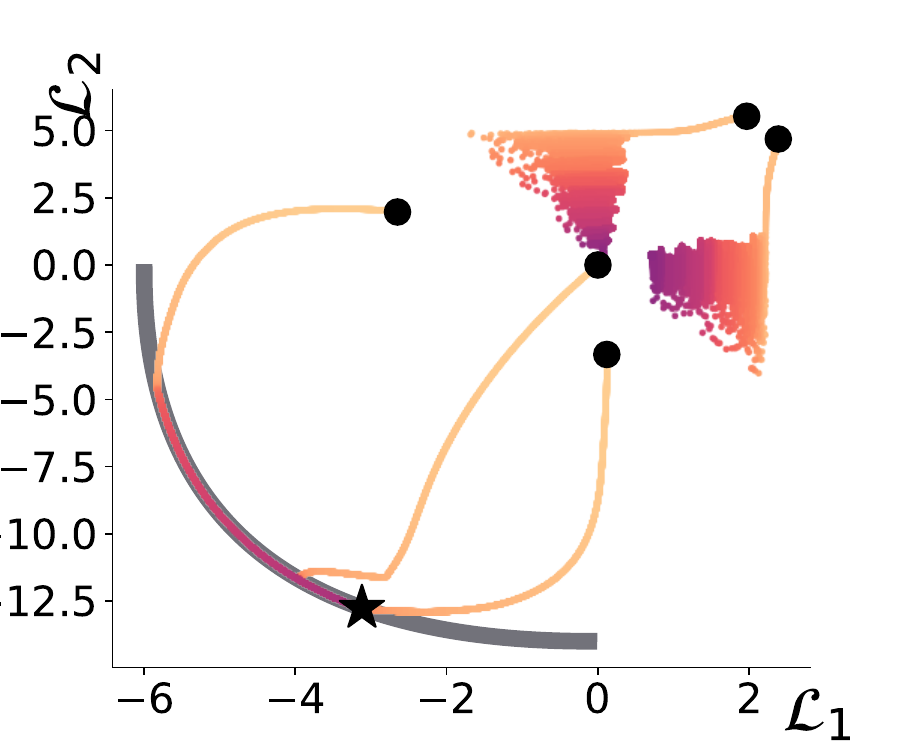}}
	\end{minipage} &
    \begin{minipage}[b]{0.30\columnwidth}
		\raisebox{-.5\height}{\includegraphics[width=\linewidth]{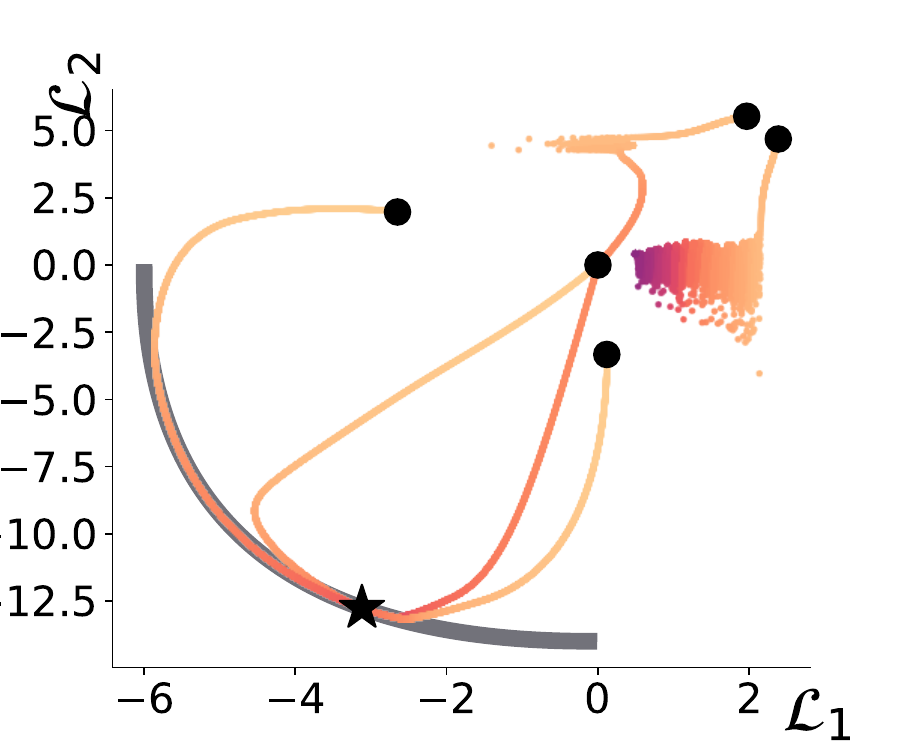}}
	\end{minipage} &  
    \begin{minipage}[b]{0.30\columnwidth}
		\raisebox{-.5\height}{\includegraphics[width=\linewidth]{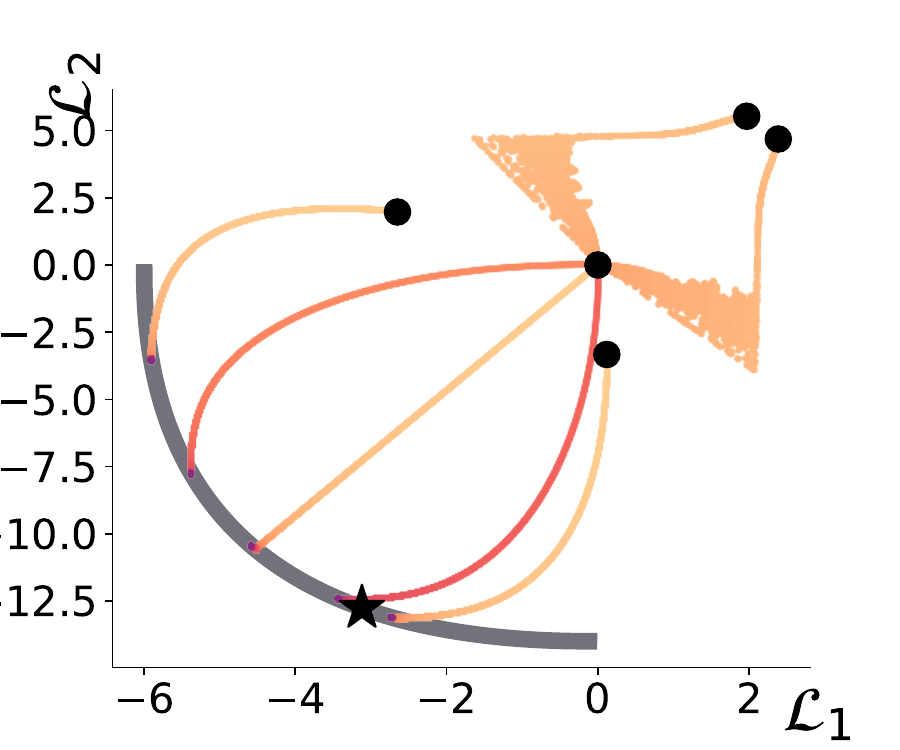}}
	\end{minipage} &
    \begin{minipage}[b]{0.30\columnwidth}
		\raisebox{-.5\height}{\includegraphics[width=\linewidth]{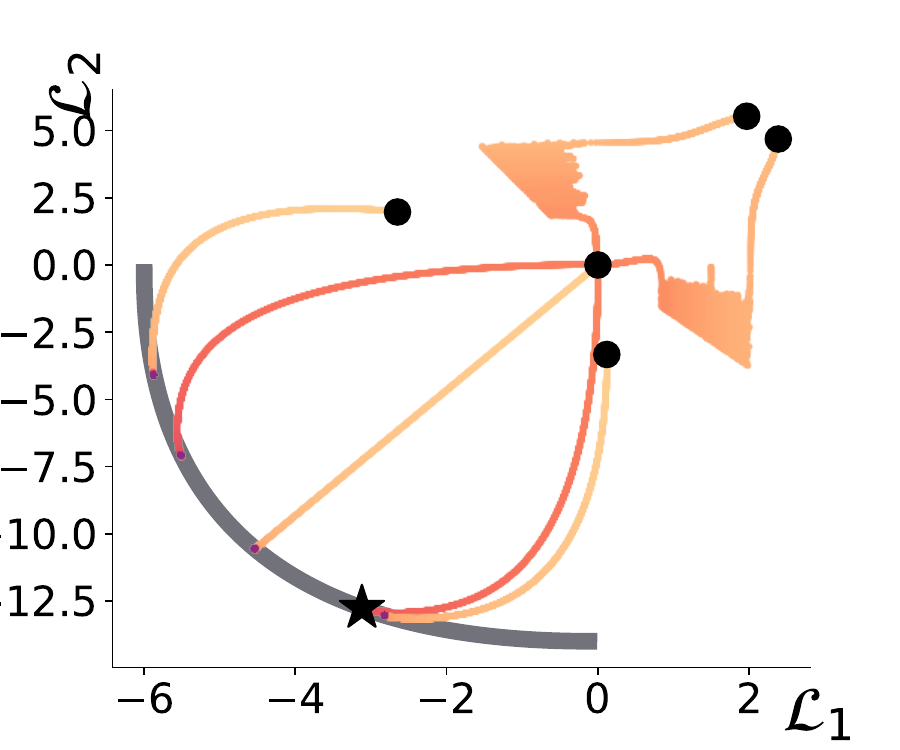}}
	\end{minipage} &
    \begin{minipage}[b]{0.30\columnwidth}
		\raisebox{-.5\height}{\includegraphics[width=\linewidth]{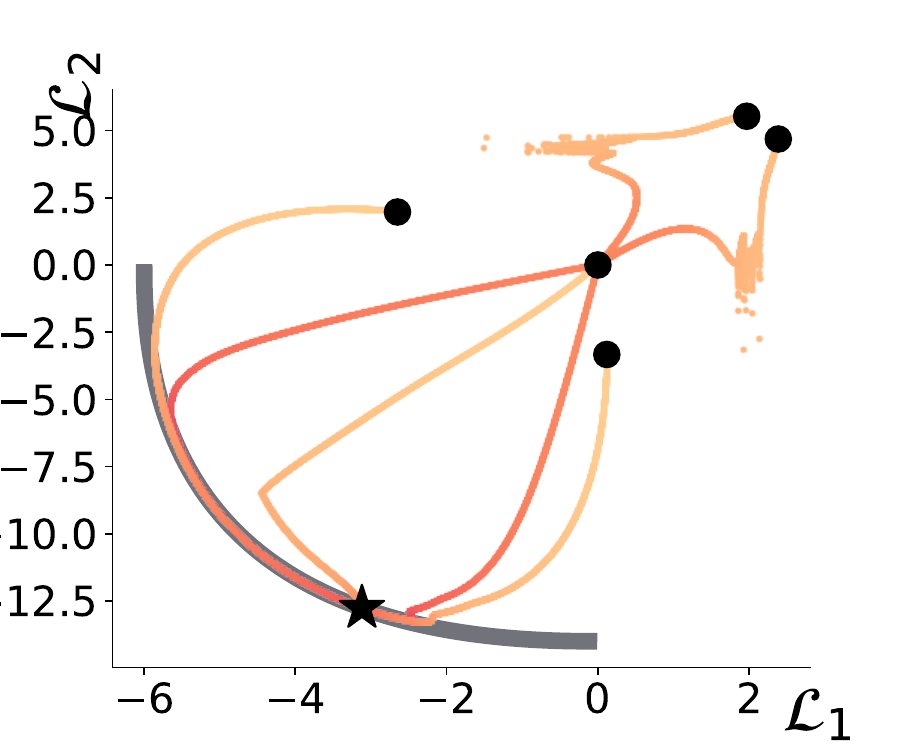}}
	\end{minipage} 
   \\
   
$(0.5, 0.5)$ &  
    \begin{minipage}[b]{0.30\columnwidth}
		\raisebox{-.5\height}{\includegraphics[width=\linewidth]{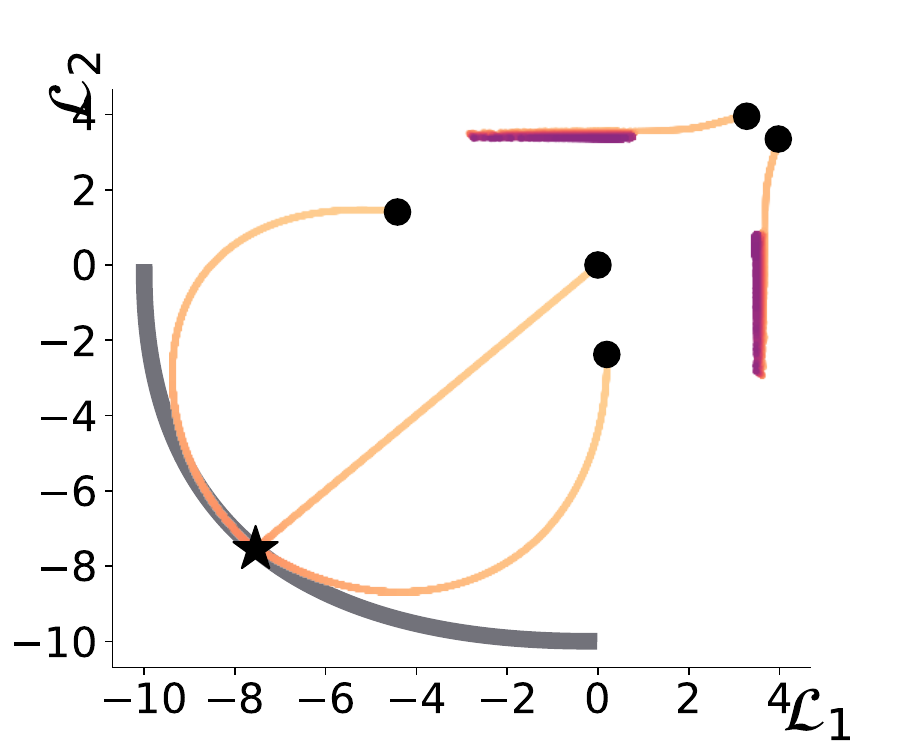}}
	\end{minipage} &
     \begin{minipage}[b]{0.30\columnwidth}
		\raisebox{-.5\height}{\includegraphics[width=\linewidth]{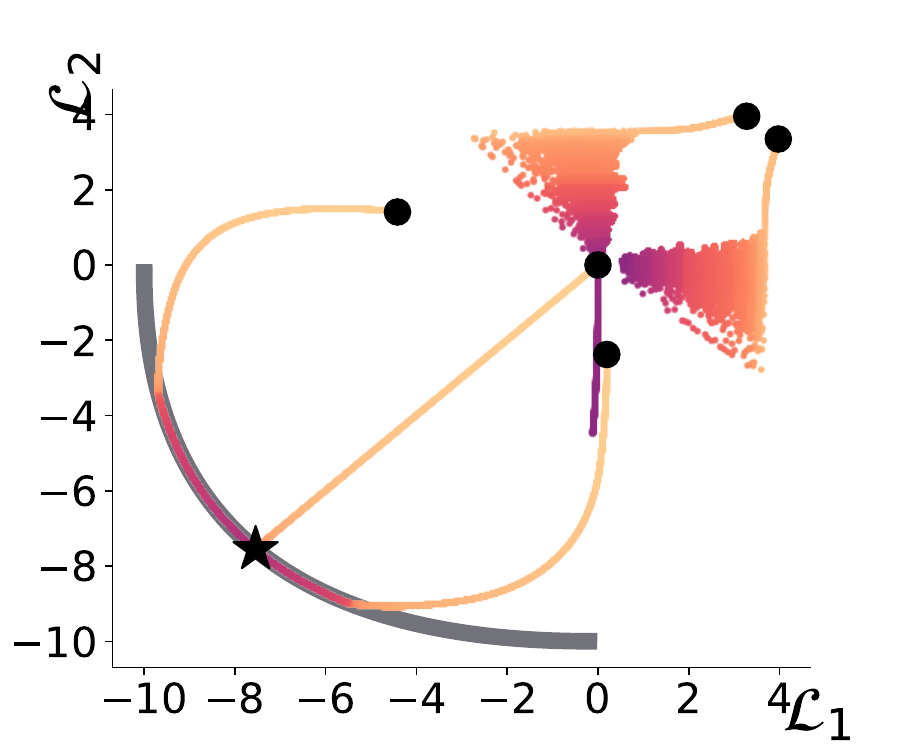}}
	\end{minipage} &
    \begin{minipage}[b]{0.30\columnwidth}
		\raisebox{-.5\height}{\includegraphics[width=\linewidth]{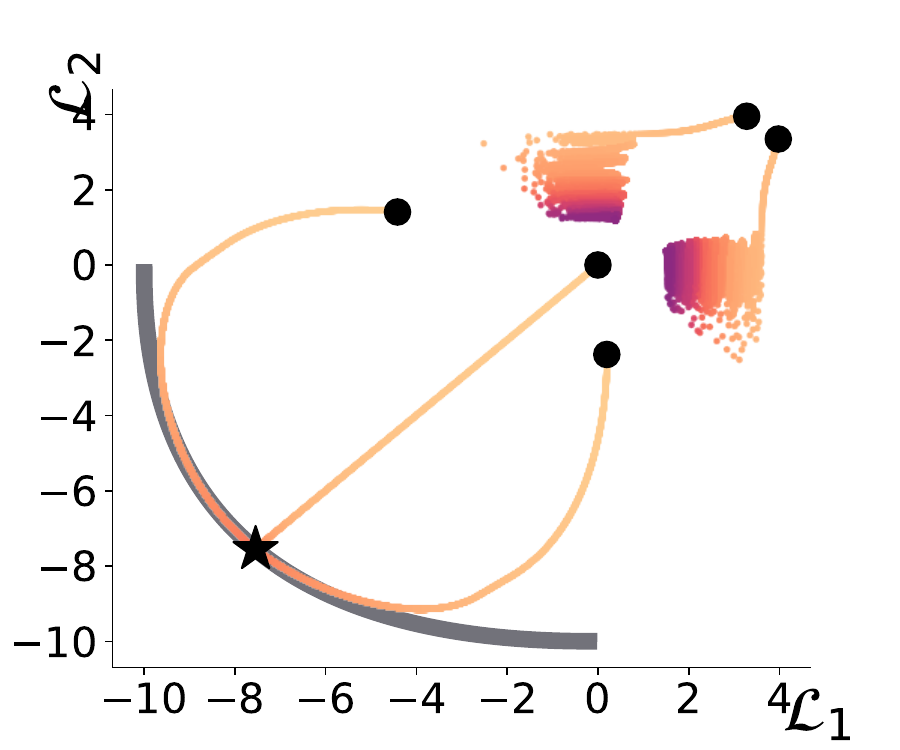}}
	\end{minipage} &  
    \begin{minipage}[b]{0.30\columnwidth}
		\raisebox{-.5\height}{\includegraphics[width=\linewidth]{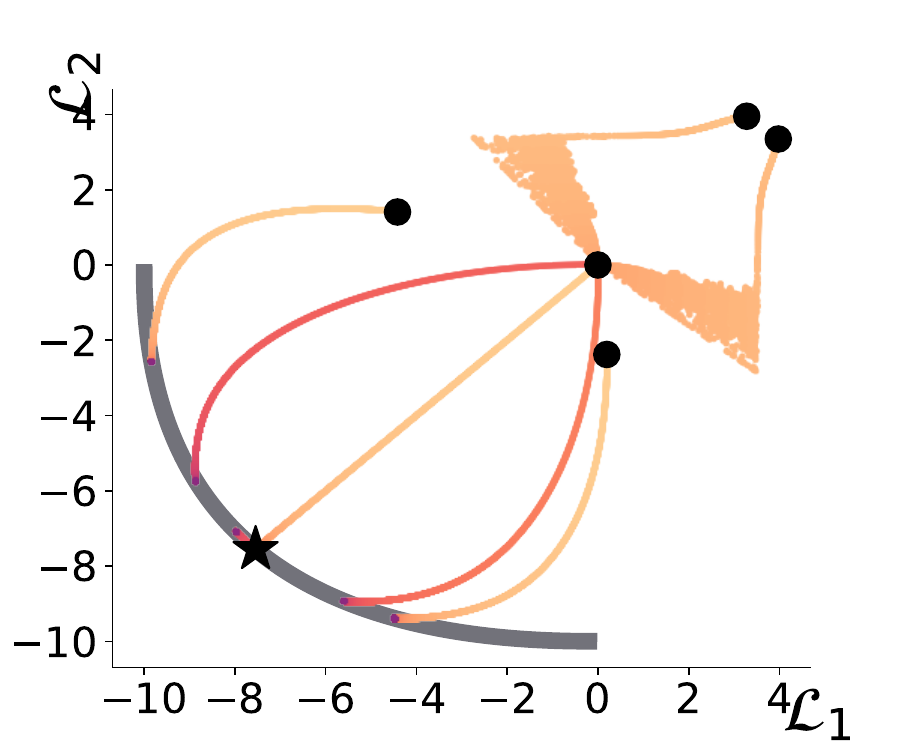}}
	\end{minipage} &
    \begin{minipage}[b]{0.30\columnwidth}
		\raisebox{-.5\height}{\includegraphics[width=\linewidth]{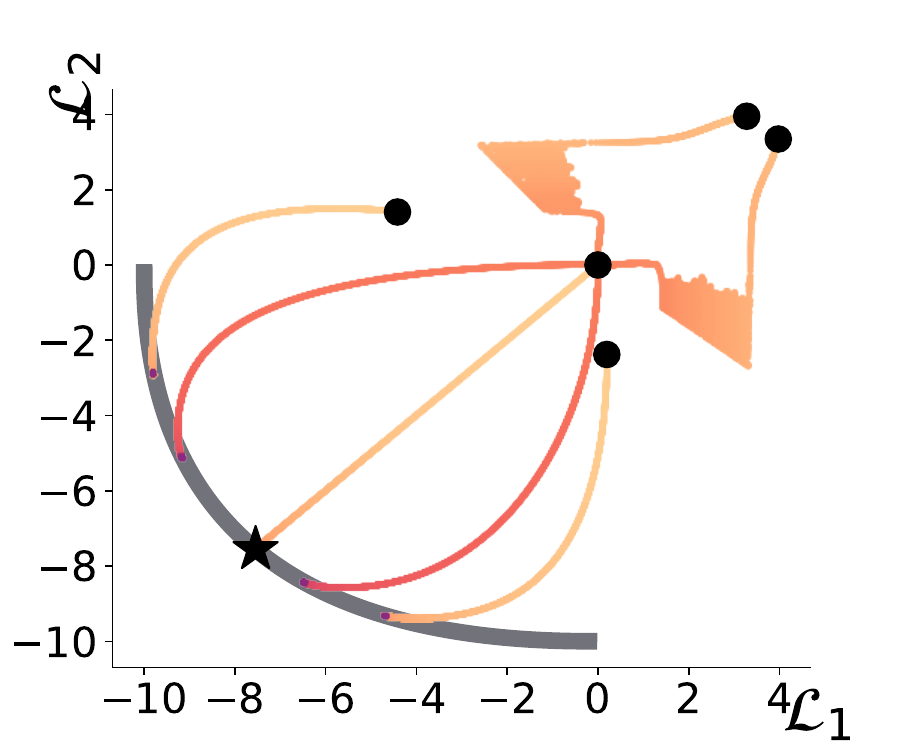}}
	\end{minipage} &
    \begin{minipage}[b]{0.30\columnwidth}
		\raisebox{-.5\height}{\includegraphics[width=\linewidth]{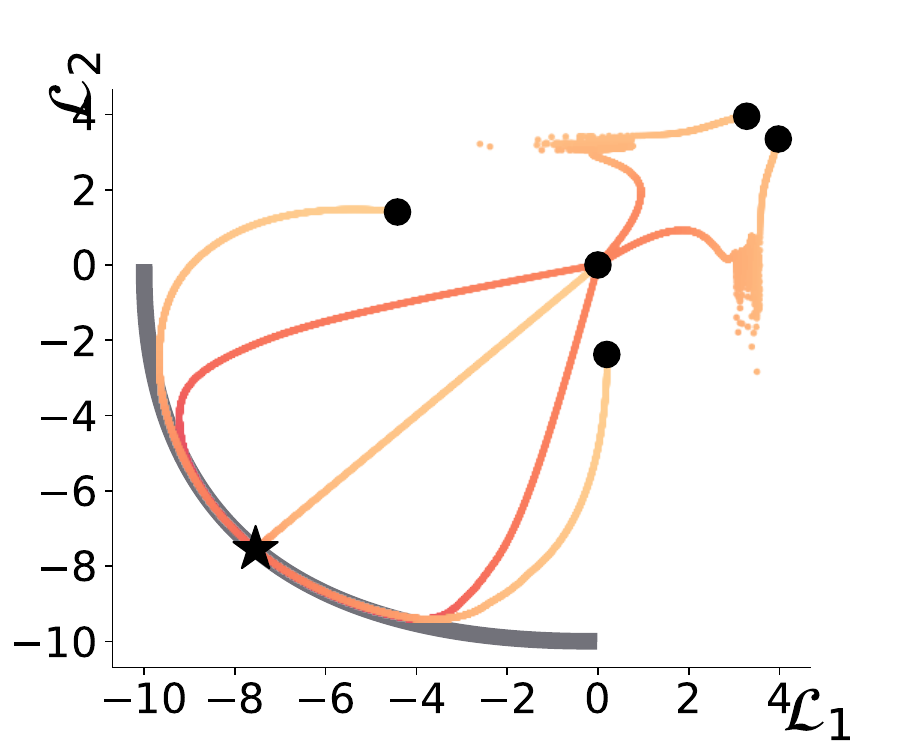}}
	\end{minipage} 
   \\ %\hline

$(0.7, 0.3)$ &  
    \begin{minipage}[b]{0.30\columnwidth}
		\raisebox{-.5\height}{\includegraphics[width=\linewidth]{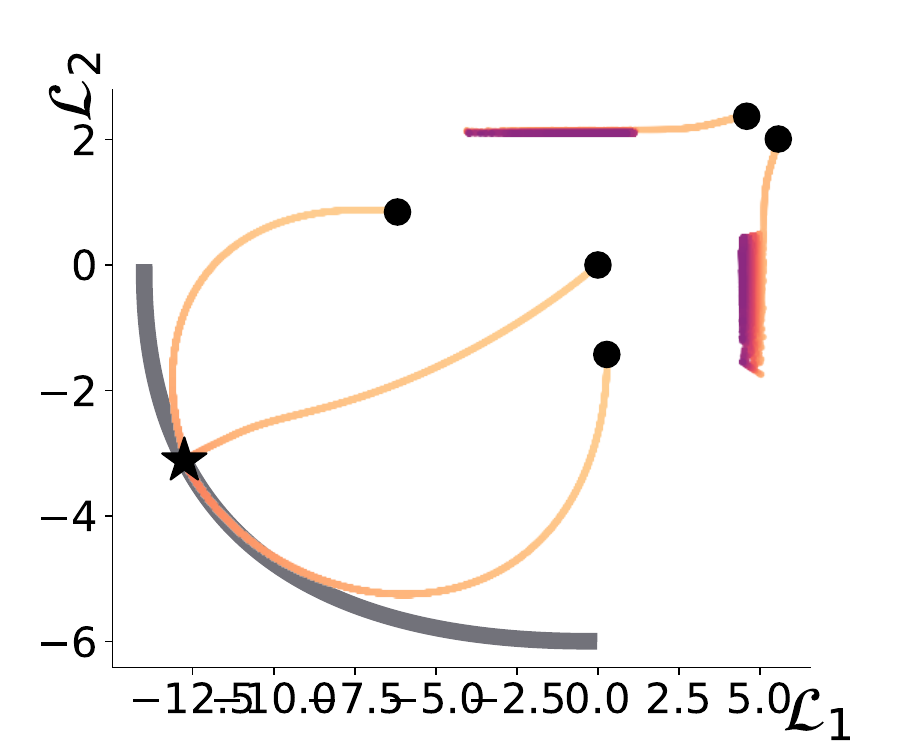}}
	\end{minipage} &
     \begin{minipage}[b]{0.30\columnwidth}
		\raisebox{-.5\height}{\includegraphics[width=\linewidth]{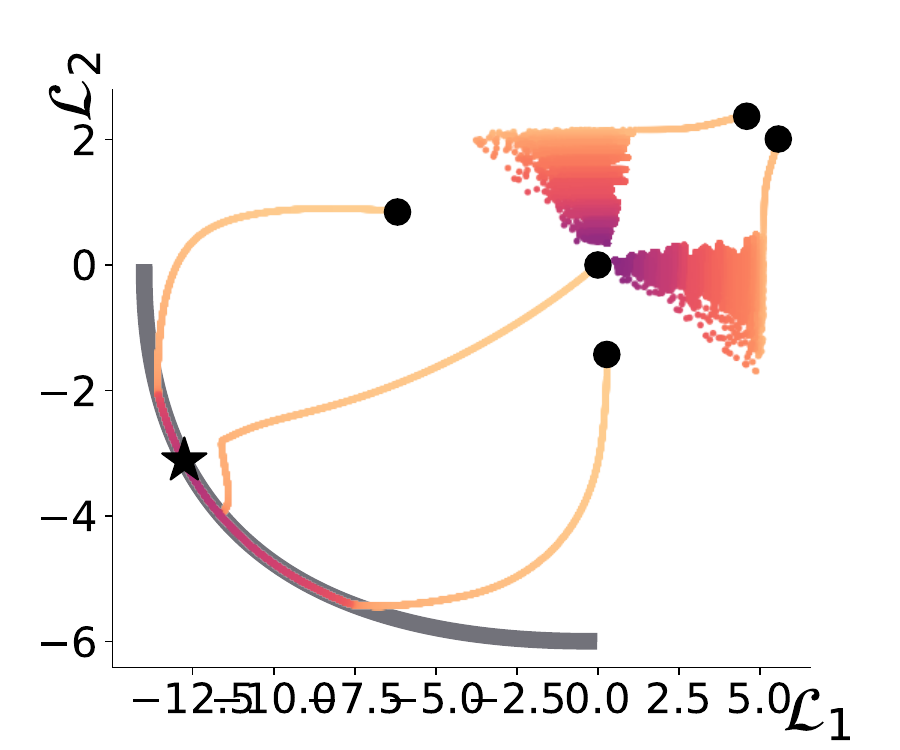}}
	\end{minipage} &
    \begin{minipage}[b]{0.30\columnwidth}
		\raisebox{-.5\height}{\includegraphics[width=\linewidth]{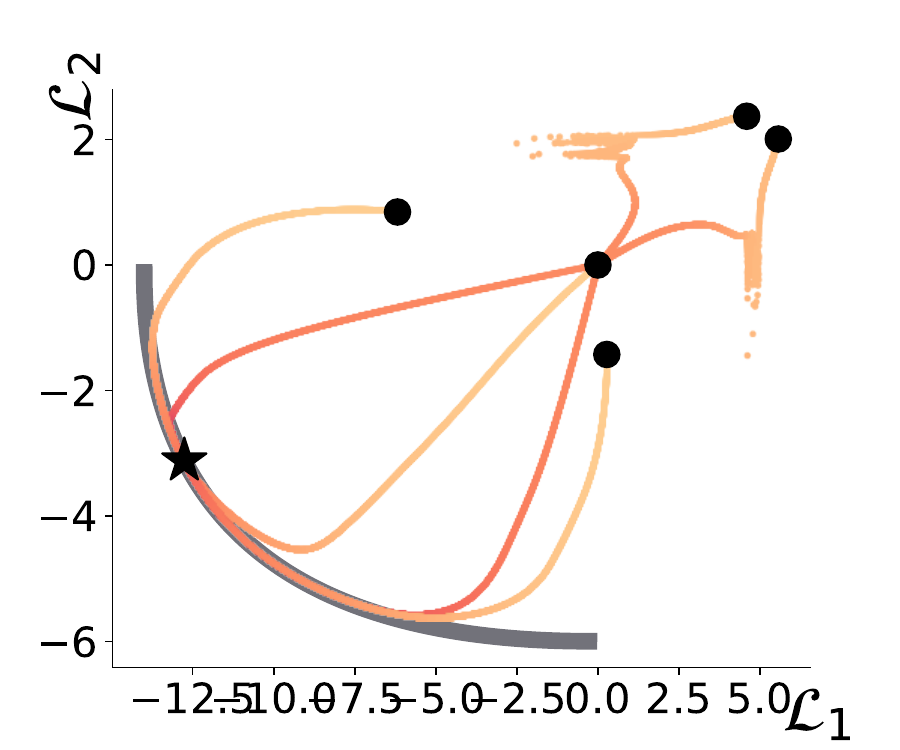}}
	\end{minipage} &  
    \begin{minipage}[b]{0.30\columnwidth}
		\raisebox{-.5\height}{\includegraphics[width=\linewidth]{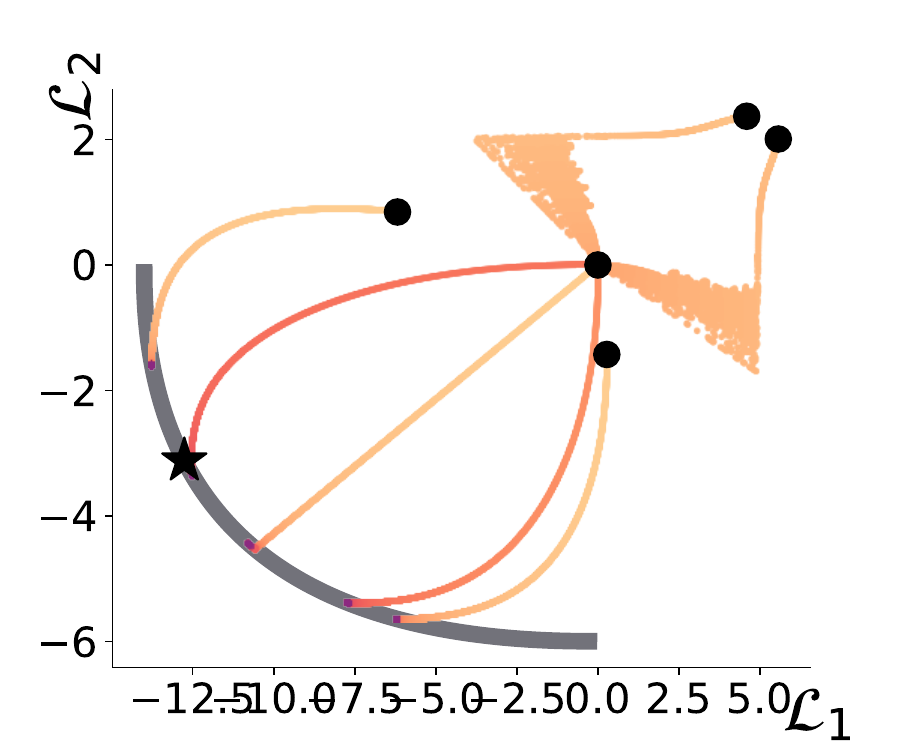}}
	\end{minipage} &
    \begin{minipage}[b]{0.30\columnwidth}
		\raisebox{-.5\height}{\includegraphics[width=\linewidth]{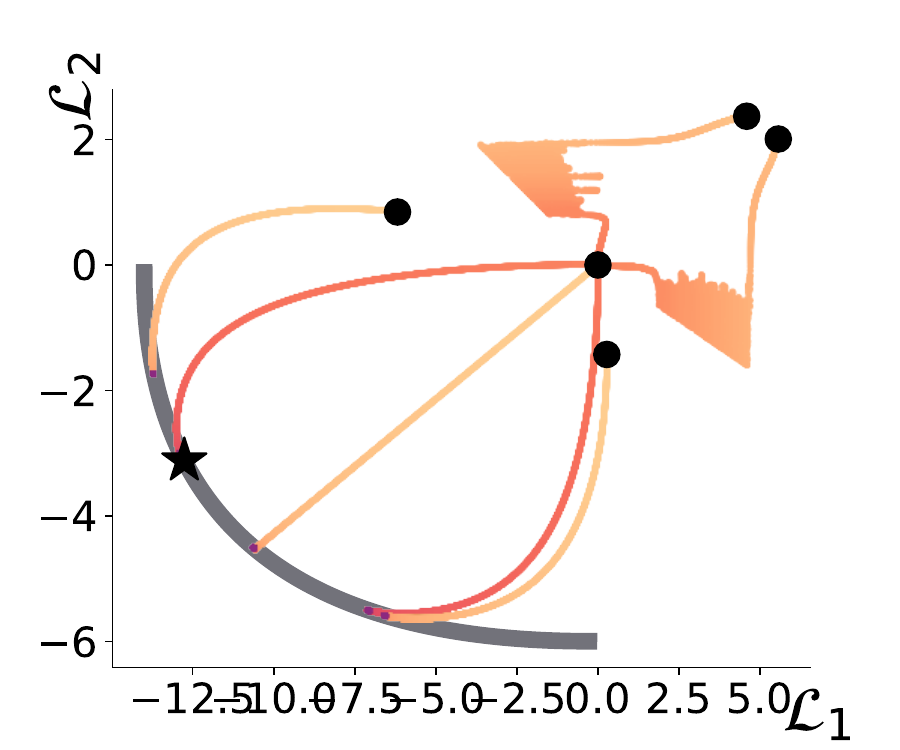}}
	\end{minipage} &
    \begin{minipage}[b]{0.30\columnwidth}
		\raisebox{-.5\height}{\includegraphics[width=\linewidth]{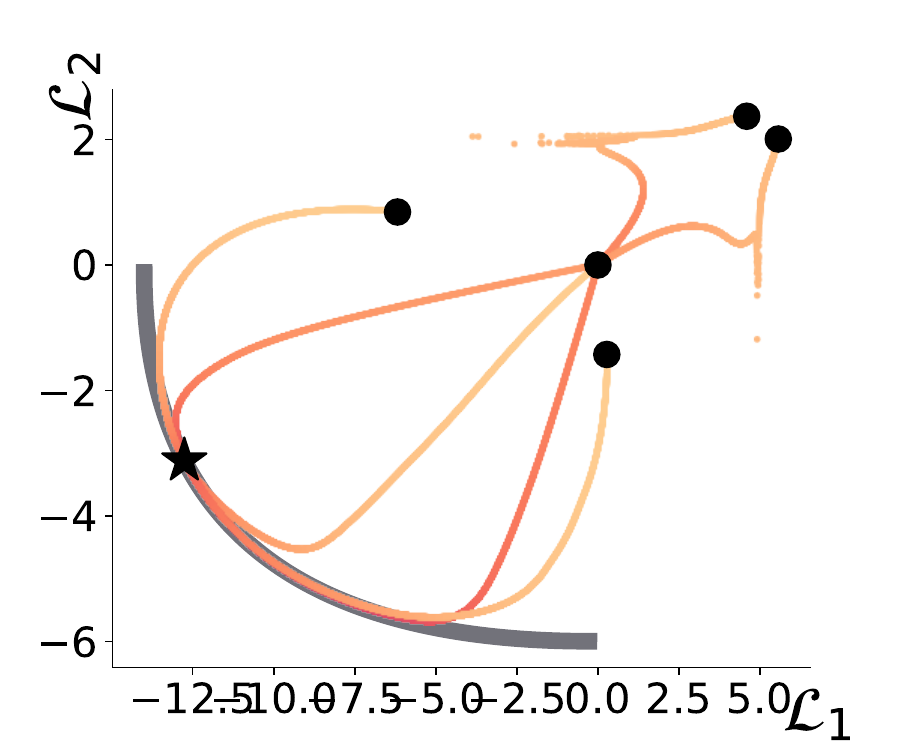}}
	\end{minipage} 
   \\
   
$(0.9, 0.1)$ & 
    \begin{minipage}[b]{0.30\columnwidth}
		\raisebox{-.5\height}{\includegraphics[width=\linewidth]{figs/appendix/toy/plot_ls_0.9.pdf}}
	\end{minipage} &
     \begin{minipage}[b]{0.30\columnwidth}
		\raisebox{-.5\height}{\includegraphics[width=\linewidth]{figs/appendix/toy/plot_pcgrad_0.9.pdf}}
	\end{minipage} &
    \begin{minipage}[b]{0.30\columnwidth}
		\raisebox{-.5\height}{\includegraphics[width=\linewidth]{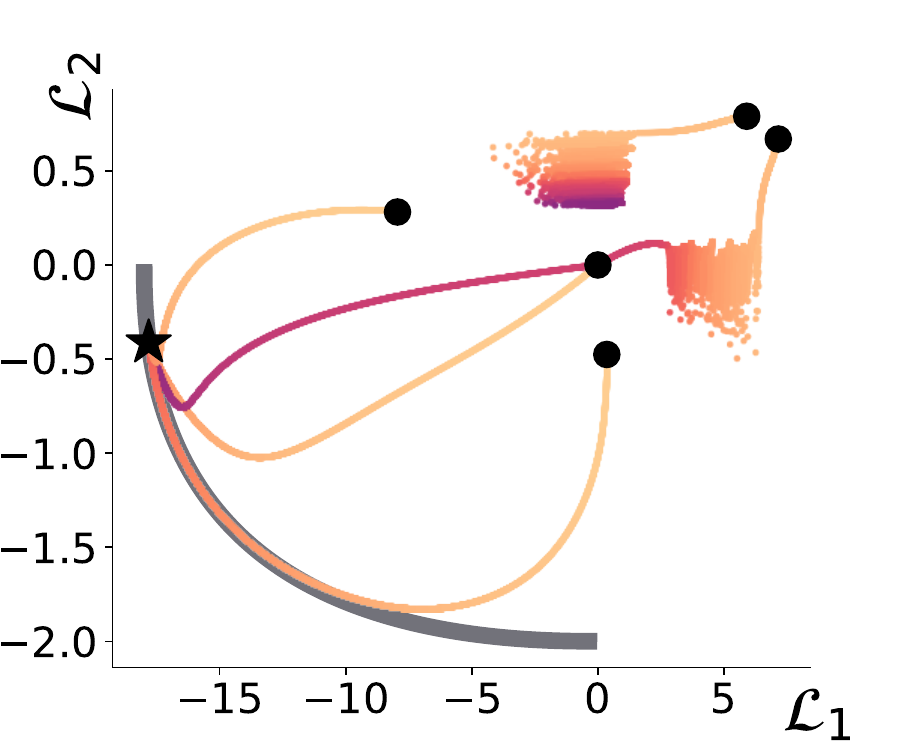}}
	\end{minipage} &  
    \begin{minipage}[b]{0.30\columnwidth}
		\raisebox{-.5\height}{\includegraphics[width=\linewidth]{figs/appendix/toy/plot_imtl_0.9.pdf}}
	\end{minipage} &
    \begin{minipage}[b]{0.30\columnwidth}
		\raisebox{-.5\height}{\includegraphics[width=\linewidth]{figs/appendix/toy/plot_nash_0.9.pdf}}
	\end{minipage} &
    \begin{minipage}[b]{0.30\columnwidth}
		\raisebox{-.5\height}{\includegraphics[width=\linewidth]{figs/appendix/toy/plot_imgrad_c_0.4_0.9.pdf}}
	\end{minipage} 
   \\
 \bottomrule
\end{tabular}
\end{table*}

\subsection{Pareto Failure and Individual Progress Examinations}\label{appsec:pareto_fail}
Here we aim to provide a visual understanding of the advantages brought by \texttt{IMGrad} in terms of reducing Pareto failures and achieving a balanced individual progress. In Figure~\ref{fig:sim_pro_imgrad}, concerning CAGrad, the \texttt{IMGrad}-augmented version can effectively avoids Pareto failures and optimizes both individuals in a more balanced manner by seeking a direction that exhibits rough similarities among all individuals. In the case of Nash-MTL, \texttt{IMGrad} successfully reduces the number of Pareto failures from \underline{249} to \underline{161} and further narrows the gap in individual progress. 

Besides, we present supplementary analysis on the Pareto failure and individual progress using MGDA and GradDrop, and showcase the outcomes in Figure~\ref{fig:app_sim_pro}. As depicted, both MGDA and GradDrop exhibit pronounced imbalance issues based on their gradient similarities, leading to unsatisfactory and imbalanced individual progress. These findings align with the tendencies and indications stated above, further supporting the claims made.
\input{app_sec/complex_analysis}
\section{Algorithm}   \label{sec:alg}
In order to summarize our method, we present a concise pseudocode in Algorithm~\ref{alg:cagrad}. Additionally, we provide an implementation of CAGrad + \texttt{IMGrad} on CityScapes, which can be found in the \textbf{Supplementary Material}.
\begin{algorithm*}
	\caption{CAGrad + \texttt{IMGrad}}
	\label{alg:cagrad}
	\KwIn{Initial model parameter $\Theta$, differentiable loss functions are ${\mathcal{L}_i(\bm{\Theta}), i \in [N]}$}, a constant $c \in [0, 1]$ and learning rate $\alpha \in \mathbb{R}^{+}$. \\  
	\KwOut{Model trained with CAGrad + \texttt{IMGrad}} 
	\BlankLine\
	\While{\textnormal{not converged}}{
		At the $t_{th}$ optimization step, define $\bm{g_0} = \frac{1}{K}\sum_{i=1}^K \bigtriangledown_{\bm{\theta}} \mathcal{L}_i(\bm{\theta_{t-1}})$ and $\bm{\phi} = c^2 \left \| \bm{g_0}\right \|$. \\
        {\color{red} Obtain $\bm{g_m}$ via MGDA algorithm, and calculate the cosine similarity $\mu$ between $\bm{g_m}$ and $\bm{g_0}$.} \\
        Solve 
        \begin{align*}
         \mathop{\rm{min}}_{\bm{\omega} \in \bm{\mathcal{W}}} F(\omega) := {\color{red}(1- \mu)} \bm{g_{\omega}}^{\rm{\top}} \bm{g_0} + {\color{red}\mu} \sqrt{\phi} \left \| \bm{g_{\omega}} \right \|, {\rm where}\ \phi = c^2 \left \| \bm{g_0} \right \|^2
        \end{align*}\\
        Update $\bm{\theta_t} = \bm{\theta_{t-1}} - \alpha \left (  \bm{g_0} + \frac{\phi^{1/2}}{\left \| \bm{g_{\omega}} \right \|} \bm{g_{\omega}} \right ).$
    }
\end{algorithm*}

\section{Past Imbalance-Sensitive MTL} \label{app_sec:past_imb}
As we have discussed Nash-MTL previously, here we mainly introduce another two imbalance-sensitive MTL methods, i.e., IMTL~\cite{liu2021towards}, and FAMO~\cite{liu2023famo}.

\noindent \textbf{\underline{IMTL}}: IMTL addresses the imbalance issue by integrating two key strategies. It employs gradient balance to adjust shared parameters without bias towards any task, using a closed-form solution that ensures equal projections of the aggregated gradient onto individual task gradients. Additionally, it applies loss balance to dynamically weigh task losses, preventing any single task's loss scale from dominating the training process. By combining these strategies into a hybrid approach, IMTL achieves scale invariance, maintaining robust performance regardless of the magnitude of task losses. 

\noindent \textbf{\underline{FAMO}}: FAMO is a gradient-free optimization-based MTL method, which introduces a dynamic weighting method that adapts to the performance of individual tasks, ensuring that each task progresses at a similar rate without the need for extensive computational resources. FAMO operates with a constant space and time complexity of $\mathcal{O}(1)$ per iteration, which is a significant advantage over traditional methods that require $\mathcal{O}(k)$ space and time complexities, where $k$ represents the number of tasks. This efficiency makes FAMO particularly well-suited for large-scale multitask scenarios. 

\section{Limitations}   \label{sec:limit}
Although our proposed method has demonstrated effectiveness, it is important to acknowledge its limitations. First, the method is only applicable to scenarios with a decoupled objective, such as CAGrad and Nash-MTL, since IMGrad aims to dynamically balance objectives during optimization. Second, the introduction of additional computations to calculate $\bm{g_m}$ through MGDA represents another limitation. Although reducing the iterations of MGDA can mitigate this issue, this paper currently does not provide the corresponding performance and efficiency reports.

% \section{Impact Statements} \label{sec:impact}
% This paper presents work whose goal is to advance the field of Machine Learning. There are many potential societal consequences of our work, none which we feel must be specifically highlighted here.

\newpage
 
\bibliographystyle{named}
\bibliography{appendix}